\documentclass[a4paper,11pt]{amsart}%
\usepackage{hyperref}
\usepackage{pdfsync}
\usepackage{dsfont}
\usepackage{color}
\usepackage{esint}
\usepackage{amsthm}
\usepackage{enumerate}
\usepackage{amsfonts}
\usepackage{amssymb}%
\usepackage{scalerel,lmodern}

\usepackage{mathtools}
\usepackage{braket}
\usepackage{bm}
\usepackage{amsmath}
\usepackage{graphicx}
\usepackage{caption}
\usepackage{fullpage}
\usepackage{xcolor}
\usepackage{ulem}
\usepackage{mathtools}
\numberwithin{equation}{section}

\usepackage{float}
\usepackage{subfig}
\newcommand{\divv}{\mathrm{div}}

\usepackage{amsfonts}
\usepackage{graphicx}%
\providecommand{\U}[1]{\protect\rule{.1in}{.1in}}
\newtheorem{theorem}{Theorem}[section]

\usepackage{algorithm}
\usepackage{algpseudocode}

\newtheorem{corollary}[theorem]{Corollary}

\newtheorem{definition}[theorem]{Definition}
\newtheorem{example}[theorem]{Example}

\newtheorem{lemma}[theorem]{Lemma}

\newtheorem{proposition}[theorem]{Proposition}
\newtheorem{remark}[theorem]{Remark}

\newtheorem{assumption}[theorem]{Assumption}

\renewcommand{\H}{\mathcal{H}}
\newcommand{\C}{\mathbb{C}}
\newcommand{\veps}{\varepsilon}

\newcommand{\onabla}{\overline{\nabla}}
\newcommand{\odivv}{\overline{\divv}}
\newcommand{\A}{\mathcal{A}}
\renewcommand{\C}{\mathcal{C}}

\newcommand{\divergence}{\mathrm{div}}

\newcommand{\diam}{\mathrm{diam}}

\newcommand{\Risk}{\mathcal{R}}
\newcommand{\J}{\mathcal{J}}
\newcommand{\T}{\mathcal{T}}
\newcommand{\F}{\mathcal{F}}
\newcommand{\G}{\mathcal{G}}

\renewcommand{\U}{\mathcal{U}}
\newcommand{\I}{\mathcal{I}}

\newcommand{\R}{\mathbb{R}}

\newcommand{\M}{\mathcal{M}}
\newcommand{\red}{\color{red}}

\definecolor{mygreen}{rgb}{0.1,0.75,0.2}

\newcommand{\nc}{\normalcolor}
\newcommand{\E}{\mathbb{E}}

\newcommand{\N}{\mathbb{N}}

\newcommand{\pathp}[1]{\mathbbf[0.5]{#1}}
\newcommand{\bgamma}{\pathp{\gamma}}
\newcommand{\bsigma}{\pathp{\sigma}}

\newcommand{\Z}{\mathcal{Z}}

\newcommand{\one}{\mathds{1}}

\newcommand{\mc}{\mathcal}

\input pdf-trans
\newbox\qbox
\def\usecolor#1{\csname\string\color@#1\endcsname\space}
\newcommand\outline[1]{\leavevmode%
  \def\maltext{#1}%
  \setbox\qbox=\hbox{\maltext}%
  \boxgs{Q q 2 Tr \thickness\space w 0 0 0 rg 0 G}{}%
  \copy\qbox%
}
\newcommand\mathbbf[2][.2]{%
  \def\thickness{#1}%
  \ThisStyle{\outline{$\mathbf{\SavedStyle#2}$}}%
}

\definecolor{darkred}{rgb}{0.6,0.1,0.1}
\definecolor{darkgreen}{rgb}{0.1,0.6,0.1}
\definecolor{darkblue}{rgb}{0.1,0.1,0.6}

\title{On adversarial robustness and the use of Wasserstein ascent-descent dynamics to enforce it}
\author{Camilo Andr\'es Garc\'ia Trillos}
\address{}
\email{camilo.garcia@ucl.ac.uk}
\author{Nicol\'as Garc\'ia Trillos }
\address{}
\email{garciatrillo@wisc.edu}


\usepackage[
    style=numeric-comp,
  ]{biblatex}
\usepackage{biblatex}
\begin{document}
\normalem

\begin{abstract}
We propose iterative algorithms to solve adversarial problems in a variety of supervised learning settings of interest. Our algorithms, which can be interpreted as suitable ascent-descent dynamics in Wasserstein spaces, take the form of a system of interacting particles. These interacting particle dynamics are shown to converge toward appropriate mean-field limit equations in certain large number of particles regimes. In turn, we prove that, under certain regularity assumptions, these mean-field equations converge, in the large time limit, toward approximate Nash equilibria of the original adversarial learning problems. We present results for nonconvex-nonconcave settings, as well as for nonconvex-concave ones. Numerical experiments illustrate our results. 
\end{abstract}

\maketitle

\section{Introduction}

\subsection{The problem}

In many machine learning tasks one faces the need to solve min-max problems over probability spaces of the type
\begin{equation}
  \min\limits_{\nu \in \mathcal P (\Theta)} \max_{ \tilde \mu \in \mathcal{P}(\Z)} R(\tilde\mu,\nu) - C(\mu , \tilde \mu).
  \label{eqn:min_max_problem}
\end{equation}
For example, \eqref{eqn:min_max_problem} can be obtained as a functional relaxation of a (parametric) distributionally robust optimization (DRO) problem of the form:
\begin{equation}
  \min\limits_{\theta \in \Theta} \max_{ \tilde \mu \in \mathcal{P}(\Z)} \hat R(\tilde \mu,\theta) - C(\mu , \tilde \mu).
  \label{eqn:min_max_problem_motiv}
\end{equation}
Here $\Theta$ represents the space of parameters of a learning model, e.g., a classifier or regression function; $\mc Z= \mc X \times \mc Y$ is a space of input to output samples; $R$ is the risk function enforcing the learning problem; and $C$ represents the \textit{cost} that an adversary must pay to change the original data distribution of inputs to outputs $\mu$ to a new distribution $\tilde \mu$. In a nutshell, both problems, \eqref{eqn:min_max_problem} and \eqref{eqn:min_max_problem_motiv}, can be interpreted as games played by a learner and an adversary: for the learner the goal is to choose a parameter/distribution of parameters that is able to fare well when facing the attack of a reasonable adversary (reasonable as modeled by the cost function $C$). In this context, the relaxation procedure to go from \eqref{eqn:min_max_problem_motiv} to \eqref{eqn:min_max_problem} can be obtained in various reasonable ways. For instance, $R$ can be defined as an average of the risk $\hat{R}$, i.e. $R(\tilde \mu,\nu) = \int_\Theta \hat R\left(\tilde \mu,\theta) d\nu(\theta\right),$
or by considering the value function of some sort of average control, i.e.,
$R(\tilde \mu,\nu) =  \hat R \left( \tilde \mu, \int_\Theta \theta d\nu(\theta)\right).$
Either of these types of relaxation can be used as a way to improve the convexity properties of the original objective function over the parameter space $\Theta$, or, from a methodological viewpoint, as a technique for model aggregation. Other problems can be directly modelled by \eqref{eqn:min_max_problem}, as is the case in adversarial training of shallow neural networks, one of the motivating examples for this work; see section \ref{sec:RobustRegression} below.  

The overall goal of this paper is to propose and analyze ascent-descent dynamics to find approximate solutions (interpreted as Nash equilibria) to problem \eqref{eqn:min_max_problem} under reasonable assumptions on the risk $R$ and the cost $C$ that are general enough to encompass important settings in applications, especifically in adversarial machine learning. These dynamics will take the form of interacting particle systems, and we will be particularly interested in studying their behavior as the number of particles in the system grows. We will also discuss the long time behavior of the limiting mean-field dynamics and explore their ability to produce approximate Nash equilibria for \eqref{eqn:min_max_problem}. 

To write down the concrete particle system that we will analyze in this work, and discuss the underlying geometric structure motivating it, it will be convenient to reformulate problem \eqref{eqn:min_max_problem} in a slightly different way, assuming that $C$ has the structure of an optimal transport problem; we will discuss this reformulation in section \ref{sec:CouplingReform} below. Likewise, it will be convenient to introduce a concrete working example motivating problem \eqref{eqn:min_max_problem}. In particular, we will give a concrete meaning to the space of learning parameters $\Theta$, and define a specific risk functional $R$ and cost function $C$; we do this in section \ref{sec:RobustRegression} below.

\nc



\subsubsection{Reexpressing the problem in terms of couplings $\pi$}
\label{sec:CouplingReform}

Throughout this work we will assume that the adversarial cost $C$ in \eqref{eqn:min_max_problem} can be written as an optimal transport problem of the form:
\begin{equation}
  C(\mu , \tilde \mu) := \inf_{\pi \in \Gamma(\mu, \tilde \mu)} \int c(z, \tilde z) d\pi(z, \tilde z),  
  \label{eq:AdversarialCost}
\end{equation}
where $c:\Z \times \Z \rightarrow [0,\infty]$ is a lower semi-continuous cost function w.r.t. a topology that $\Z$ is tacitly endowed with, and $\Gamma(\mu, \tilde \mu)$ denotes the space of couplings between $\mu$ and $\tilde \mu$, i.e. the space of probability measures on the product space $\Z \times \Z$ whose first and second marginals are, respectively, $\mu$ and $\tilde \mu$. The cost function $c$ can be interpreted as the marginal cost that the adversary must pay in order to move a data sample $z$ to a new location $\tilde z$. 

The cost structure in \eqref{eq:AdversarialCost} is quite broad and encompasses several models for adversarial robustness in the literature. Moreover, when $C$ has the structure \eqref{eq:AdversarialCost}, the optimization variable $\tilde \mu$ in \eqref{eqn:min_max_problem}, i.e., the variable corresponding to a (global) adversarial attack, can be replaced with a coupling variable $\pi$ as in \eqref{eq:AdversarialCost}. To see this, let $\pi \in \mathcal{P}(\Z \times \Z)$ and define
\[ \mathcal{C}(\pi) := \int_{\Z \times \Z} c(z , \tilde z ) d\pi(z, \tilde z) \]
We will use $\pi_z$ and $\pi_{\tilde z}$ to denote, respectively, the first and second marginals of $\pi$. Given $\nu \in \mathcal{P}(\Theta)$ and $\pi \in \mathcal{P}(\Z \times \Z)$, we define:
\begin{equation}
\label{eqn:DefU}
\U(\pi, \nu):= \Risk( \pi,\nu) - \C (\pi), 
\end{equation}
where $\Risk(\pi, \nu):= R(\pi_{\tilde z}, \nu)$, and where we implicitly assume that the above difference makes sense, i.e., we exclude $( \pi,\nu)$ from the domain of definition of $\U$ whenever $\Risk(\pi,\nu) = \C (\pi) = \infty$. Problem \eqref{eqn:min_max_problem} then admits the following equivalent reformulation:
\begin{equation}
  \min\limits_{\nu \in \mathcal \mathcal P(\Theta)} \max_{ \pi \in \mc P(\mc Z\times \mc Z) ; \pi_z = \mu} \U(\pi, \nu). 
  \label{min max problem couplings}
\end{equation}
Indeed, it is straightforward to see that if $( \pi^*,\nu^*)$ is a Nash equilibrium for problem \eqref{min max problem couplings}, then $( \pi^*_{\tilde z},\nu^* )$ is a Nash equilibrium for \eqref{eqn:min_max_problem}. Conversely, if $( \tilde \mu^*,\nu^*)$ is a solution to \eqref{eqn:min_max_problem}, then $(\pi^*,\nu^*)$, where $\pi^*$ is an optimal solution for problem \eqref{eq:AdversarialCost} with $\tilde \mu= \tilde\mu^*$, is a solution for problem \eqref{min max problem couplings}.  

The reformulation \eqref{min max problem couplings} and the change of variables $\tilde \mu \mapsto \pi$ effectively allow us to replace the non-linear objective $C(\mu, \cdot)$ with the linear functional $\C(\cdot)$. This substitution facilitates the definition of our algorithms, in some cases inducing minmax problems with improved concavity properties (e.g., see section \ref{sec:strongCon}). We will thus focus on the reformulation \eqref{min max problem couplings} and on ascent-descent dynamics aimed at solving it. Conceptually, the variable $\nu$ in \eqref{min max problem couplings} specifies a regression/classification function (i.e. an action/strategy played by the learner), while $\pi$ determines a way in which the adversary modifies the original data; see more details in section \ref{sec:RobustRegression} below. Problem \eqref{min max problem couplings} can be interpreted as a DRO (distributionally robust optimization) version of adversarial training with an explicit penalty, as opposed to an explicit constraint; see \cite{DROWithPenalty}.




\nc

\subsection{Motivating example: Robust regression or binary classification with shallow neural networks}
\label{sec:RobustRegression}

Before we move on with the description of our algorithms and main results, it will be convenient to provide some concrete examples of risk functionals $\Risk$ and cost functions $\C$ that are of interest in practical settings. We do so in this section. A different setting is considered in section \ref{sec:Experiments}.

Let $\Theta \subseteq  \R \times \R^{d'}$, $\Z = \R^{d'} \times \R$, and write $\theta=(a,b)$ and $z=(x,y)$. We consider the following risk and adversarial cost:
\begin{equation}
  \Risk(\pi,\nu) := \int_{\Z \times \Z}  \ell (h_\nu(\tilde x), \tilde y) d \pi(z, \tilde z),\quad h_\nu(x):= \int_{\Theta} a f(b \cdot x) d \nu(a,b),
  \label{Ex of h_nu}
\end{equation}
where $\ell: \R \times \R \rightarrow [0,\infty]$ is a loss function, $f$ is an activation function, and
\begin{equation}
\label{eq:Cost}
  \C(\pi) := c_a\int_{\Z \times \Z} |z-\tilde z|^2 d\pi(z, \tilde z),
\end{equation}
for $c_a$ a positive parameter. This case, with a  square Euclidean distance, is one of the many possible choices for the cost function $c$.  

Notice that the function $h_\nu$ can be interpreted as a continuum analog of a two-layer (one hidden layer) neural network without offset parameters (for simplicity); see \cite{E2022}. Popular examples of activation functions are the following:
\begin{enumerate}
    \item Sigmoid: $f(t)=\frac{1}{1+e^{-t}}$.
    \item ReLu: $f(t)= (t)_+ $
    \item Squared ReLu: $f(t)=(t_+)^2$. 
\end{enumerate}

Examples of loss functions of interest are:
\begin{enumerate}
    \item Squared loss: $\ell(y,y') = |y-y'|^2$.
    \item Logistic loss: $\ell (y, y') = -\log(|1-y'-y|)$, when $y, y'$ are constrained to belong to $[0,1]$.
\end{enumerate}

The parameter $c_a$ in the adversarial cost, which can be interpreted as reciprocal to an adversarial budget, determines the strength of adversarial attacks. In particular, if $c_a$ is small, the attacker can carry out stronger attacks, while the opposite is true when $c_a$ is large. In applications, the goal of solving a problem such as \eqref{min max problem couplings} is to enhance the robustness of a learning model. In other words, instead of simply solving an optimization problem of the form $\min_{\nu} \Risk(\pi,\nu)$, which corresponds to the standard training process, one considers problem \eqref{min max problem couplings} to account for the possible actions of an adversary. However, in order to avoid enhancing robustness at the expense of a considerable loss in accuracy, it is important to tune the adversarial budget appropriately. Some papers that have studied the trade-off between robustness and accuracy include \cite{RobustnessVsAccuracy1,RobustnessVsAccuracy2}.


 



\begin{remark}
In the setting considered above, and when working with the squared loss function $\ell(y, y')= |y-y'|^2$, the loss function $\ell( h_\nu(\tilde x), \tilde y )$ can be written, after expanding the square, as
\[ \ell( h_\nu(\tilde x), \tilde y ) =   \int_{\Theta}\int_{\Theta} a f(b \cdot \tilde x)  a' f(b' \cdot  \tilde x) d \nu( a',  b') d \nu(a,b) - 2   \tilde y\int_{\Theta} af(b \cdot \tilde x) d \nu(a,b) + \tilde y^2.   \]
In particular, we can write 
\begin{align*}
\begin{split}
\Risk( \pi,\nu) & =    \int_{\Theta}\int_{\Theta} \left(\int_{\Z^2} a f(b \cdot \tilde x)  a' f(b'\cdot  \tilde x) d \pi(z, \tilde z)\right) d \nu( a',  b') d \nu(a,b) 
\\& - 2   \int_{\Theta}   \int_{\Z^2}\tilde y  af(b \cdot \tilde x)  d \pi(z, \tilde z) d \nu(a,b) + \int_{\Z^2}\tilde y^2 d\pi(z, \tilde z)
\end{split}
\end{align*}
This expression has the merit of exposing explicitly the dependence of the risk  with respect to the measure $\nu$.
\nc 
\end{remark}


\subsection{Algorithm}
\label{sec:Algorithm}

We introduce in Algorithm \ref{Algo:WADA} a particle-based scheme for solving the min-max problem \eqref{min max problem couplings}. 
\begin{algorithm}
    \caption{Wasserstein ascent-descent algorithm \label{Algo:WADA}}
    \begin{algorithmic}[1]
        \Require A collection $\{z_{i,0}, \omega_{i,0}\}_{i=1,\ldots, n}$ such that $\frac{1}{n}\sum_{i=1}^n \omega_{i,0} \delta_{z_{i,0}} $ approximates $\mu$.
        \State Set $t=0$
        \State Choose $\{\vartheta_{k,0}\}_{k= 1, \ldots M}$,  $\{  \alpha_{k,0} \}_{k=1,\ldots, M }$, $\{\tilde z_{ij,0}\}_{i=1,, \dots, n; \: j= 1, \ldots N}$, and $\{\omega_{ij,0} \}_{i=1,\ldots, n;\ j= 1, \ldots N }$
        with the constraint 
        \[ \sum_{j=1}^N \omega_{ij,0} = \omega_{i,0}  \text{ for all } i=1, \ldots n. \]
        \While { Stopping condition has not been satisfied }
        \State Set
        \[ \pi_t^{n,N} := \sum_{i=1}^n\sum_{j=1}^N \omega_{ij,t} \delta_{ (z_{i,0}, \tilde z_{ij,t})} \qquad  \nu_t^M := \sum_{k=1}^M \alpha_{k,t} \delta_{\vartheta_{k,t}} \]

        \For {$i=1$ to $n$ ; $j=1$ to $N$ }
          \State $\tilde z_{ij, t+1} =  P_{\Z} \left( \tilde z_{ij, t}  +  \eta_{ t \nc} \nabla_{\tilde z} \U_{\pi} ( \pi_t,\nu_t; z_{i,0} \tilde z_{ij, t}) \right) $        
          \State $\hat{\omega}_{ij, t+1}:= {\omega}_{ij, t} \exp \left(\eta_{ t } \kappa \sum_{j'} \omega_{ij',t} \U_{\pi}( \pi_t,\nu_t; z_{i,0},  \tilde z_{ij,t})  \right)$
          \State $\omega_{ij, t+1}:= \frac{\hat{\omega}_{ij, t+1}}{ \sum_{j'} \hat{\omega}_{ij', t+1}  }$ 
        \EndFor
        \For {$k=1$ to $M$}
          \State $\vartheta_{k, t+1} =  P_{\Theta} \left( \nc \vartheta_{k, t}  - \eta_{ t \nc} \nabla_{\theta} \U_{\nu}(\pi_t, \nu_t; \vartheta_{k,t})  \right) \nc$
          \State $\hat{\alpha}_{k, t+1}:= {\alpha}_{k, t} \exp \left(-\eta_{ t \nc}\kappa \sum_{k'} \alpha_{k',t} \U_{\nu}(\pi_t, \nu_t;\vartheta_{k,t})  \right)$
          \State $\alpha_{k, t+1}:= \frac{\hat{\alpha}_{k, t+1}}{ \sum_{k'} \hat{\alpha}_{k', t+1}  }$
        \EndFor
        \State $t=t+1$
        \EndWhile 
        \State \textit{**Calculate time-average**}
        \State$\bar z_{ij} := \frac{1}{t}\sum_{s=0}^{t} \omega_{ij,s} \tilde z_{ij,s}$ for $i=1, \ldots, n ; j=1, \ldots, N $
        \State$\bar \vartheta_{k} := \frac{1}{t} \sum_{s=0}^{t} \alpha_{k,s} \tilde \vartheta_{k,s}$ for $ k=1, \ldots, M$       \label{AlgLine: time-avg_net}
    \end{algorithmic} 
    \label{Algo}
  \end{algorithm} 
Implicit in the definition of Algorithm \ref{Algo} is the use of the \textit{first variations} of the functional $\U$ in the directions $\nu$ and $\pi$. 

Following Definition 7.12. in \cite{santambrogio2015optimal}, we say that the measurable function $\U_\pi: \Z \times \Z \rightarrow \R$ is the first variation of $\U$ in the direction $\pi$ at the point $(\pi,\nu)$ if for any $\pi^* \in \mathcal{P}(\mc Z\times \mc Z )$ we have
\[ \left. \frac{d}{d\epsilon} ( \U(\pi + \epsilon  (\pi^*-\pi) \nc , \nu )) \right|_{\epsilon =0} = \int_{\Z \times \Z} \U_\pi (\pi,\nu; z,\tilde z)   d (\pi^*-\pi). \]
In general, $\U_\pi$ may depend on the point $(\pi, \nu)$ at which the first variation is being evaluated, but we will drop the explicit reference to this dependence whenever no confusion may arise from doing so, for otherwise we will write all of $\U_{\pi}$'s arguments like this:
$\U_\pi( \pi, \nu;z,\tilde z )$. Likewise, we say that the measurable function $\U_\nu: \Theta \rightarrow \R$ is the first variation of $\U$ in the direction $\nu$ at the point $(\pi,\nu)$ if for any $\nu^* \in \mathcal{P}(\Theta)$ we have
\[ \left. \frac{d}{d\epsilon} ( \U( \pi, \nu+  \epsilon(\nu^*-\nu) \nc   )) \right|_{\epsilon =0} = \int_\Theta \U_\nu (\pi,\nu; \theta)   d (\nu^*-\nu). \]
Throughout the paper we will assume that the first variations of $\U$ are well defined and satisfy certain regularity properties that are stated precisely in Assumptions \ref{Hyp U}. 

In Algorithm \ref{Algo} the $\eta_t$ can be interpreted as a time dependent learning rate, and $\kappa$ as a fixed parameter. The projection maps $P_{\Z}$ and $P_\Theta$ are introduced to guarantee that the iterates stay within the sets $\Z$ and $\Theta$. The averaging steps in lines 18-19 will be discussed in section \ref{sec:Experiments}; Algorithm \ref{Algo} is related to algorithms introduced in \cite{domingo2020mean,wang_exponentially_2022}; a comparison between the content of those papers and ours is presented in section \ref{sec:LitReview}.

\medskip 

The collection of iterates in Algorithm \ref{Algo} can be thought of as a time discretization of the system of ODEs:
\begin{equation}
  \begin{aligned}
    d Z_t^{i} & = 0 \\
    d \tilde Z_t^{i} & = \eta_{ t \nc} \nabla_{\tilde z} \U_\pi (\pi^N_t,\nu^N_t; Z_t^{i},\tilde Z_t^{i} )dt \\
    d \omega_t^{i} & = \kappa \omega_t^{i}  \left( \U_\pi (\pi^N_t,\nu^N_t; Z_t^{i},\tilde Z_t^{i} ) - \int  \U_\pi (\pi_t^N, \nu_t^N; Z_t^{i},\tilde z') d  \pi^N_t(\tilde z'| Z_t^{i} )\right) dt\\
    d \vartheta_t^{i} & = - \eta_{ t \nc} \nabla_{\theta} \U_\nu (\pi^N_t,\nu^N_t;\vartheta_t^{i})dt \\
    d \alpha_t^{i} & = -\kappa \alpha_t^{i}  \left( \U_\nu (\pi^N_N,\nu^N_t;\vartheta_t^{i}) - \int  \U_\nu (\pi_t^N, \nu_t^N;\theta') d \nu^N_t(\theta')\right)dt,
\end{aligned}
\label{ODEsIntro}
\end{equation}
with given initial condition $(Z^i_0,\tilde Z^i_0,\omega^i_0,\vartheta^i_0,\alpha^i_0)$ (possibly random)
and
\begin{align}
     \pi_t^N &:=\frac 1 N \sum_{i=1}^N \omega^i_t \delta_{( Z_t^i, \tilde Z_t^i)}, 
     & \nu_t^N &:= \frac 1 N \sum_{i=1}^N \alpha^i_t \delta_{\vartheta_t^i}.
     \label{eqn:InitialiCondEmpirirical1}
  \end{align}
  Notice that in the above ODE, as well as in our analysis in section \ref{sec:ConvergenceMeanField}, we have considered the same number of particles $Z, \tilde Z, \vartheta$ and we have eliminated the double indexes. This we do for simplicity and in order to reduce the burdensome notation throughout our analysis. We will only return to the double indexes when needed.

Of particular interest to us is the evolution of the measures  $\pi_t^N, \nu_t^N$, which can be  showed to satisfy the PDE (in weak form)
\begin{equation}
  \begin{cases}
  \partial_t \pi_t &= - \eta_{ t \nc} \divergence_{z,\tilde z}  (\pi_t (0, \nabla_{\tilde z} \U_{\pi}(\pi_t, \nu_t;z,\tilde z) ))+ \kappa \pi_t \left( \U_\pi(\pi_t, \nu_t; z,\tilde z) - \int \U_\pi(\pi_t, \nu_t;z, \tilde z ') d\pi_{t}(\tilde z '| z)      \right) \\
  \partial_t \nu_t &= \eta_{ t } \divergence_{\theta}  (\nu_t \nabla_{\theta} \U_{\nu}(\pi_t, \nu_t;\theta)) - \kappa \nu_t \left( \U_{\nu}(\pi_t, \nu_t;\theta) - \int \U_{\nu}(\pi_t, \nu_t; \theta') d\nu_t(\theta')   \right), 
  \end{cases}
  \label{WFR dynamics}
\end{equation}
 with initial condition
 \begin{equation}
  \pi_0 = \pi_0^N=  \sum_{i=1}^N  \omega_0^i \delta_{(Z^i_{0},\tilde{Z}^i_{0} )}, 
  \quad \nu_0 = \nu_0^N= \sum_{i=1}^N \alpha^i_{0} \delta_{\vartheta^i_{0}}.   
  \label{eqn:EmpiricalInitialization}
 \end{equation} 

In our first main result, Theorem \ref{thm:Main1} below, we state that, under certain conditions on $\U$ and its first variations $\U_\pi$ and $\U_{\nu}$ –ultimately determined by the choice of risk function $\Risk$ and cost function $\C$ in applications–, the large number of particles limit ($N \rightarrow \infty$) of the system \eqref{WFR dynamics} initialized as in \eqref{eqn:EmpiricalInitialization} follow the same dynamics \eqref{WFR dynamics} with a more general initialization; as part of our analysis in section \ref{sec:ConvergenceMeanField} we show that this system is well-posed. The motivation for the use of the term \textit{gradient ascent-descent} to describe the system \eqref{WFR dynamics}, and by extension also for the steps in Algorithm \ref{Algo}, will be discussed in section \ref{sec:GradDescentAscent}. 




\begin{example}
  \label{example:2}
  In the context of the motivating example from section \ref{sec:RobustRegression} we see that:
\begin{equation}
\U_{\pi}(\pi, \nu; z, \tilde z)= \ell(h_\nu(\tilde x), \tilde y) - c(z, \tilde z) = \ell(h_\nu(\tilde x), \tilde y) - c_a|z- \tilde z|^2,
\end{equation}
and
\begin{equation}
\U_{\nu}(\pi, \nu; \theta) = \int_{\Z \times \Z }\ell'(h_\nu(\tilde x), \tilde y) (a f(b \cdot \tilde x)) d\pi(z, \tilde z),
\end{equation}
where $\theta=(a,b)$. Here, $\ell'$ denotes the derivative of $\ell$ in its first coordinate. 

For the case of the squared-loss, $\U_\nu$ can be rewritten as:
\begin{align}
\begin{split}
\U_\nu(\pi, \nu; \theta) = & 2\int_{\Z \times \Z} \int_\Theta (a'  f(b' \cdot \tilde x)) (a  f(b \cdot  \tilde x)) d \nu(\theta') d\pi(z,\tilde z)  \\& - 2\int_{\Z \times \Z}  \tilde y (a  f(b \cdot  \tilde x))   d\pi(z,\tilde z).
\end{split}
\label{eqn:UnuRegression}
\end{align}
\end{example}

  \nc

  We finish this section with the following remark stating that the system \eqref{WFR dynamics} with arbitrary initialization satisfies certain conservation of mass properties.

\begin{remark}
  Note that the dynamics in \eqref{WFR dynamics} imply that the first marginal of $\pi$ ($\pi_z$) remains constant. This can be verified by considering a test function $\phi:\mathcal Z \rightarrow \R$ and observing that
  \begin{align*}
     \frac{d}{dt} \int_{\mc Z \times \mc Z} \phi(z) d\pi_t(z,\tilde z) & =  \eta_{ t \nc} \int_{\mc Z \times \mc Z}  \nabla_{z,\tilde z} \phi(z) \cdot (0, \nabla_{\tilde z} \U_\pi ) d\pi_t(z, \tilde z) \\
     & \quad  + \kappa \int_{\mc Z \times \mc Z} \phi(z) \left( \U_\pi(z,\tilde z) - \int \U_\pi(z, \tilde z') d\pi_{t}(\tilde z '| z)      \right) d\pi_t(z, \tilde z)   \\
     & =  \kappa \int_{\mc Z } \phi(z)  \int_{\Z} \left( \U_\pi(z,\tilde z) - \int \U_\pi(z, \tilde z ') d\pi_{t}(\tilde z '| x)      \right) d\pi_{t}(\tilde z | z) d\pi_{t,z}(z)\\
     & = 0. 
  \end{align*}

 Similarly, one can show that $\nu_t$ and $\pi_t$ have total mass equal to one for all times, if $\nu_0, \pi_0$ are probability measures.
  \label{Rmk:ConstantMarginal}
 \end{remark}
 
 


\subsection{Main results} 
\label{sec:MainResults}
We are now ready to state our main results. Our first result, which describes the large number of particles limit ($N \rightarrow \infty$) of the system \eqref{WFR dynamics} when initialized according to \eqref{eqn:EmpiricalInitialization}, is deduced under the following assumptions on $\U$ and its first variations.

\begin{assumption}
We assume that there exist constants $M,L >0$ such that
  \begin{itemize}
    \item $\U$ is bounded, and Lipschitz with respect to the $1$-Wasserstein distance, that is 
    \[ |\U(\pi,\nu)| \leq M ; \qquad  \U(\pi^1,\nu^1) - \U(\pi^2,\nu^2) \leq L (W_1(\pi^1,\pi^2) +  W_1(\nu^1,\nu^2)  ) .\]
    \item The first variations of $\U$ are bounded and Lipschitz, i.e
 \begin{align*}
      &| \U_\pi(\pi,\nu; z,\tilde z) | + | \U_\nu(\pi,\nu; \theta) |\leq M  \\
      &|\U_\pi(\pi^1,\nu^1; z^1,\tilde z^1) - \U_\pi(\pi^2,\nu^2; z^2, \tilde z^2) | \leq L (W_1(\pi^1,\pi^2) +  W_1(\nu^1,\nu^2) + |z^1-z^2|+|\tilde z^1-\tilde z^2|  )  \\
      &|\U_\nu(\pi^1,\nu^1; \theta^1) - \U_\nu(\pi^2,\nu^2; \theta^2) |\leq L (W_1(\pi^1,\pi^2) +  W_1(\nu^1,\nu^2) + |\theta^1-\theta^2|  ) .
    \end{align*}
    \item The gradients of the first variations of $\U$ are bounded and Lipschitz, i.e.
    \begin{align*}
      &|\nabla_{\tilde z} \U_\pi(\pi,\nu; z,\tilde z) | + |\nabla_{\theta} \U_\nu(\pi,\nu; \theta) |\leq M  \\
      &|\nabla_{\tilde z}\U_\pi(\pi^1,\nu^1; z^1,\tilde z ^1) - \nabla_{\tilde z} \U_\pi(\pi^2,\nu^2; z^2, \tilde z^2) | \leq L (W_1(\pi^1,\pi^2) +  W_1(\nu^1,\nu^2) + |z^1-z^2|+|\tilde z^1-\tilde z^2|  )  \\
      &|\nabla_\theta \U_\nu(\pi^1,\nu^1; \theta^1) - \nabla_\theta \U_\nu(\pi^2,\nu^2; \theta^2) |\leq L (W_1(\pi^1,\pi^2) +  W_1(\nu^1,\nu^2) + |\theta^1-\theta^2|  ) .
    \end{align*}
  \end{itemize} 
  In the above, $\pi,\pi^i \in \mathcal{P}(\Z^2)$, $\nu, \nu^i \in \mathcal{P}(\Theta)$, $(z^i, \tilde z^i ) \in \Z^2$, and $\theta^i \in \Theta$. The sets $\Theta$ and $\Z^2$ are compact domains of the Euclidean spaces $\R^d$ and $\R^{2d'}$, respectively. We assume that these domains have Lipschitz boundaries.
  \label{Hyp U}
\end{assumption}

\begin{remark}
\label{rem:DefinitionGeneralWass}
In the sequel, we use the $p$-Wasserstein distance $W_p$ (with $p \geq 1$) to compare probability distributions over a variety of metric spaces $(\mathbb{M}, d(\cdot, \cdot))$. We recall that for given two probability measures $\upsilon, \upsilon'$ over $\mathbb{M}$, their $p$-Wasserstein distance $W_p(\upsilon,\upsilon')$ is defined according to
\[ W^p_{p}(\upsilon, \upsilon') := \inf_{\Upsilon \in \Gamma(\upsilon, \upsilon')}\int_{\mathbb{M} \times \mathbb{M}} (d(u,u'  ))^p d\Upsilon(u,u'),  \]
where $\Gamma(\upsilon, \upsilon')$ is the set of couplings between $\upsilon$ and $\upsilon'$. The metric $d(\cdot, \cdot)$ that will be used in each instance will be specified in context. For example, in Assumption \ref{Hyp U}, the $1$-Wasserstein distances considered are the ones relative to the Euclidean metric in each corresponding Euclidean space.  

Additionally, we will use $\mathcal{P}(\mathbb{M})$ to denote the space of Borel probability measures over $\mathbb{M}$.

\nc
\end{remark}

Since the sets $\Theta$ and $\Z^2$ have been assumed to be bounded, in order to simplify the writing of our proofs and guarantee that all the dynamics to be studied in the paper stay within the domains $\Theta$ and $\Z^2$, we make the following technical assumption: 
\begin{assumption}
At all points $ \tilde z$ at the boundary of $\Z$ and at all points $\theta$ at the boundary of $\Theta$, it holds that the vector $\nabla_{\tilde z} \U_\pi(\pi, \nu;z, \tilde z )$ points toward the interior of $\Z$, regardless of $\pi, \nu, z$; and the vector $\nabla_\theta \U_\nu(\pi, \nu;\theta)  $ points toward the interior of $\Theta$, regardless of $\pi, \nu$. 

\label{Hyp:Support is stationary}
\end{assumption}

By restricting our attention to compact domains $\Theta, \Z^2$ we make it simpler to verify the boundedness and Lipschitz conditions in Assumption \ref{Hyp U} as these conditions reduce to weaker properties like local-Lipschitzness. Notice that in many applications there are natural bounds on the supports of the desired solution\footnote{To give only one example, images are typically represented by pixels which have a lower and upper values}. Assumption \ref{Hyp:Support is stationary}, on the other hand, guarantees that all dynamics considered in the paper remain in the domains $\Theta$ and $\Z^2$ (e.g., the dynamics \eqref{ODEsIntro}). In order to satisfy the constraint imposed by this assumption, we can work with a modified functional $\U$ that strongly penalizes leaving the domains as we move closer to their borders. In particular, to a given $\U$ satisfying Assumptions \ref{Hyp U} we can add, if needed, an exogenous term of the form $ \int \varphi_2(\theta)d \nu(\theta) -\int \varphi_1(\tilde z)  d\pi_{\tilde z}  $, where the potentials $\varphi_1, \varphi_2$ are zero away from the boundary of the domains and grow as one approaches the boundaries. We reiterate that this assumption is made in order to simplify the writing of our proofs. Throughout the entire paper we adopt Assumption \ref{Hyp:Support is stationary}, even if not mentioned explicitly.

\begin{example}
  
  In the context of the motivating Example \ref{example:2} we see that, for compact domains $\Theta$ and $\Z$, all conditions in Assumption \ref{Hyp U} are satisfied when one considers a loss function that is twice differentiable, and an activation function whose first derivative is Lipschitz. This is the case, for example, for the squared-loss and the squared ReLu or sigmoid activations.
\end{example}

\medskip
We are ready to state our first main result. 
\begin{theorem}[Convergence particle system ]
\label{thm:Main1}
Let $T>0$, and suppose that Assumptions \ref{Hyp U} and \ref{Hyp:Support is stationary} hold. Let $\bar\pi_0, \bar \nu_0$ be probability measures with $\bar{\pi}_{0,z}=\mu$. Let $\gamma_t^N , \sigma_t^N$ 
\begin{align*}
   \gamma_t^N &:= \frac 1 N \sum_{i=1}^N  \delta_{( Z_t^i, \tilde Z_t^i),\omega^i_t}, 
   & \sigma_t^N &:= \frac 1 N \sum_{i=1}^N  \delta_{\vartheta_t^i,\alpha^i_t},
\end{align*}
for initial values $\omega^{i}_0, \alpha^{i}_0$ bounded from above by a constant $D$ (uniformly over $N$) and $Z_{0}^i$ in the support of $\mu$, and evolutions as in \eqref{ODEsIntro}.
Finally, suppose that, as $N \rightarrow \infty$,
\begin{equation}
 \inf_{\upsilon \in \Gamma_{\text{Opt}}(\pi_{0,z}^N, \overline{\pi}_{0,z}) } \int W_1({\gamma}_0^N(\cdot| z_0'),{\gamma}_0(\cdot| z_0) ) d \upsilon(z_0',z_0) \rightarrow 0, \text{ and }  W_1( {\sigma}_0^N, {\sigma}_0) \rightarrow 0, 
 \label{eq:InitialConditionConditionals}
 \end{equation}
for two probability measures $\gamma_0$ and $\sigma_0$ satisfying $\F \gamma_0= \bar{\pi}_0$ and $\F \sigma_0= \bar{\nu}_0$, where $\F$ is defined in \eqref{sec:FNU} (and in section \ref{sec:FPI}). In the above, $\Gamma_{\text{Opt}}(\pi_{0,z}^N, \overline{\pi}_{0,z})$ stands for the set of couplings that realize the $1$-Wasserstein distance between $\pi_{0,z}^N$ and $\overline{\pi}_{0,z}$.

Then, as $N \rightarrow \infty$,
\[ \sup_{t\in [0,T]} \{ W_{1}({\pi}_t^N , \pi_t) +  W_{1}(\nu_t^N , \nu_t )\}\rightarrow 0,\]
where $\pi_t, \nu_t$ solve \eqref{WFR dynamics} with initializations $\bar{\pi}_0$ and $\bar{\nu}_0$.
\end{theorem}
\nc

\begin{remark}[Constructing approximate initializations in Theorem \ref{thm:Main1}]

Fix $\overline{\pi}_0$ and $\overline{\nu}_0$ and define $\gamma_0=\overline{\pi}_0 \otimes \delta_{1}$. That is, $\gamma_0$ is the product of $\pi_0$ with a Dirac delta at $1$. Likewise, let $\sigma_0=\overline{\nu}_0 \otimes \delta_{1}$. Evidently, $\F \gamma_0 = \overline{\pi}_0,\F\sigma_0= \overline{\nu}_0$.

We use randomization to construct 
approximate initializations satisfying the assumptions in Theorem \ref{thm:Main1}. Let $\xi_1, \dots, \xi_n, \dots$ be a sequence of i.i.d. samples from $\overline{\pi}_{0,z}$, and for each $i\in \N$ let $\tilde{Z}_{i1}, \dots, \tilde{Z}_{im}, \dots,$ be i.i.d. samples from $\overline{\pi}_0(\cdot| \xi_i )$. Let $\theta_1, \dots, \theta_n, \dots$ be i.i.d. samples from $\overline{\nu}_0$.

For fixed $n,m$ and $i\leq n$ and $j \leq m$, set $\omega_{ij}=\alpha_{ij}=1$, $Z_{ij}= \xi_i$, and $\vartheta_{ij}= \theta_{i}$. Consider the measures
\[ \pi_0^{n,m}:= \frac{1}{nm}\sum_{i=1}^{n}\sum_{j=1}^{m}  \delta_{(Z_{ij}, \tilde{Z}_{ij} )}, \quad \gamma_0^{n,m}:= \frac{1}{nm}\sum_{i=1}^{n}\sum_{j=1}^{m}  \delta_{(Z_{ij}, \tilde{Z}_{ij}, \omega_{ij} )} \]
and 
\[ \nu_0^{n,m}:= \frac{1}{nm}\sum_{i=1}^{n}\sum_{j=1}^{m}  \delta_{\vartheta_{ij}}, \quad \sigma_0^{n,m}:= \frac{1}{nm}\sum_{i=1}^{n}\sum_{j=1}^{m}  \delta_{(\vartheta_{ij}, \alpha_{ij} )}. \]
Evidently, $\F \gamma_0^{n,m}= \pi_0^{n,m}$ and $\F \sigma_0^{n,m} = \nu_0^{n,m}$, and the $Z_{ij}$ can be assumed to belong to the support of $\pi_{0,z}$. It is also clear that the measure $\gamma^{n,m}_0$ has support in $\Z^2 \times [0,1]$ and $\sigma^{n,m}_0$ has support in $\Theta \times [0,1]$.

By Lemma \ref{lem:SatisfyingAssumptions} in Appendix \ref{App:1}, we can conclude that there exists a sequence $\{(n_k, m_k)\}_{k\in \N} $ such that, as $k \rightarrow \infty$, the measures $\sigma_{0}^{N_k}:= \sigma_{0}^{n_k, m_k} $ and $\gamma^{N_k}_0:= \gamma_{0}^{n_k, m_k}$ satisfy \eqref{eq:InitialConditionConditionals} with probability one. 
\label{rem:Initial}
\end{remark}

\nc 
\begin{remark}
We highlight that in order to satisfy the first condition in \eqref{eq:InitialConditionConditionals} we need to consider the iterative sampling for the variables $Z_{ij}, \tilde{Z}_{ij}$ illustrated in Remark \ref{rem:Initial}, while in general i.i.d. sampling from $\overline{\pi}_{0}$ does not provide a valid initialization for the particle system. This is because the first condition in \eqref{eq:InitialConditionConditionals} is a stronger condition than simply requiring $W_1(\gamma_0^N, \gamma_0) \rightarrow 0$; see Remark \ref{rem:App1} in Appendix \ref{app1}. 

Finally, we highlight that the assumption on the conditional distributions at intialization imposed in \eqref{eq:InitialConditionConditionals}  is used to control the conditional distributions of $\pi$ as the systems evolve in time.

\end{remark}

\nc

\medskip

Having discussed the convergence of the particle system toward the mean-field limit equation \eqref{WFR dynamics}, we turn our attention to the long-time behavior of the mean-field limit, whenever the mean-field is initialized appropriately. We will show that, under Assumptions \ref{Hyp U} and some additional convexity-concavity assumptions on $\U$ (convexity-concavity interpreted in the linear interpolation sense, see Assumptions \ref{assump:ConvConcav} below), we can recover an $\veps$-Nash equilibrium for problem \eqref{min max problem couplings} from \eqref{WFR dynamics}. \nc

\begin{definition}[$\veps$-Nash equilibrium]
We say that $( \pi^*, \nu^*)$ is an $\veps$-Nash equilibrium for problem \eqref{min max problem couplings} if $\pi^*_z = \mu $ and
\begin{equation}
\label{eqn:NashDef}
    \sup_{\pi \in \mathcal{P}(\Z \times \Z)  \: \text{ s.t. }\: \pi_z = \mu }  \{\U(\pi, \nu^*) \}   -\inf_{\nu \in \mathcal{P}(\Theta)} \{ \U(\pi^*, \nu ) \}  \leq \veps.
\end{equation}
\end{definition}

\begin{remark}
Notice that condition \eqref{eqn:NashDef} is equivalent to: 
\[ \U(\pi, \nu^*) - \U(\pi^*, \nu)   \leq \veps, \quad \forall \nu \in \mathcal{P}(\Theta), \quad \forall \pi \in \mathcal{P}(\Z \times \Z) \text{ with } \pi_z = \mu.  \]
$\veps$-Nash equilibria can be interpreted as almost solutions to the game \eqref{eqn:min_max_problem} in the following sense: any deviation from the strategy pair $(\pi^*, \nu^*)$ by one of the players (learner or adversary) will not result in a payoff decrease for the other player of more than $\veps$. In particular, if the learner plays strategy $\nu^*$, then $\nu^*$ is an $\veps$-minimizer of the \textit{robust risk functional} $ F : \nu \in \mathcal{P}(\Theta) \longmapsto \sup_{\pi \in \mathcal{P}(\Z \times \Z) \text{ s.t. } \pi_z =\mu} \U(\pi, \nu )  $. 

\end{remark}

\begin{assumption}
\label{assump:ConvConcav}
  We assume that $\U$ is convex in $\nu$ and concave in $\pi$ in the linear interpolation sense. That is, 
  \[ \U( \tau \pi + (1-\tau ) \hat{\pi},\nu ) \geq  \tau \U( \pi,\nu ) +  (1-\tau) \U( \hat{\pi},\nu )   \]
  and 
   \[ \U(\pi, \tau \nu + (1-\tau ) \hat{\nu} ) \leq  \tau \U( \pi,\nu ) +  (1-\tau) \U(\pi,  \hat{\nu} ),   \]
  for all $\tau \in [0,1]$ and all probability measures $\pi, \hat \pi \in \mathcal{P}(\Z \times \Z)$, and $\nu, \hat \nu \in \mathcal{P}(\Theta)$.
  
  Assuming that $\U$ has the form \eqref{eqn:DefU}, the above conditions are equivalent to analogous convexity-concavity assumptions on $\Risk(\pi,\nu)$, given that $\C$ is linear in $\pi$.
\end{assumption}

\begin{remark}
 The fact that $\U$ is convex-concave according to linear interpolation (i.e. as introduced in Assumptions \ref{assump:ConvConcav}) does not imply that $\U$ is geodesically convex-concave with respect to the geometry that induces the dynamics \eqref{WFR dynamics} (see section \ref{sec:GradDescentAscent} for a discussion on the geometric interpretation of equations \eqref{WFR dynamics}), so that convergence to a global Nash equilibrium or an approximate Nash equilibrium is not immediate. Due to this, despite Assumptions \ref{assump:ConvConcav}, we will refer to the setting considered in this section as the nonconvex-nonconcave setting. We contrast this setting with the one in section \ref{sec:strongCon}, that we will refer to as the nonconvex-concave setting.

\end{remark}

\begin{example}
In the context of the motivating Example \ref{example:2} we see that Assumption \ref{assump:ConvConcav} is satisfied provided that the loss function $\ell$ is a convex function in its first coordinate. This is certainly the case for both the squared-loss and the logistic loss. 
\end{example}

\medskip 

We are ready to state our second main result.

\begin{theorem}[Long time behavior mean-field PDE]
\label{thm:Main2}
Let $\epsilon>0$. Suppose that Assumptions \ref{Hyp U}, \ref{Hyp:Support is stationary}, and \ref{assump:ConvConcav} hold. Assume  that $\nu_0$ and $\pi_0$ are probability measures (with $\pi_{0,z}=\mu$) such that there exists $k>0$ for which $\frac{d \nu_0}{d\theta}(\theta)  > k$, and $\frac{d \pi_0(\tilde z |z)}{d\tilde z} (\tilde z) > k$ for all $z$ in the support of $\mu$. Finally, assume that $\eta_t$ is chosen so that
\[ \lim_{t\rightarrow \infty }\frac{1}{t} \int_0^t\int_0^s \eta_\tau d\tau ds =  \bar \eta \]
for $\bar \eta$ satisfying 
\[  4(L+M)^2 \bar \eta < \epsilon. \]
Then there exists $T^*$ such that for all $t>T^*$
\[  \sup_{\pi^* \in \mathcal{P}(\Z^2) \text{ s.t } \pi^*_z=\mu}\;\U(\pi^*, \bar \nu(t))  - \inf_{\nu^*\in \mathcal{P}(\Theta)}\;\U(\bar \pi(t), \nu^*)  \leq \epsilon,\]
where $\overline{\pi}_t:= \frac{1}{t}\int_0^t \pi_s ds$ and $\overline{\nu}_t:= \frac{1}{t}\int_0^t \nu_s ds$, and $(\pi_t, \nu_t)$ solve \eqref{WFR dynamics}, when initialized at $\pi_0, \nu_0$ as above.
\end{theorem}

\begin{remark}
The assumptions on the initializations $\pi_0$ and $\nu_0$ effectively imply that the particles in Algorithm \ref{Algo}  are well spread out throughout the domains at time zero. This is certainly a strong assumption, but it is not unlike other theoretical assumptions in the literature studying, mathematically, the training process of neural networks; see \cite{chizat_global_2018,domingo2020mean,wang_exponentially_2022,StephanGlobalConvergence,StephanW_WE}.  

In the next section we discuss how the strong assumption on $\pi_0$ can be removed when one restricts the adversarial budget of the adversary in the setting described in section \ref{sec:RobustRegression}.
\end{remark}


\subsubsection{The non-convex and strongly concave case}
\label{sec:strongCon}

 In contrast to Theorem \ref{thm:Main2}, the results we present in this section hold under no assumptions on the initialization $\pi_0$ but at the expense of additional assumptions on the payoff function $\U$. As we discuss below, in settings like the one described in section \ref{sec:RobustRegression} these additional assumptions are not unnatural, as they are linked to the strength given to the adversarial cost functional $\mathcal{C}$ in \eqref{eqn:DefU}.

\begin{assumption}
\label{assump:StrongConcavity}
We assume the following uniform PL (Polyak-Lojasiewicz) condition on the functions $\U(\cdot,\nu)$: There exists $\lambda>0$ such that for all $\nu\in \mathcal{P}(\Theta)$ and all $\pi\in \mathcal{P}(\Z^2)$ with $\pi_z=\mu$ we have
 \[  \int |\nabla_{\tilde z} \U_\pi(\pi,\nu;z,\tilde z)|^2 d\pi(z, \tilde z) \geq \lambda ( m^*_\nu - \U(\pi,\nu) ), \]
where $m^*_\nu:= \sup_{ \tilde \pi \: \text{s.t.} \: \tilde \pi_z=\mu  } \U( \tilde \pi,\nu)$.
\end{assumption}

\begin{remark}
For simplicity, we will refer to the setting when Assumption \ref{assump:StrongConcavity} holds as the strongly concave setting, as it is often the case that one can deduce the PL condition from strong (geodesic) concavity; see Proposition \ref{prop:app0} in Appendix \ref{app0}.  
\end{remark}

\begin{example}
Suppose that the payoff function $\U$ has the form \eqref{eqn:DefU} for $\Risk$ and $\C$ as in \eqref{Ex of h_nu} and \eqref{eq:Cost}, respectively. As we show in Proposition \ref{prop:app0} in Appendix \ref{app0}, if the set $\Z$ is convex (a reasonable assumption in applications), the activation and loss functions are twice continuously differentiable, and, more importantly, the parameter $c_a$ is large enough, then Assumption \eqref{assump:StrongConcavity} is satisfied. 
\end{example}
\nc

To exploit the additional assumptions on $\U(\cdot, \nu)$, it will be useful to consider a slight variation of \eqref{WFR dynamics} where we slow down time in the descent equation and where we remove the scaling factor $\eta$ in the equation for $\pi_t$. Precisely, given $K \geq 1$ we consider the system
\begin{equation}
  \begin{cases}
  \partial_t \nu_t &= \frac{\eta_t}{K} \divergence_{\theta} (\nu_t \nabla_{\theta} \U_{\nu}(\pi_t, \nu_t;\theta)) - \frac{\kappa}{K} \nu_t \left( \U_{\nu}(\pi_{t}, \nu_t;\theta))- \int \U_{\nu}(\pi_t, \nu_t;\theta') d\nu_t(\theta')   \right) 
  \\
  \partial_t \pi_t &= - \divergence_{z,\tilde z} (\pi_t (0, \nabla_{\tilde z} \U_{\pi}(\pi_t, \nu_t; z, \tilde z) ) )+ \kappa \pi_t \left( \U_\pi(\pi_t, \nu_t;z,\tilde z) - \int \U_\pi(\pi_t, \nu_t;z, \tilde z ') d\pi_{t}(\tilde z '| z)      \right),
  \end{cases}
  \label{WFR dynamicsTimeChange}
\end{equation}
initialized at an arbitrary $\pi_0\in \mathcal{P}(\Z^2)$ with $\pi_{0,z}=\mu$, and some $\nu_0$. Well-posedness for this equation under Assumptions \ref{Hyp U} and \ref{Hyp:Support is stationary} can be established as for equation \eqref{WFR dynamics}; we omit the details. To reflect the variations introduced in \eqref{WFR dynamicsTimeChange} in our Algorithm \eqref{Algo:WADA} it suffices to remove the $\eta$ in the update for the variables $\tilde z_{ij}$ and to allow for the for loop over $i,j$ to be repeated a number of times (quantity that can be tuned) before entering the loop over $k$.

We prove the following result.

\begin{theorem}
\label{thm:mainStrongConvex}
Suppose Assumptions \ref{Hyp U}, \ref{Hyp:Support is stationary}, \ref{assump:ConvConcav}, and \ref{assump:StrongConcavity} hold. Assume further that there exists $k>0$ such that $\frac{d \nu_0}{d\theta}  > k$, and let $\pi_0$ be an arbitrary probability measure with $\pi_{0,z}=\mu$. Finally, assume that 
\[ \lim_{t\rightarrow \infty }\frac{1}{t} \int_0^t\int_0^s \eta_\tau d\tau ds =  \bar \eta<\infty .\]

Fix $\epsilon>0$. Then there exists $K_0, r_0, r_1, t_0 >0$ such that, if $K \geq K_0$ and $\overline{\eta}/K \leq r_1$, then for all $t \geq \max\{ t_0 ,  K /r_0   \} $,  we have
\[  \sup_{\tilde \pi \in \mathcal{P}(\Z^2)\text{ s.t. } \tilde \pi _z=\mu}\;\U(\tilde \pi, \bar \nu_t)  - \inf_{\tilde \nu \in \mathcal{P}(\Theta)}\;\U(\bar \pi_t, \tilde\nu)  \leq \epsilon .\]
In the above, $\overline{\pi}_t:= \frac{1}{t}\int_0^t \pi_s ds$ and $\overline{\nu}_t:= \frac{1}{t}\int_0^t \nu_s ds$, and $(\pi_t, \nu_t)$ solve \eqref{WFR dynamicsTimeChange} initialized at $\nu_0, \pi_0$ as above.
\end{theorem}


\subsection{Literature review}
\label{sec:LitReview}

In this section we provide a brief literature review of the topic of adversarial robustness in supervised learning settings, focusing on some developments in recent years. Since the literature in this field has expanded very quickly and spans a variety of disciplines, our review is necessarily non-exhaustive.

The concept of adversarial training, or how to create learning models that are robust to adversarial perturbations of data, gained popularity a few years ago, not long after 
neural networks became the state of the art technology for tackling image processing and natural language processing tasks. Indeed, different researchers noticed that neural network models (as well as others), although highly effective at making accurate predictions on clean data, were quite sensitive to adversarial attacks –see \cite{GoodfellowIntriguing,Goodfellow1,KurakinAdvExamples}.  In order to gain some protection against adversarial attacks, researchers in machine learning and computer sciences have proposed to replace the standard training process of models, which is based on the optimization of a standard loss function objective, with a training that involves the optimization of an objective function that incorporates the actions of a well defined adversary. For example, the paper \cite{madry2018towards} proposes the optimization problem
\begin{equation}
\min_{\theta \in \Theta} \E_{(x,y) \sim \mu} \left[ \sup_{\tilde x \in \mc B_{\veps}(x)} \ell(\theta, (\tilde x , y)) \right].
\label{eq:ALOriginal}
\end{equation} 
In the above, the minimization ranges over the parameters $\theta$ of a learning model (e.g., the parameters of a neural network), and the inner maximization ranges over all possible ``adversarial attacks" $\tilde x$ around a given input data point $x$. $\ell(\theta, (\tilde x , y))$ is the data misfit for the model with parameters $\theta$.  By solving the minimization problem \eqref{eq:ALOriginal}, one expects to produce models  that are robust against the specific attacker modelled by the choice of ball $\mc B_\veps(\cdot)$, i.e., by specifying a ``distance" function, and by setting the adversarial budget $\veps$. In the context of deep neural network models, for example, works such as \cite{YOPO} have discussed different strategies to efficiently implement gradient ascent-descent dynamics (ascent in the data perturbation, descent in the model parameters) to solve a problem like \eqref{eq:ALOriginal}. 

Problems like \eqref{eq:ALOriginal} have motivated recent work in optimization, prompting researchers to renew their interest in analyzing general min-max games in finite dimensional settings (the setting of interest in those papers is more alike to \eqref{eqn:min_max_problem_motiv} rather than \eqref{eqn:min_max_problem}). Some works in this line include \cite{LinJordanNonconvexConcave,MInmaxOptim,MaherMimax,RatliffMinimax}. As discussed in \cite{LinJordanNonconvexConcave,RatliffMinimax}, nonconvex-concave problems are actually not rare in the context of adversarial machine learning.

More general objectives aimed at enforcing robustness of learning models take a form more closely related to \eqref{eqn:min_max_problem} or \eqref{eqn:min_max_problem_motiv}, where the inner maximization is taken over families of distributions, as opposed to considering a point by point maximization as in \eqref{eq:ALOriginal}. All those formulations can be integrated under the general umbrella term of distributionally robust optimization (DRO). The DRO formulation has the advantage of clearly casting adversarial robustness in supervised learning as a minmax game. Several works have explored adversarial training in the DRO setting when considering learning models such as those coming from linear regression (or other classical parametric settings) or neural networks; see \cite{Blanchet2,blanchet2019robust,Kuhn,OPT-026,DROWithPenalty,NEURIPS2019_16bda725,pmlr-v139-meunier21a}.

A strategy that has been considered when analyzing the problem \eqref{eq:ALOriginal}, as well as in the DRO formulations of adversarial robustness, is to replace the inner maximization associated to the adversary's actions with a regularized risk surrogate. For example, in the setting of \eqref{eq:ALOriginal}, this surrogate objective can be formally obtained by expanding the inner maximization objective around $\veps=0$. Some papers that follow this strategy include \cite{ObermanFinlay,GradientRegularization,CurvatureRegularization,SlavinRoss,StructuredGradReg,pmlr-v139-yeats21a}. In their paper \cite{garcia_trillos_regularized_2022}, the authors discuss this strategy for general DRO problems when the learner uses deep neural networks to build classifiers/regression functions; they utilize ideas from control theory to implement the optimization of the resulting surrogate objectives.  Other papers such as \cite{Certified} discuss learning settings where models are trained by substituting the objective \eqref{eq:ALOriginal} with informative upper bound surrogates. The resulting problems are then solved using
ascent-descent algorithms.

A few recent works have discussed adversarial robustness in the setting of agnostic learners (i.e., no modeling assumption for the learner), which can be understood as a limiting case of a problem with a very expressive family of learning models; those settings are useful because they provide lower bounds for more general adversarial robustness problems. Some of those works include: \cite{pydi2019adversarial, BhagojietAl, VarunMuni2,NF222,awasthi2021existence}. In particular, \cite{pydi2019adversarial, BhagojietAl, VarunMuni2} derive and discuss an equivalence between optimal transport problems and \eqref{eq:ALOriginal} when the space of learning models is the set of all measurable binary classifiers. Other works such as \cite{RyanNicolas,LeonNGTRyan,LeonKerrek} connect adversarial robustness in the binary classification setting with geometric variational problems involving perimeter and curvature, and discuss existence of robust classifiers and their regularity; the works \cite{NF222,awasthi2021existence} explore the existence of robust binary classifiers satisfying certain ``certifiability" properties. The work \cite{MOTJakwang} discusses multiclass classification adversarial problems and their equivalence with multimaginal optimal transport problems and (generalized) barycenters in spaces of measures.

\medskip

We would like to end this section by highlighting our contributions in this work, especially in relation to \cite{domingo2020mean,wang_exponentially_2022}, which are the works that are more closely related to ours.  

In the work \cite{domingo2020mean} and the very recent work \cite{wang_exponentially_2022} the authors consider minmax problems with a linear (with respect to the measures) payoff function. Our setting is broader as it covers not only non-linear objectives but also studies the effect of one of the components being pinned to an input function, which allows us to study broad cases of adversaries in the space of measures (DRO version of adversarial training). It is worth remarking that under the simpler setting in \cite{wang_exponentially_2022}, the authors are able to show the exponential convergence (toward an actual Nash equilibrium) of an algorithm with a similar geometric motivation than ours but with an additional entropy term. In its practical implementation, both algorithms look very similar. The convergence in \cite{wang_exponentially_2022} is obtained under similar regularity hypotheses but assuming in addition that the (unique) solution is supported in a discrete measure. Since we do not assume a priori the existence of a unique solution, our results are weaker in terms of convergence rate, as well as on the fact that we can only recover approximate Nash equilibria (at any specified precision).

We emphasize that by not restricting our analysis to payoff function $\U$ that are linear in both variables we can cover a wider variety of settings relevant in the study of adversarial machine learning than previous works in the literature. This gain in generality 
naturally comes at the expense of additional technical challenges. To point at some of these specific challenges, notice that when the payoff is non-linear its first variations are measure dependent, already suggesting the need for a more delicate analysis at the moment of proving the convergence of particle dynamics toward mean-field limits. The difficulties in our analysis are heightened by the presence of conditional distributions in the evolving systems. In order to handle these additional terms we must recur to new ideas and constructions. In the end, the general mean-field analysis that we present can be also combined with lower-semicontinuity arguments to justify certain steps in the second part of the paper, i.e., the analysis of the long-time behavior of the mean-field system, providing in this way alternatives to approaches in the literature that may not be fully justified; we discuss this in section \ref{sec:LongTimeNash} below. 
Moreover, we believe that some of the ancillary results we obtained to support the targeted level of generality of our model may be of interest in their own right.

We also highlight our study of the nonconvex-concave setting delineated in \ref{sec:strongCon}. Indeed, we exploit the specific structure of our adversarial problem and use the strong concavity that comes when considering adversaries with low budget to obtain stronger convergence results toward approximate Nash equilibria of the adversarial problem. Other papers that have explored this setting include \cite{DROWithPenalty}, but the results presented there only guarantee, for the learner, convergence toward stationary points (although it is worth highlighting that they do not work in a mean-field regime). 

In summary, our work is complementary to other papers such as \cite{domingo2020mean,wang_exponentially_2022,DROWithPenalty} (among others). Our results can be viewed as analogue to those in works such as \cite{StephanGlobalConvergence,chizat_global_2018}, which have studied the global convergence of (non-robust) training of shallow neural networks in the mean-field regime.

\nc

\subsection{Outline}

\medskip 

The remainder of the paper is organized as follows. In section \ref{sec:GradDescentAscent} we discuss the gradient ascent-descent interpretation of our Algorithm \ref{Algo}. We review some concepts from optimal transport theory that are necessary to make sense of this interpretation. Section \ref{sec:ConvergenceMeanField} is dedicated to the proofs of Theorem \ref{thm:Main1} and of some corollaries that are later used in section \ref{sec:LongTimeNash} in the proof of our second main result, Theorem \ref{thm:Main2}. In section \ref{sec:nonconvexnonconcave} we present the proof of Theorem \ref{thm:Main2}, and in section \ref{sec:stronglyConcave} the proof of Theorem \ref{thm:mainStrongConvex}. In section \ref{sec:Experiments} we discuss some numerical results of an implementation of our algorithm when used in an actual machine learning task. We wrap up the paper in section \ref{sec:Conclusions}, where we present some conclusions and discuss future directions for research.

\section{Gradient ascent-descent interpretation of algorithm \ref{Algo}}
\label{sec:GradDescentAscent}


Apart from introducing some mathematical definitions and some notation that we will use in the remainder, our purpose in this section is to provide an intuitive geometric motivation for the system of equations \eqref{WFR dynamics} and its discretization in Algorithm \ref{Algo}. Thus, our discussion in this section will be largely formal. Several references motivate the discussion in this section, e.g., \cite{LeonardUnbalanced,LeonardUnbalanced0,CHIZATUnbalanced,PeyreUnbalanced,liero2016optimal}.



\subsubsection{A lifted space}
\label{sec:liftedspace}

In what follows we use $\M_+(\Theta)$ to denote the space of finite positive measures over $\Theta$. We consider a \textit{projection map} $\mathcal{F}: \mathcal{P}(\Theta \times [0,\infty)) \rightarrow \M_+(\Theta) $ which associates to each $\sigma \in \mathcal{P}(\Theta \times [0,\infty)) $ a measure $\mathcal{F} \sigma \in \M_+(\Theta)$ that is determined by the identity
\begin{equation}
 \int \varphi(\theta) d (\F\sigma)(\theta) = \int \alpha \varphi(\theta)d\sigma(\theta,\alpha) 
 \label{sec:FNU}
\end{equation}
for all regular enough test functions $\varphi:\Theta \rightarrow \R $. As we see below, the map $\F$  allows us to lift functionals defined over $\M_+(\Theta)$ to functionals defined over $\mathcal{P}(\Theta \times [0,\infty))$. Once the energy has been lifted, one can consider gradient descent dynamics in the lifted space. In turn, these lifted dynamics can be projected down to the original space of measures to generate an evolution there.   


\begin{remark}
Notice that the function $\mathcal{F}$ is a surjection. Indeed, let $\nu \in \M_+(\Theta)$ and let $M= \nu(\Theta)$, which we first assume is non-zero. Consider the probability measure $\sigma = \frac{\mu}{M} \otimes \delta_{M}$. It is straightforward to show that $\F \sigma = \nu$. In case $M=0$, which means $\nu$ is the measure that is identically equal to zero, we may take $\sigma$ to be any probability measure over $\Theta \times [0,\infty)$ that satisfies $\sigma(\Theta \times \{ 0 \})=1$ to conclude that $\mathcal{F}\sigma = \nu$.

Finally, while $\F$ is surjective, it is worth highlighting that it is far from being one to one.
\end{remark}

To lift a functional $J: \M_+(\Theta) \rightarrow (-\infty, \infty] $ to a functional on $\mathcal{P}(\Theta \times [0, \infty))$, we simply consider the composition of $J$ with the projection map $\F$ as follows:
\begin{equation}
    \mathcal{J}(\sigma  ):= J(\F \sigma ), \quad \sigma \in \mathcal{P}(\Theta \times [0,\infty)).  
    \label{eqn:LiftedEnergyJ}
\end{equation}

Next, we discuss a Riemannian structure for the space $\mathcal{P}(\Theta \times [0, \infty))$.



\subsubsection{A metric on the lifted space}
\label{sec:LiftedSpace}

We start by defining a metric tensor over the space $\Theta \times (0,\infty) $ according to:
\[ g_{(\theta, \alpha)}( (v,s) , (\tilde v, \tilde s) ):= \frac{\alpha}{\eta}\langle  v , \tilde v \rangle + \frac{1}{\kappa\alpha}s \tilde s,     \]
where $\langle \cdot , \cdot \rangle$ denotes the standard inner product in Euclidean space, and $\kappa$ and $\eta$ are two positive parameters. In what follows we use the notation $|(v,s)|_{(\theta, \alpha)}^2 := g_{(\theta, \alpha)}( (v,s) , ( v,  s) )$. 

It is straightforward to verify that the gradient of a scalar function $\phi(\theta, \alpha)$ with respect to the inner product $g$, which we denote by $\overline{\nabla}\phi$, can be written as
\begin{equation}
\overline{\nabla}\phi =( \frac{\eta}{\alpha} \nabla_\theta \phi , \kappa \alpha \partial_\alpha \phi ), 
\label{eq:Gradient}
\end{equation}
where $\nabla_\theta \phi(\theta, \alpha)$ is the usual gradient of $\phi$ in the $\theta$ variable and $\partial_\alpha \phi(\theta, \alpha)$ is the partial derivative of $\phi$ with respect to $\alpha$. Notice that $\overline{\nabla} \phi$ is a vector in $\R^p \times \R$.

Relative to the base metric $g$ in $\Theta \times (0,\infty)$, we define a Wasserstein metric, in dynamic form, over the space of probability measures $\mathcal{P}(\Theta \times [0,\infty)])$. More precisely, for $\sigma , \sigma' \in \mathcal{P}(\Theta \times [0,\infty)) $ we consider
\begin{equation}
\label{eq:WassBenamouBrenier}
   W_{2,g}^2(\sigma , \hat \sigma)= \inf_{ \{(\beta_t, \sigma_t )\}_{t \in [0,1]} \in \text{CE}(\sigma, \tilde \sigma) } \int_{0}^1 \int | \overline \nabla \beta_t(\theta, \alpha) |^2_{\theta, \alpha}  d\sigma_t(\theta, \alpha) dt,  
\end{equation}
where the set $\text{CE}(\sigma, \tilde \sigma)$ consists of all solutions $t\in [0,1]\mapsto (\beta_t , \sigma_t )$ to the (intrinsic) \textit{continuity equation}
\begin{equation}
\begin{cases}
      \partial_t \sigma_t + \odivv (\sigma_t \onabla \beta_t ) = 0,  \\
     \sigma(0) = \sigma, \quad \sigma(1) = \sigma';
\end{cases}     
\label{eq:WFR-1}
\end{equation}
in particular, $\odivv$ denotes the divergence in the space $\Theta \times (0,\infty)$ when endowed with the metric $g$. In general, equation \eqref{eq:WFR-1} has to be interpreted in the weak sense, i.e., it must hold that
\[ \frac{d}{dt} \int \phi(\theta,\alpha ) d\sigma_t(\theta, \alpha) = \int g_{(\theta, \alpha)} ( \onabla \beta _t(\theta, \alpha) , \onabla \phi(\theta, \alpha)  ) d \sigma(\theta, \alpha) \]
  for all  $t \in (0,1)$ and all $ \phi$ regular enough test functions.
  
More than the metric \eqref{eq:WassBenamouBrenier} itself, from formula \eqref{eq:WassBenamouBrenier} we are interested in the implicit formal Riemannian structure that we can endow $\mathcal{P}(\Theta \times [0, \infty))$ with and that can be used to motivate, heuristically, gradient descent or projected gradient descent dynamics in the space $\mathcal{P}(\Theta \times [0, \infty)])$. As is standard when interpreting optimal transport from a Riemannian geometric perspective, one can think of the set $\T_\sigma := \{ \overline \nabla \beta \: \text{ s.t. } \:  \beta: \Theta \times (0,\infty) \mapsto \R \} $ as a formal tangent plane to the formal manifold $\mathcal{P}(\Theta \times [0,\infty))$ at the point $\sigma$, and over this formal tangent plane one can define an inner product $\langle \cdot , \cdot \rangle_{\sigma} $ according to
\[  \langle \onabla \beta , \onabla \beta' \rangle_{\sigma} := 
  \int g_{(\theta, \alpha)} ( \onabla \beta(\theta, \alpha) , \onabla  \beta'(\theta, \alpha) ) d \sigma_t(\theta, \alpha). \]

Before we finish this section, we state a result that we use in the sequel and that allows us to write the continuity equation \eqref{eq:WFR-1} in terms of basic Euclidean divergence and gradient operators.  

\begin{proposition}
\label{prop:EquivalenceCEs}
The intrinsic continuity equation from \eqref{eq:WFR-1} can be written, in terms of the Euclidean divergence $\divv_{\theta, \alpha}$ in $\R^p \times \R$, as
\[ \partial_t \sigma_t + \divv_{\theta, \alpha}(\sigma_t v_{\sigma_t}) =0, \]
where $v_{\sigma}$ is the vector field
\[ v_{\sigma}(\theta, \alpha):= ( \frac{\eta}{\alpha}\nabla_\theta \beta (\theta, \alpha) , \kappa \alpha \partial_\alpha \beta(\theta, \alpha) ).  \]
\end{proposition}

\begin{proof}
This is a consequence of the following simple observation. For all regular enough test functions $\phi$ we have
\begin{align*}
\frac{d}{dt} \int \phi d \sigma_t &= \int g_{(\theta, \alpha)}( \onabla \beta , \onabla \phi   ) d \sigma_t
\\& =  \int \left( \frac{\eta}{\alpha} \nabla_\theta  \phi \cdot \nabla_\theta \beta  + \kappa \alpha \partial_\alpha \phi \partial_\alpha \beta   \right)   d \sigma_t 
\\&= \int \langle  \nabla_{\theta, \alpha} \phi, v_\sigma \rangle    d \sigma_t,
\end{align*}
where in the above we use $\langle \cdot, \cdot \rangle$ to denote the standard Euclidean inner product in $\R^p \times \R$ and $\nabla_{\theta, \alpha} \phi$ to denote the standard gradient in $\R^p \times \R$.

\end{proof}

\subsubsection{Vertical and horizontal vector fields in $\mathcal{P}(\Theta \times [0, \infty))$}
\label{sec:HorizontalVectFields}

We now introduce and discuss some relevant subspaces of the formal tangent plane $T_\sigma$. We will use these subspaces later on.

The horizontal space $T_\sigma^h$ at $\sigma$ is defined as 
\[ T_\sigma^h:=  \{ \onabla \beta \text{ s.t. }  \beta(\theta, \alpha) = \alpha \varphi(\theta) \text{ for some } \varphi   \},  \]
and the vertical space $T_\sigma^v$ as
\[  T_\sigma^v:=  \{ \onabla \beta \text{ s.t. }  \langle \onabla \beta , \onabla \beta'  \rangle_\sigma =0 , \quad \text{for all} \quad \onabla \beta' \in T^h_\sigma \}. \]
The vertical space $T_\sigma^v$ represents the directions that infinitesimally leave $\F \sigma$ invariant, while the horizontal space is $T_\sigma^v$'s orthogonal complement.



Let us denote by $\mathcal{N} = \{ \sigma \: \text{s.t.} \: \F \sigma \in \mathcal{P}(\Theta)  \}$. For $\sigma \in \mathcal{N}$, we consider the subspace $\T_\sigma \mathcal{N}$ of $\T_\sigma$ defined as
\[ \T_\sigma \mathcal{N} := \{\onabla \phi \text{ s.t.} \int \alpha \partial_\alpha \phi(\theta, \alpha) d \sigma =0 \}.\]
The subspace $\T_\sigma \mathcal{N}$ can be interpreted as the space of tangent vectors of all curves passing by $\sigma$ that stay in $\mathcal{N}$.

\begin{remark}
The space $\M_+(\Theta)$ can be endowed with a metric, the Wasserstein-Fisher-Rao metric, that makes the map $\F$ into a Riemannian submersion. Indeed, notice that for two potentials of the form $\alpha \varphi(\theta)$ and $\alpha \varphi'(\theta)$ (i.e., two potentials inducing horizontal vector fields at a point $\sigma$), we have the identity
\[ \langle \onabla (\alpha \varphi)  , \onabla (\alpha \varphi')  \rangle_\sigma  =  \int_{\Theta\times [0,\infty)} \alpha  (\eta \nabla_\theta \varphi \cdot \nabla_{\theta} \varphi'  + \kappa \varphi \varphi'  ) d \sigma(\theta, \alpha) 
 =  \int_{\Theta} (\eta \nabla_\theta \varphi \cdot \nabla_{\theta} \varphi'  + \kappa \varphi \varphi'  ) d \F\sigma(\theta). \]
In other words, the above inner product in fact does not depend on the specific $\sigma$, but only on $\F \sigma$. 

We refer the reader to the references \cite{LeonardUnbalanced,LeonardUnbalanced0,CHIZATUnbalanced,PeyreUnbalanced,liero2016optimal} for details about the Wasserstein-Fisher-Rao geometry. 
\end{remark}




\subsubsection{Gradient flows of lifted energies}


Given an arbitrary energy $\J: \mathcal{P}(\Theta \times [0, \infty)) \rightarrow (-\infty, \infty]$ (not necessarily of the form \eqref{eqn:LiftedEnergyJ}), the gradient (descent) flow of $\J$ with respect to the Riemannian geometry introduced in section \ref{sec:LiftedSpace} takes the form:
\begin{equation}
 \partial_t \sigma_t - \odivv( \sigma_t \onabla \J_{\sigma_t})=0,   
 \label{eq:GradFlowSigma}
\end{equation}
where $\J_\sigma$ is the first variation of $\J$ at the point $\sigma$, defined as we did in the beginning of section \ref{sec:Algorithm}. For more details on the interpretation of \eqref{eq:GradFlowSigma} as a gradient flow see Chapter 8.2 in \cite{VillaniBook}.

In case $\J$ has the structure of a lifted energy as in \eqref{eqn:LiftedEnergyJ}, its first variation can be computed as follows. Let $\sigma, \sigma^*$ and let $\nu = \F \sigma$ and $\nu^* = \F \sigma^*$. Using the linearity of the map $\mathcal{F}$ (which is evident from its definition) we get: 
\begin{align*}
\frac{d}{d\veps} \vert_{\veps=0}  \J(\sigma + \veps(\sigma^* - \sigma))&=   \frac{d}{d\veps} \vert_{\veps=0} J( \F (\sigma + \veps(\sigma^* - \sigma) ) )  
\\& = \frac{d}{d\veps} \vert_{\veps=0} J( \F \sigma + \veps (\F\sigma^* - \F\sigma) )    
\\& = \int_{\Theta} J_{\nu}(\theta) d(\nu^* - \nu )
\\&=\int_{\Theta \times [0, \infty)} \alpha J_{\nu}(\theta) d(\sigma ^* - \sigma ) ,
\end{align*}
where $J_\nu$ is the first variation of $J$ at the point $\nu$. In other words, the first variation of $\J$ at $\sigma$ takes the form $\alpha J_\nu$, where $J_{\nu}$ is the first variation of $J$; this specific form for $\J_\nu$ should not be surprising, since the function $\J$ is constant along vertical vector fields and thus its gradient should be a horizontal vector field. Plugging this expression back in \eqref{eq:GradFlowSigma}, we conclude that the gradient flow of a lifted energy $\J$ takes the form:
\[ \partial \sigma_t - \odivv( \sigma_t \onabla (\alpha \J_{\nu_t} ) )=0;  \quad \F \sigma_t = \nu_t,\]
which, after using Proposition \eqref{prop:EquivalenceCEs}, can also be written as
\begin{equation}
  \begin{cases}
  \partial_t \sigma_t -  \divv_{\theta, \alpha} (\sigma_t v_{\sigma} ) =0; 
  \\v_{\sigma}(\theta, \alpha)  = \left(\eta \nabla_\theta J_{\nu_t} (\theta), \kappa \alpha J_{\nu_t}(\theta)  \right) ; \quad \nu_t = \F(\sigma_t). 
  \end{cases}
  \label{WFR lifted dynamics Aux}
\end{equation}

\subsubsection{Projected gradients}


In general, $\sigma_t$ from \eqref{WFR lifted dynamics Aux} may not belong to $\mathcal{N}$ for $t>0$, even if initialized at a $\sigma_0 \in \mathcal{N}$. If we want to guarantee that $\nu_t= \F \sigma_t \in \mathcal{P}(\Theta)$ for all $t$, we must then project the (Wasserstein) gradient of the energy $\mathcal{J}$ driving the dynamics \eqref{WFR lifted dynamics Aux} onto the subspace $\T_\sigma \mathcal{N}$.

Given $\sigma$ and $\nu=\F \sigma$, we write the potential $\alpha J_\nu$ as
\[ \alpha J_\nu(\theta)= \alpha (J_\nu(\theta) - \int J_\nu(\theta') d\nu (\theta')) + \alpha \int J_\nu(\theta') d\nu (\theta').  \]
A direct computation shows that
\[ \langle  \onabla (  \alpha \int J_\nu(\theta') d\nu(\theta'))  ) , \onabla  \phi(\theta, \alpha) \rangle_{\sigma}  =0,  \]
for all $\onabla \phi \in \T_\sigma \mathcal{N}$; this means that $\onabla (  \alpha \int J_\nu(\theta') d\nu(\theta'))  ) \in \T _\sigma \mathcal{N} ^\perp$. Another direct computation shows that $\onabla (\alpha (J_\nu(\theta) - \int J_\nu(\theta') d\nu (\theta'))) \in \T_\sigma\mathcal{N}$. From this we can then see that $\onabla (\alpha (J_\nu(\theta) - \int J_\nu(\theta') d\nu (\theta')))$ is the projection of $ \onabla (\alpha \J_\nu)$ onto $ \T_\sigma \mathcal{N}  $.

Using Proposition \eqref{prop:EquivalenceCEs}, we can thus conclude that
\begin{equation}
  \begin{cases}
  \partial_t \sigma_t -  \divv_{\theta, \alpha} (\sigma_t v_{\sigma} ) =0; 
  \\v_{\sigma}(\theta, \alpha)  = \left(\eta \nabla_\theta J_{\nu_t} (\theta), \kappa \alpha (J_{\nu_t}(\theta) - \int J_{\nu_t}(\theta')  \nu_t(\theta') ) \right) ; \quad \nu_t = \F(\sigma_t), 
  \end{cases}
\end{equation}
represents projected (onto $\mathcal{N}$) gradient descent dynamics of the lifted energy $\J$.

\subsubsection{An analogous geometric structure for $\mathcal{\M}_+(\Z \times \Z)$}
\label{sec:FPI}

There is a similar geometric structure to the one we discussed in the previous sections that the space $\mathcal{M}_+(\Z \times \Z)$ can be endowed with. In what follows we use $\gamma$ to denote elements in the lifted space $\mathcal{P}(\Z \times \Z \times [0,\infty))$ and represent elements in $\Z \times \Z \times [0,\infty)$ with triplets of the form $(z, \tilde z , \omega)$. The space $\mathcal{P}(\Z \times \Z \times [0,\infty))$ is endowed with a Wasserstein metric just as in \eqref{eq:WassBenamouBrenier}, obtained by changing any appearance of $\theta$ with $(z, \tilde z)$ and any appearance of $\alpha$ with $\omega$. We will use $\mathcal{F}$ (we use the same notation as in section \ref{sec:liftedspace} for simplicity) to denote the projection map $\F: \mathcal{P}(\Z \times \Z \times [0,\infty)) \rightarrow \mathcal{M}_+(\Z \times \Z)$. An arbitrary functional $J: \mathcal{M}_+(\Z \times \Z) \rightarrow (-\infty, \infty]$ can be lifted to $\mathcal{P}(\Z \times \Z \times [0,\infty))$ by composition with $\F$ (we use $\J$ as in the previous sections to denote this composition). The structure of the first variation of $\J$ is $\omega J_{\pi}$, where  $J_\pi$ is the first variation of $J$ at $\pi = \F(\gamma)$. 

Since problem \eqref{min max problem couplings} forces us to restrict to measures $\pi$ with first marginal equal to $\mu$, we consider evolution equations that can be seen as suitable (projected) gradient \textit{ascent} versions of the gradient \textit{ascent} flow of a lifted energy $\J$ w.r.t. the Wasserstein metric discussed above. Such evolution equation takes the form:
\begin{equation}
  \begin{cases}
  \partial_t \gamma_t + \divv_{(z,\tilde z), \omega} (  \gamma_t v_\gamma)=0,
  \\ v_\gamma(z, \tilde z, \omega)  = \left( 0,  \eta \nabla_{\tilde z} J_\pi(z, \tilde z) , \kappa \omega \left( J_\pi(z,\tilde z) - \int J_\pi(z, \tilde z') d\pi_{t}(\tilde z '| z) \right)  \right); \quad 
\pi_t = \F \gamma_t.
  \end{cases}
  \label{WFR lifted dynamicsZZ}
\end{equation}

To motivate the zero in the first component of $v_\gamma(z,\tilde z , \omega)$, suppose that  $ t \mapsto \pi_t$ has the form
\[ \pi_t= \sum_{i,j=1}^N \omega_{ij,t} \delta_{(z_{ij,t}, \tilde{z}_{ij,t})}, \]
where $\pi_t$ solves the evolution equation 
\[ \partial _t \pi_t + \divv_{z, \tilde z}(\pi_t \vec{V}_t)   =0  \]
for some vector field $\vec{V}_t(z, \tilde z) = (V_{1,t}(z, \tilde z), V_{2,t}(z, \tilde z)) $ that changes smoothly in time. We claim that if $\pi^1_t$ is constant in time, then $V_{1,t}$ must be equal to zero at all points in the support of $\pi_t^1$ (and thus of the support of $\pi_0^1$). Indeed, it is enough to notice that if $V_{0,t}({z}_{ij}, \tilde z_{ij})$ was different from $0$, then for all small enough $t>0$ we would have that $z_{ij,t}$ is different from $z_{i'j',0}$ for all $i'j'$, implying that the support of $\pi_t^1$ is different from the support of $\pi_0^1$ for small enough $t>0$. This would contradict the assumption that $\pi_t^1$ was constant in time. 

\subsection{Ascent-descent equations in the lifted space}

 According to our discussion in the previous sections, equation \eqref{WFR lifted dynamics Aux} is analogous to a (projected) gradient descent ODE in Euclidean space of a fixed energy $U_q$
\[\dot{q}_t= -\nabla U_q(q_t),     \]
while equation \ref{WFR lifted dynamicsZZ} is analogous to a gradient ascent equation of the form
\[ \dot{p}= \nabla U_p (p_t),\]
for a fixed energy $U_p$. 

For minmax problems, where there isn't one fixed function to minimize/maximize, but rather, one has a target payoff function for which one wishes to find its saddles, one could consider, at least in the Euclidean case, a system of the form
\[ \begin{cases} \dot{q}_t = -\nabla_q U(q_t, p_t) \\ \dot{p}_t = \nabla_p U(q_t,p_t),\end{cases}\]
 or a projected version thereof in case additional constraints on the variables $p,q$ are present. 
 
 The above inspires the following gradient ascent-descent equations in the space of measures:

 \begin{equation}
  \begin{cases}
  \partial_t \gamma_t &= - \divergence_{(z,\tilde z), \omega} ( \gamma_t v_\gamma(z, \tilde z, \omega)),
  \\ \partial_t \sigma_t &=  \divergence_{\theta, \alpha} (\sigma_t v_{\sigma}(\theta, \alpha) ), 
  \end{cases}
   \label{WFR lifted dynamics}
\end{equation}
where
\[v_\gamma(z, \tilde z, \omega)  = \left( 0,  \eta_t \nabla_{\tilde z} \U_\pi(\pi_t, \nu_t; z, \tilde z) , \kappa \omega \left( \U_\pi(\pi_t, \nu_t; z, \tilde z) - \int \U_\pi(\pi_t, \nu_t; z, \tilde z') d\pi_{t}(\tilde z '| z) \right)  \right),\]

\[  v_{\sigma}(\theta, \alpha)  = \left(\eta_t \nabla_\theta \U (\pi_t, \nu_t;\theta), \kappa \alpha (\U_{\nu}(\pi_t, \nu_t;\theta) - \int \U_{\nu}(\pi_t, \nu_t;\theta')  d\nu_t(\theta') ) \right),  \]
and $\pi_t = \F \gamma_t,\quad \nu_t= \F \sigma_t$.

Notice that here we are allowing the scaling factor $\eta$ to change in time, for convenience. Section \ref{sec:ConvergenceMeanField} is devoted to studying equation \eqref{WFR lifted dynamics}. In particular, we prove well-posedness and show that system \eqref{WFR lifted dynamics} can be recovered as a suitable limit of systems of interacting particles. Looking forward to applications in section \ref{sec:LongTimeNash}, in section \ref{sec:ConvergenceMeanField} we will actually study a sightly more general system than \eqref{WFR lifted dynamics}.

To finally return to the original system \eqref{WFR dynamics} it now suffices to project the dynamics \eqref{WFR lifted dynamics} via the map $\F$, as we discuss next.

\begin{proposition}
Suppose that $(\gamma, \sigma)$ solves the lifted dynamics \eqref{WFR lifted dynamics}. Then the pair $\pi_t=\F \gamma_t, \quad \nu_t= \F \sigma_t$ solves the system \eqref{WFR dynamics}.
\end{proposition}
\begin{proof}
Taking a test function $\phi(\theta)$ we see that
\begin{align*}
 \frac{d}{dt} \int \phi(\theta) d \nu_t = \frac{d}{dt} \int \alpha \phi(\theta) d \sigma_t(\theta, \alpha) & = \eta_t \int \alpha \nabla_\theta \phi(\theta) \cdot \nabla_\theta \U(\pi_t, \nu_t; \theta) d \sigma_t(\theta,\alpha)  
\\& +  \kappa \int   \alpha (\U_{\nu}(\pi_t, \nu_t;\theta) - \int \U_{\nu}(\pi_t, \nu_t;\theta')  d\nu_t(\theta') ) d\sigma_t(\theta, \alpha)
\\& = \eta_t \int  \nabla_\theta \phi(\theta) \cdot \nabla_\theta \U(\pi_t, \nu_t; \theta) d \nu_t(\theta) 
\\&+ \kappa \int   (\U_{\nu}(\pi_t, \nu_t;\theta) - \int \U_{\nu}(\pi_t, \nu_t;\theta')  d\nu_t(\theta') ) d\nu_t(\theta),
\end{align*}
which is the weak form of the second equation in \eqref{WFR dynamics}. The equation for $\pi$ is deduced similarly.
\end{proof}

The bottom line is that, by studying the system \eqref{WFR lifted dynamics} and its approximation with particle systems, we will be implicitly studying the system \eqref{WFR dynamics} and its approximation with particle systems. System \eqref{WFR lifted dynamics}, however, has the advantage of having a direct Lagrangian formulation that we exploit.



\subsection{Notation}

In the sequel, we will use the following notation:
\begin{itemize}
    \item[--] $\mu, \tilde \mu$ probability measures over $\Z$. $\mu$ is the observed data distribution and $\tilde \mu$ represents an adversarial perturbation of $\mu$.
    \item[--] $\pi$ is a measure over $\Z \times \Z$, and we write points in the support of $\pi$ as $(z, \tilde z)$. $z$ can be interpreted as an observed data point, while $\tilde z$ corresponds to a perturbed data point.
   \item [--] $\pi_{z}$ will be used to denote the first marginal of $\pi$, whenever $\pi$ is a probability measure. $\pi(\cdot|z)$, on the other hand will be used to denote the conditional, according to $\pi$, of the second variable given that the first one is equal to $z$.
    
    \item[--] $\gamma$ will represent a probability measure over the lifted space $\Z^2 \times \R_+$.
    
    \item[--] $\nu$ will represent a measure over $\Theta$. 
    
    \item[--] $\sigma$ will denote measures over the lifted space $\Theta \times \R_+$.

\item[--] $\F$ is the projection map from either $\mathcal{P}(\Theta \times \R_+)$ onto $\M_+(\Theta)$, or from $\mathcal{P}(\Z^2 \times \R_+)$ onto $\M_+(\Z^2)$.
    
     \item[--] $\bgamma$ will denote a probability measure over the space $\C([0,T], \Z \times \R_+)$, and $\bsigma$ will be used to denote probability measures over the space $\C([0,T], \Theta \times \R_+)$.  The space $\mc{C}([0,T], \Theta\times \R_+)$ is the space of continuous functions from the interval $[0,T]$ into $\Theta\times \R_+$ and $\mc{C}([0,T], \Z^2\times \R_+)$ is defined analogously. These spaces will be endowed with the metric of uniform convergence. 
    \item [--] We will use $\check{\gamma}$ to represent probability measures over the lifted space $\Z^2\times \R^2_+$ (notice the additional coordinate), and $\check \sigma $ will be used to represent probability measures over the lifted space $\Theta \times \R^2_+$.

    \item  [--] $\check \bgamma$ will denote a probability measure over the space $\C([0,T], \Z \times \R_+^2)$, and $\check \bsigma$ will be used to denote probability measures over the space $\C([0,T], \Theta \times \R_+^2)$. 
     
    \item[--] $\U(\pi, \nu)$ denotes the payoff associated to the measures $\pi$ and $\nu$, and $\U_\pi$ and $\U_\nu$ denote the first variations of $\U$ in the coordinates $\pi$ and $\nu$, respectively.

    \item[--] We will use $\H(\cdot|| \cdot)$ to denote the KL-divergence, or Shannon relative entropy, between two arbitrary probability measures defined over the same space. That is, given $\upsilon, \upsilon'$ probability measures,  $\H(\upsilon'|| \upsilon)$ is defined as $\int \log ( \frac{d \upsilon'}{d \upsilon }) d \upsilon' $, if $\upsilon'\ll \upsilon$, and $+\infty$ otherwise.

\end{itemize}




\section{From particle system to mean-field PDE}
\label{sec:ConvergenceMeanField}

We start this section by introducing an enlarged system of ODEs closely related to the system in \eqref{ODEsIntro}. For $i=1, \ldots, N$, let
\begin{equation}
  \begin{aligned}
    d Z_t^{i} & = 0 \\
    d \tilde Z_t^{i} & = \eta_{ t \nc} \nabla_{\tilde z} \U_\pi (\pi^N_t,\nu^N_t; Z_t^{i},\tilde Z_t^{i} )dt \\
    d \omega_t^{i} & = \kappa \omega_t^{i}  \left( \U_\pi (\pi^N_t,\nu^N_t; Z_t^{i},\tilde Z_t^{i} ) - \int  \U_\pi (\pi_t^N, \nu_t^N; Z_t^{i},\tilde z') d  \pi^N_t(\tilde z'| Z_t^{i} )\right) dt\\
    d \vartheta_t^{i} & = - \eta_{ t \nc} \nabla_{\theta} \U_\nu (\pi^N_t,\nu^N_t;\vartheta_t^{i})dt \\
    d \alpha_t^{i} & = -\kappa \alpha_t^{i}  \left( \U_\nu (\pi^N_N,\nu^N_t;\vartheta_t^{i}) - \int  \U_\nu (\pi_t^N, \nu_t^N;\theta') d \nu^N_t(\theta')\right)dt
    \\ d \beta_t^i & =0
    \\ d\varrho_t^i &=0;
\end{aligned}
\label{ParticleSystem}
\end{equation}
with given initial condition $(Z^i_0,\tilde Z^i_0,\omega^i_0,\vartheta^i_0,\alpha^i_0,\beta^i_0, \varrho_0^i)$ (possibly random)
and
\begin{align}
\begin{split}
   \check{\gamma}_t^N := \frac 1 N \sum_{i=1}^N  \delta_{( Z_t^i, \tilde Z_t^i),\omega^i_t, \beta_t^i},   \quad \gamma_t^N &:= \frac 1 N \sum_{i=1}^N  \delta_{( Z_t^i, \tilde Z_t^i),\omega^i_t}, \quad
   \pi_t^N := \F[\gamma^N_t ]=\frac 1 N \sum_{i=1}^N \omega^i_t \delta_{( Z_t^i, \tilde Z_t^i)}, \\
    \check{\sigma}_t^N := \frac 1 N \sum_{i=1}^N  \delta_{\vartheta_t^i,\omega^i_t,\varrho_t^i}, \quad  \sigma_t^N &:= \frac 1 N \sum_{i=1}^N  \delta_{\vartheta_t^i,\alpha^i_t}, \quad
   \nu_t^N := \F[\sigma^N_t] =\frac 1 N \sum_{i=1}^N \alpha^i_t \delta_{\vartheta_t^i}.
   \end{split}
   \label{eqn:InitialiCondEmpirirical}
\end{align}
The variables $\beta, \varrho \in \R_+$ have been added to the system for convenience. In particular, the extra degrees of freedom that come from the different ways to initialise these variables will come in useful in the second half of section \ref{sec:CorollariesPropChaos}. As can be seen from \eqref{ParticleSystem}, these variables do not affect the time evolution of the remaining variables.

Our main goal in this section is to obtain a propagation of chaos result for a $1$-Wasserstein chaotic initialization of \eqref{ParticleSystem}. Namely, we will assume that, as $N \rightarrow \infty$, we have
\begin{equation}
  W_1(\check \gamma_0^N,\check \gamma_0) \rightarrow 0, \qquad W_1(\check \sigma_0^N,\check \sigma_0) \rightarrow 0,
  \label{Initial chaos}
 \end{equation}
 for some probability measures $\check \gamma_0$ and $\check \sigma_0$. We will actually assume a stronger condition guaranteeing the consistency of certain conditionals at initialization. This extra condition, which we make explicit in \eqref{eq:InitialConditionalsConditionalPROPCHAOS}, will be necessary in our analysis, given the explicit dependence of the dynamics \eqref{ParticleSystem} on conditional distributions. In the above, we use the Euclidean distance in $\Theta \times \R_+^2$ (or in $\Z^2 \times \R_+^2$) to define the Wasserstein distance $W_1$ (see Remark \ref{rem:DefinitionGeneralWass}). 
 Under these assumptions, we will characterize the large $N$ limit of the measures of positions of particles as they evolve in time.

Indeed, in the large $N$ limit, system \eqref{ParticleSystem} is expected to behave like a system where the interactions in \eqref{ParticleSystem} have been replaced by \emph{mean-field} dynamics. Such a system reads as follows. For $i=1, \dots, N$, let
\begin{equation}
  \begin{aligned}
    dZ_t^{mf,i} & = 0\\
    d\tilde Z_t^{mf,i} & = \eta_{ t \nc} \nabla_{\tilde z} \U_\pi (\pi^{mf}_t,\nu^{mf}_t; Z_t^{mf,i},\tilde Z_t^{mf,i} )dt \\
    d \omega_t^{mf,i} & = \kappa \omega_t^{mf,i}  \left( \U_\pi (\pi^{mf}_t,\nu^{mf}_t; Z_t^{mf,i},\tilde Z_t^{mf,i} ) - \int  \U_\pi (\pi_t^{mf}, \nu_t^{mf}; Z_t^{mf,i},\tilde z') d  \pi^{mf}_t(\tilde z'| Z_t^{mf,i} )\right)dt \\
    d\vartheta_t^{mf,i} & = - \eta_{ t \nc} \nabla_{\theta} \U_\nu (\pi^{mf}_t,\nu^{mf}_t;\vartheta_t^{mf,i})dt \\
    d \alpha_t^{mf,i} & = -\kappa \alpha_t^{mf,i}  \left( \U_\nu (\pi^{mf}_N,\nu^{mf}_t;\vartheta_t^{mf,i}) - \int  \U_\nu (\pi_t^{mf}, \nu_t^{mf};\theta') d \nu^{mf}_t(\theta')\right)dt
    \\ d \beta^{mf,i} &=0
    \\ d \varrho^{mf,i} &=0;
\end{aligned}
\label{MFParticleSystem}
\end{equation}
with the same initial conditions as in \eqref{ParticleSystem}, and where $\pi^{mf}_t = \F(\bgamma_t), \nu^{mf}_t=\F(\bsigma_t) $ and $(\bgamma, \bsigma)$ solves \eqref{WFR lifted dynamics} with initial condition $\gamma_0,\sigma_0$ (in section \ref{sec:WellPosednessMeanFieldPDE} we prove the well-posedness of this equation);  we recall that the map $\F$ has been introduced in Section \ref{sec:GradDescentAscent}. The fundamental difference between the mean-field system \eqref{MFParticleSystem} and the original particle system \eqref{ParticleSystem} is that the measures determining the dynamics in \eqref{MFParticleSystem} can be treated as fixed and independent of the evolving particles, while in system \eqref{ParticleSystem} there is an explicit dependence of the driving dynamics on the empirical measures $\pi^N, \nu^N$ associated to the underlying evolving particles. 
\nc

In order to deduce a \textit{propagation of chaos} result for the system \eqref{ParticleSystem}, we need to show that the mean-field system \eqref{MFParticleSystem} is well-defined and that we can control how far the evolutions \eqref{ParticleSystem} and \eqref{MFParticleSystem} are from each other; we show this under Assumptions \ref{Hyp U}. To this aim, it is convenient to introduce some extra mathematical structure that will allow us to use standard analytical arguments: we will work on spaces of measures over continuous paths on $\Z^2 \times \R_+^2$ and $\Theta \times \R_+^2$ and eventually use a fixed point argument to establish well-posedness of a mean-field equation. We start by introducing a family of particle evolutions that will play an instrumental role in our analysis. 
 
 Let us fix $T>0$, and let $\A_T$ be the set of pairs $(\check{\pathp \gamma}, \check{\pathp\sigma}) \in \mc P(\mc{C}([0,T], \Z^2 \times \R_+^2))\times \mc P(\mc{C}([0,T], \Theta\times \R_+^2))$ such that:
\begin{enumerate}
\item $\F {\bsigma}_t$ and $\F {\bgamma}_t$ are probability measures for all $t \in [0,T]$.
\item $\F[\bgamma_t]( \cdot \times \Z) =  \F[\bgamma_0]( \cdot \times \Z) \quad \forall t \in [0,T]$.
\end{enumerate}
 Here, as well as in the remainder, for a given $\check{\bgamma}$ we denote by $\bgamma$ the pushforward of $\check \bgamma$ by the map $\{ (z_t, \tilde z_t, \omega_t, \varrho_t) \} \mapsto \{ (z_t, \tilde z_t, \omega_t) \}$, and abusing notation slightly, in the remainder we may use $\F \check{\bgamma}_t$ and $\F \bgamma_t$ indistinctly; we can analogously relate $\check{\bsigma}$ and $\bsigma$.
  
  Associated to $(\check{\pathp{\gamma}},\check{\pathp{\sigma}})\in \A_{T}$, we consider the multidimensional ODE:
\begin{equation}
  \begin{aligned}
    dZ_t^{\check{\pathp{\gamma}},\check{\pathp{\sigma}}} & = 0\\
    d\tilde Z_t^{\check{\pathp{\gamma}},\check{\pathp{\sigma}}} & = \eta_{ t \nc} \nabla_{\tilde z} \U_\pi (\pi_t,\nu_t; Z_t^{\check{\pathp{\gamma}},\check{\pathp{\sigma}}} ,\tilde Z_t^{\check{\pathp{\gamma}},\check{\pathp{\sigma}}} ) dt \\
    d \omega_t^{\check{\pathp{\gamma}},\check{\pathp{\sigma}}} & = \kappa \omega_t  \left( \U_\pi (\pi_t,\nu_t; Z_t^{\check{\pathp{\gamma}},\check{\pathp{\sigma}}} ,\tilde Z_t^{\check{\pathp{\gamma}},\check{\pathp{\sigma}}} ) - \int  \U_\pi (\pi_t,\nu_t; Z_t^{\check{\pathp{\gamma}},\check{\pathp{\sigma}}} ,\tilde z') d  \pi_t(\tilde z'| Z_t^{\check{\pathp{\gamma}},\check{\pathp{\sigma}}} )\right) dt \\
    d\vartheta_t^{\check{\pathp{\gamma}},\check{\pathp{\sigma}}} & = - \eta_{ t \nc} \nabla_{\theta} \U_\nu (\pi_t,\nu_t;\vartheta_t^{\check{\pathp{\gamma}},\check{\pathp{\sigma}}} ) dt \\
    d \alpha_t^{\check{\pathp{\gamma}},\check{\pathp{\sigma}}} & = -\kappa \alpha_t  \left( \U_\nu (\pi_t,\nu_t;\vartheta_t^{\check{\pathp{\gamma}},\check{\pathp{\sigma}}} ) - \int  \U_\nu (\pi_t,\nu_t;\theta') d  \nu_t(\theta')\right) dt \\
    d\beta^{\check{\pathp{\gamma}},\check{\pathp{\sigma}}}_t & =0
    \\ d\varrho^{\check{\pathp{\gamma}},\check{\pathp{\sigma}}}_t&=0
    \\\pi_t & =\F({{\bgamma}}_t), \qquad \nu_t=\F(\pathp{\sigma}_t),
\end{aligned}
\label{DecoupledEquation}
\end{equation}
with initial conditions
\begin{equation}
\label{eq:SamplingInitialWeights}
((Z_0^{\check{\pathp{\gamma}},\check{\pathp{\sigma}}}, \tilde Z_0^{\check{\pathp{\gamma}},\check{\pathp{\sigma}}}), \omega_0^{\check{\pathp{\gamma}},\check{\pathp{\sigma}}}, \varrho_0^{\check{\bgamma}, \check{\bsigma}}) = ((\xi,\tilde \xi),\omega_0,\varrho_0) \sim \check{\bgamma}_0, \qquad   (\vartheta_0^{\check{\pathp{\gamma}},\check{\pathp{\sigma}}}, \alpha_0^{\check{\pathp{\gamma}},\check{\pathp{\sigma}}}, \beta_0^{\check{\pathp{\gamma}},\check{\pathp{\sigma}}}) = (\vartheta, \alpha_0, \beta_0) \sim \check{\bsigma}_0.
\end{equation}
The fact that $(\check{\bgamma}, \check{\bsigma})\in \A_T$ is (in particular) used to make sense of the term $\pi_t(\cdot | Z_t^{\check \gamma, \check \sigma})$ in equation \eqref{DecoupledEquation}. Indeed, let us denote by $\pi_{t,z}$ the marginal on the $z$ coordinate of $\pi_t$. By assumption, $\pi_{t,z}=\pi_{0,z}$, while $Z_t^{\check \gamma, \check \sigma} =Z_0^{\check \gamma, \check \sigma}$ can be assumed to be in the support of $\pi_{0,z}$ without the loss of generality. The conditional distribution $\pi_{t}(\cdot| Z_t^{\check \gamma, \check \sigma} )$ is thus well defined thanks to the disintegration theorem. 
\nc

Equation \eqref{DecoupledEquation} is a multidimensional classical ODE describing an isolated particle following dynamics driven by an exogenous measure. A key observation is that, under Assumptions \ref{Hyp U}, equation \eqref{DecoupledEquation} is driven by Lipschitz coefficients and so it is well-posed by Caratheodory's existence theorem (see Theorem 5.3 in \cite{Hale1980-mb}). Assumption \ref{Hyp U} and Gronwall's inequality further imply a bound on $\omega^{\check{\pathp{\gamma}},\check{\pathp{\sigma}}}$ and $\alpha^{\check{\pathp{\gamma}},\check{\pathp{\sigma}}}$. We summarize these observations in the next proposition for easy reference.
\begin{proposition} Under Assumption \ref{Hyp U}, there exists a unique solution to \eqref{DecoupledEquation} for any fixed intialization. Moreover, we have
 \[\omega^{\check{\pathp{\gamma}},\check{\pathp{\sigma}}}_t \in  [0,\omega_0 e^{2\kappa M t}], \quad  \alpha^{\check{\pathp{\gamma}},\check{\pathp{\sigma}}}_t \in  [0,\alpha_0 e^{2\kappa M t}] ;\qquad \forall  T \geq t>0.\]
  \label{Prop:control weights}
\end{proposition}

For a given $T>0$, let us now consider the map:
\[\Psi_T: \A_T \mapsto  \mc P(\mc{C}([0,T], \Z^2 \times \R_+^2))\times \mc P(\mc{C}([0,T], \Theta\times \R_+^2))  \]
defined by
\[ \Psi_T(\check{\pathp \gamma},\check{\pathp \sigma}) =(\Psi^1_T(\check{\pathp \gamma},\check{\pathp \sigma}), \Psi^2_T(\check{\pathp \gamma},\check{\pathp \sigma})) :=( \mathrm{Law}[(Z^{\check{\pathp \gamma},\check{\pathp \sigma}}, \tilde Z^{\check{\pathp \gamma},\check{\pathp \sigma}}), \omega^{\check{\pathp \gamma},\check{\pathp \sigma}},  \varrho^{{\check{\pathp \gamma},\check{\pathp \sigma}}}] , \mathrm{Law}[\vartheta^{\check{\pathp \gamma},\check{\pathp \sigma}}, \alpha^{\check{\pathp \gamma},\check{\pathp \sigma}}, \beta^{\check{\pathp \gamma},\check{\pathp \sigma}}]) ,\]
i.e., $\Psi_T$ maps paths in the space of measures in the lifted space to itself.  Moreover, $\Psi_T$ maps $\A_T$ into itself, as we state in the next lemma.

\begin{lemma}
  Under Assumption \ref{Hyp U}, it follows $$\Psi_T(\mc A_T) \subseteq \mc A_T. $$ Moreover, for every $(\check{\bgamma}, \check{\bsigma})\in \A_T $ we have
  \[ \F[(\Psi^1_T(\check{\pathp\gamma}, \check{\pathp \sigma}))_t] (\cdot \times \Z) = \F[\check{\bgamma}_t](\cdot \times \Z). \]
  \label{lem:fixed first marginal}
\end{lemma}
\begin{proof}
This result is immediate from Remark \ref{Rmk:ConstantMarginal} and the fact that $ dZ_t^{\check {\pathp{\gamma}},\check{ \pathp{\sigma}}}=0$. 
\end{proof}

For technical reasons, it will be convenient to introduce a version of the set $\A_T$ whose elements have supports satisfying a certain boundedness condition. Precisely, for a given $T>0$ and $D>0$, we let $\mathcal{A}_{T,D}$ be the set 
\[  \A_{T,D} := \{  (\check{\pathp \gamma}, \check{\pathp\sigma}) \in \A_T \text{ s.t. }  \check{\pathp \gamma}_t (  \Z^2 \times [0,De^{2 \kappa M t}]\times [0,D]  ) = 1, \quad  \check{\pathp \sigma}_t (  \Theta \times [0,De^{2 \kappa M t}] \times [0,D]  ) = 1, \quad \forall t \in [0,T]  \}. \]
In particular, for $(\check{\pathp\gamma},\check{ \pathp \sigma}) \in \A_{T,D} $, the weights $(\omega_0,\varrho_0)$ and $(\alpha_0,\beta_0)$ obtained as in \eqref{eq:SamplingInitialWeights} can be assumed to belong to $[0,D]^2$. Combining with Proposition \ref{Prop:control weights}, we can deduce that for $(\check{\pathp\gamma}, \check{\pathp \sigma}) \in \A_{T,D}$ the weights $\omega_{t}^{\check{\pathp{\gamma}}, \check{\pathp{\sigma}}}, \alpha_{t}^{\check{\pathp{\gamma}}, \check{\pathp{\sigma}}}$ in the dynamics \eqref{DecoupledEquation} can be bounded above by $D e^{2 \kappa M t}$. We summarize this in the following lemma.

\begin{lemma}
For every $T,D>0$ we have $\Psi_T(\A_{T,D}) \subseteq \A_{T,D} $.
\end{lemma}

Under Assumptions \ref{Hyp U}, we prove that we can control the distance between the image $\Psi_T$ of two pairs $(\check{\pathp \gamma}^1, \check{\pathp \sigma}^1 )$ and $(\check{\pathp \gamma}^2, \check{\pathp \sigma}^2 )$ in $\A_{T,D}$ with their own distance. Given $p\geq 1$, we use 
\begin{equation}
  W_{t,p}^p (\upsilon,\upsilon') := \inf_{\Upsilon \in \Gamma(\upsilon,\upsilon')} \int \sup_{s\in [0,t]} |u_s - v_s|^p d\Upsilon(x,y),
  \label{WassersteinPaths}
\end{equation}
to compare two probability measures over the same path space. In particular, we will use these distances to compare measures over paths in any of the lifted spaces.

First, we show a continuity property for the map $\F$ when considering a  restriction of its domain.
\begin{lemma}
    Let $\sigma,\sigma'$ be two probability measures over $\Theta\times [0,D]$, where $D$ is a fixed constant, and suppose that $\F \sigma$ and $\F \sigma'$ are also probability measures.  
    
    Then
    \[ W_p^p(\F(\sigma),\F(\sigma' ) ) \leq  C_{\Theta, p, D} W_p(\sigma,\sigma' ),\]
    where the constant $C_{\Theta, p, D}$ can be written as $C_{\Theta, p, D} = \diam(\Theta)^{p-1} ( \diam(\Theta) + D)$. 
    
    In particular, when its domain has been restricted, the map $\F$ is Lipschitz in the $1$-Wasserstein sense.  
    \label{Lem:F_lipschitz}
\end{lemma}

\begin{proof}
Let us start by noticing that the measures $ \sigma$ for which $\F \sigma$ is a probability measure are precisely the measures satisfying $\int \alpha d\sigma(\theta, \alpha)=1$. 

We first prove the result for $p=1$.

Assume that $\sigma$ and $\sigma'$ take the form $\sigma= \sigma_n=\frac{1}{n}\sum_{i=1}^n \delta_{(\theta_i, \alpha_i)} $ and  $\sigma'= \sigma_n'=\frac{1}{n}\sum_{i=1}^n \delta_{(\theta_i', \alpha_i')} $. It is well known that in that case there exists a permutation $T: \{ 1, \dots, n\} \mapsto \{1, \dots, n \}$ such that $W_1(\sigma_n, \sigma'_n) = \frac{1}{n}\sum_{i=1}^n | (\theta_i,\alpha_i) - (\theta_{T(i)}', \alpha_{T(i)}') |.$ Now, we can write the measures $\F \sigma_n $ and $\F \sigma_n'$ as
\[ \F \sigma_n = \frac{1}{n} \sum_{i=1}^n \min \{\alpha_i, \alpha_{T(i)}' \} \delta_{\theta_i} + \frac{1}{n} \sum_{i=1}^n ( \alpha_i - \min \{\alpha_i, \alpha_{T(i)}') \} \delta_{\theta_i}   \]
and
\[ \F\sigma_n' = \frac{1}{n} \sum_{i=1}^n \min \{\alpha_i, \alpha_{T(i)}' \} \delta_{\theta_i'} + \frac{1}{n} \sum_{i=1}^n ( \alpha_{T(i)}' - \min\{\alpha_i, \alpha_{T(i)}') \} \delta_{\theta_i'}.   \]
Notice that the mass from $ \frac{1}{n} \sum_{i=1}^n \min \{\alpha_i, \alpha_{T(i)}' \} \delta_{\theta_i} $ can be used to cover for the mass demanded in $ \frac{1}{n} \sum_{i=1}^n \min \{\alpha_i, \alpha_{T(i)}' \} \delta_{\theta_i'} $. We carry out the following mass transfer: for each $i$, we send $\min \{\alpha_i, \alpha_{T(i)}'\}$ units of mass from $\theta_i$ to $\theta_{T(i)}'$. The total cost of this mass transfer is $\frac{1}{n} \sum_{i=1}^n \min \{\alpha_i, \alpha_{T(i)}' \} | \theta_i - \theta_{T(i)}' | \leq D W_1( \sigma_n , \sigma_n')   $. Finally, the mass $\frac{1}{n} \sum_{i=1}^n ( \alpha_i - \min \{\alpha_i, \alpha_{T(i)}') \} \delta_{\theta_i}$ can be used to cover for the mass demanded in $\frac{1}{n} \sum_{i=1}^n ( \alpha_{T(i)}' - \min \{\alpha_i, \alpha_{T(i)}') \} \delta_{\theta_i'}$. This mass transfer can be carried out in any way, the important point being that the total cost of such a mass transfer will not be larger than the total amount of mass to be transferred $\frac{1}{n}\sum_{i=1}^n ( \alpha_i - \min \{\alpha_i, \alpha_{T(i)}'\})$ (which is less than $W_1(\sigma_n, \sigma_n')$)  times the diameter of the set $\Theta$. The bottom line is that $W_1(\F \sigma_n, \F \sigma_n') \leq ( D + \diam(\Theta) ) W_1(\sigma_n, \sigma_n) $.

We can extend to arbitrary probability measures $\sigma, \sigma'$ by noticing that: 1) any probability measure $\sigma$ can be approximated in the weak sense by empirical probability measures $\sigma_n$ for growing $n$; 2) the map $\F$ is continuous in the weak sense (as can be verified directly); 3) given that all measures are supported on a fixed bounded set, Wasserstein metrics are continuous with respect to weak convergence. 

Finally, to extend to arbitrary $p\geq 1$, notice that $W_1(\sigma,\sigma') \leq W_p(\sigma, \sigma')$, while $W_p^p(\F \sigma, \F \sigma') \leq \diam(\Theta)^{p-1} W_1(\F \sigma,  \F \sigma')$.
\end{proof}
\begin{remark}
    Lemma \ref{Lem:F_lipschitz} also holds, mutatis mutandis, for $\F$ when it acts on measures $\gamma\in \mc Z^2\times [0,D]$.
\end{remark}

We now deduce an a priori control on the difference between solutions to \eqref{DecoupledEquation} for two different pairs of measures $(\check{\pathp \gamma}^i, \check{\pathp \sigma}^i)$, $i=1,2$.

\begin{lemma}
Let $T,D>0$. Suppose that Assumption \ref{Hyp U} holds. For $i=1,2$, let  $(\check{\pathp{\gamma}}^i, {\check{\pathp{\sigma}}}^i) \in \mc A_{T,D}$, and denote by
$$\zeta^i=(Z^{\check{\pathp{\gamma}}^i,\check{\pathp{\sigma}}^i},\tilde Z^{\check{\pathp{\gamma}}^i,\check{\pathp{\sigma}}^i},\omega^{\check{\pathp{\gamma}}^i,\check{\pathp{\sigma}}^i},\vartheta^{\check{\pathp{\gamma}}^i,\check{\pathp{\sigma}}^i}, \alpha^{\check{\pathp{\gamma}}^i,\check{\pathp{\sigma}}^i},\beta^{\check{\pathp{\gamma}}^i,\check{\pathp{\sigma}}^i}, \varrho^{\check{\pathp{\gamma}}^i,\check{\pathp{\sigma}}^i} )$$
the corresponding evolution determined by \eqref{DecoupledEquation}. We assume that $Z_0^{\check{\bgamma}^i, \check{\bsigma}^i}$ (although possibly random) belongs to the support of $\pi_{0,z}^i$, the marginal on the $z$ coordinate of $\pi_0^i$. We also assume that $\omega_0^i,\varrho_0^i, \alpha_0^i,\beta_0^i \in [0,D].$

Then there exists a constant $K_{T,D}$, depending only on $T$, $D$, the function $\eta$, $\kappa$, and on the constants in Assumption \ref{Hyp U}, such that for all $t \in [0,T]$ we have
\begin{align*}
   & \E|\zeta^1_t-\zeta^2_t| \leq  \E[\sup_{0\leq s\leq t} |\zeta^1_s-\zeta^2_s| ]
    \\& \leq  K_{T,D} \left( \E|\zeta^1_0-\zeta^2_0| +    \int_0^t \{W_1(\check{\bgamma}_s^1,\check{\bgamma}_s^2  ) + W_1(\check{\bsigma}_s^1,\check{\bsigma}_s^2  ) + \E \left( W_1(\pi_s^2(\cdot |Z_0^{\check{\pathp \gamma}^2, \check{\pathp \sigma}^2 } ),\pi_s^1(\cdot |Z_0^{\check{\pathp \gamma}^1, \check{\pathp \sigma}^1 } ) ) \right) \} ds\right).
\end{align*} 
In the above, the expectation is taken over the prescribed (joint) initializations of the two systems.

\label{lem:apriori control zeta}
\end{lemma}

\begin{proof}
  From \eqref{DecoupledEquation} and the Lipschitzness and boundedness conditions in Assumption \ref{Hyp U} we get
  \begin{align*}
    |\frac{d}{dt}(\tilde Z_t^{\check{\pathp{\gamma}}^1,{\check{\pathp{\sigma}}}^1}-\tilde Z_t^{\check{\pathp{\gamma}}^2,\check{\pathp{\sigma}}^2})| 
    & = \eta_{ t \nc} \bigg| \nabla_{\tilde z} \mc U_\pi (\pi_t^1,\nu_t^1; Z_t^{\check{\pathp{\gamma}}^1,\check{\pathp{\sigma}}^1}, \tilde Z_t^{\check{\pathp{\gamma}}^1,\check{\pathp{\sigma}}^1}) -  \nabla_{\tilde z} \mc U_\pi (\pi_t^2,\nu_t^2; Z_t^{\check{\pathp{\gamma}}^2,\check{\pathp{\sigma}}^2}, \tilde Z_t^{\check{\pathp{\gamma}}^2,\check{\pathp{\sigma}}^2}) )  \bigg|\\
    & \leq \eta_{ t \nc} L \{|  Z_t^{\check{\pathp{\gamma}}^1,\check{\pathp{\sigma}}^1}- Z_t^{\check{\pathp{\gamma}}^2,\check{\pathp{\sigma}}^2} | + | \tilde Z_t^{\check{\pathp{\gamma}}^1,\check{\pathp{\sigma}}^1}-\tilde  Z_t^{\check{\pathp{\gamma}}^2,\check{\pathp{\sigma}}^2} | + W_1(\pi^1_t,\pi^2_t) +W_1(\nu^1_t,\nu^2_t)\}. \\
  \end{align*}
  By performing a similar analysis on the other components of the systems, and using the assumption $(\check{\pathp \gamma}^i, \check{\pathp \sigma}^i) \in \A_{T,D}$ and Assumption \ref{Hyp U}, we deduce that we can find a constant $C_{T,D}$ such that for all $t \in[0, T]$
  \begin{align*}
    |\zeta_t^1-\zeta_t^2|  & \leq |\zeta_0^1-\zeta_0^2|
    \\& + C_{T,D} \int_0^t \left\{ |\zeta_s^1-\zeta_s^2|  + W_1(\pi^1_s,\pi^2_s) + W_1(\nu^1_s,\nu^2_s)+ W_1(\pi_s^1(\cdot|  Z_s^{\check{\pathp{\gamma}}^1,\check{\pathp{\sigma}}^1}),\pi_s^2(\cdot|  Z_s^{\check{\pathp{\gamma}}^2,\check{\pathp{\sigma}}^2})  ) \right\} ds.
  \end{align*}
  Thus, using Gronwall's inequality, we get that for all $t \in [0,T]$
  \begin{align}
    |\zeta_t^1-\zeta_t^2| & \leq \sup_{0\leq s \leq t}|\zeta_s^1-\zeta_s^2| \label{eq:gronwall_control}\\
    &  \leq   e^{C_{T,D}T} \left(  |\zeta_0^1-\zeta_0^2| + C_{T,D} \int_0^t \left\{ W_1(\pi^1_s,\pi^2_s) + W_1(\nu^1_s,\nu^2_s)+ W_1(\pi_s^1(\cdot |  Z_s^{\check{\pathp{\gamma}}^1,\check{\pathp{\sigma}}^1}),\pi_s^2(\cdot |  Z_s^{\check{\pathp{\gamma}}^2,\check{\pathp{\sigma}}^2})  ) \right\} ds \right) \notag.
  \end{align}
  Now, from the fact that $Z_s^{\check{\pathp{\gamma}}^i,\check{\pathp{\sigma}}^i} = Z_0^{\check{\pathp{\gamma}}^i,\check{\pathp{\sigma}}^i} $, it follows
  \[ \E \left( W_1(\pi_s^2(\cdot|  Z_s^{\check{\pathp{\gamma}}^2,\check{\pathp{\sigma}}^2}),\pi_s^1(\cdot|  Z_s^{\check{\pathp{\gamma}}^1,\check{\pathp{\sigma}}^1})  ) \right) = \E \left( W_1(\pi_s^2(\cdot |Z_0^{\check{\pathp \gamma}^2, \check{\pathp \sigma}^2 } ),\pi_s^1(\cdot |Z_0^{\check{\pathp \gamma}^1, \check{\pathp \sigma}^1 } ) ) \right). \]
  From this and \eqref{eq:gronwall_control} it then follows that for all $t \in [0,T]$
  \begin{align*}
    &\E[|\zeta_t^1-\zeta_t^2|]  \leq \E[\sup_{0\leq s \leq t}|\zeta_s^1-\zeta_s^2|] \\
      &\leq   e^{C_{T,D}T} \left(  |\zeta_0^1-\zeta_0^2| + C_{T,D} \int_0^t W_1(\pi^1_s,\pi^2_s) + W_1(\nu^1_s,\nu^2_s) +  \E \left( W_1(\pi_s^2(\cdot |Z_0^{\check{\pathp \gamma}^2, \check{\pathp \sigma}^2 } ),\pi_s^1(\cdot |Z_0^{\check{\pathp \gamma}^1, \check{\pathp \sigma}^1 } ) ) \right)  ds \right) .
  \end{align*}
  To conclude, we apply Proposition \ref{Prop:control weights} and Lemma \ref{Lem:F_lipschitz}.
\end{proof}

In general, the terms $ \E \left( W_1(\pi_s^2(\cdot |Z_0^{\pathp \gamma^2, \pathp \sigma^2 } ),\pi_s^1(\cdot |Z_0^{\pathp \gamma^1, \pathp \sigma^1 } ) ) \right)$ cannot be bounded above by the Wasserstein distance between $\pi^2_s$ and $\pi^1_s$. We introduce the following notion, related to the concept of Knothe-transport (see Section 2.3 in \cite{santambrogio2015optimal}, or \cite{Bogachev}), in order to be able to handle this term more directly without making further assumptions on the original measures: Given two collections $ \pathp \pi^1:= \{ \pi_s^{1}\}_{0 \leq s \leq T}$ and $\pathp \pi^2 :=\{ \pi_s^{2}\}_{0 \leq s \leq T}$, we define their cost $\tilde{W}_{t,1}$ by
\begin{equation}
\tilde{W}_{t,1}( \pathp \pi^1 , \pathp \pi^2  ):=   \sup_{s\in[0,t] } \inf_{\upsilon_s \in \Gamma_{\text{Opt}}( \pi_{s,z}^1 , \pi_{s,z}^2   )} \{ \int W_1( \pi_s^2(\cdot| z^2) , \pi_s^1(\cdot|z^1) )       d\upsilon_s(z^1,z^2) \};  
\end{equation}
in the above, we interpret $\pi_{s,z}^i$ as the marginal in the $z$ coordinate of the measure $\pi_s^i$, and $\Gamma_{\text{Opt}}( \pi_{s,z}^1 , \pi_{s,z}^2   )$ is the set of optimal couplings realizing the 1-Wasserstein distance between $\pi_{s,z}^1$ and $ \pi_{s,z}^2 $.

With this notion in hand, we can state and prove the following corollary of Lemma \ref{lem:apriori control zeta}.

\begin{corollary}
  Suppose the assumptions in Lemma \ref{lem:apriori control zeta} hold. Assume further that $\check{\pathp \gamma}_0^1=\check{\pathp \gamma}_0^2$ and $\check{\pathp \sigma}_0^1= \check{\pathp \sigma}_0^2$. Then there exists a constant $\tilde C_{T,D}$ depending only on $M,L,T,D, \kappa, \eta$ such that for all $t\in [0,T]$
  \begin{align*}
 W_{t,1}( \Psi_T^1(\check{\pathp\gamma}^1, \check{\pathp \sigma}^1),\Psi_T^1(\check{\pathp{\gamma}}^2, \check{\pathp \sigma}^2) ) + W_{t,1}( \Psi_T^2(\check{\pathp\gamma}^1, \check{\pathp \sigma}^1),\Psi_T^2(\check{\pathp{\gamma}}^2, \check{\pathp \sigma}^2) )  + \tilde{W}_{t,1} (  \Psi_T^1(\pathp \pi ^1) , \Psi_T^1(\pathp \pi^2)   )
 \\ \leq t \tilde C_{T,D}  \{  W_{t,1}(\check{\pathp \gamma}^1, \check{\pathp \gamma}^2  ) + W_{t,1}(\check{\pathp \sigma}^1,\check{\pathp \sigma}^2   ) + \tilde{W}_{t,1}(\pathp \pi^1, \pathp \pi^2) \}.\end{align*} 
Here we are abusing notation slightly to denote the collection of measures $\{ \F ( (\Psi_T^1( \check{\pathp \gamma}^i, \check{\pathp \sigma}^i  ))_s ) \}_{s \in [0,T]}$ by $\Psi_T^1(\pathp\pi^i)$.

\label{Corollary:control of regularity of Psi}
\end{corollary}

\begin{proof}
Take $\zeta^i$ for $i=1,2$ in Lemma \eqref{lem:apriori control zeta}  with identical initial conditions, sampling $(Z_0^1, \tilde{Z}_0^1, \omega_0^1, \varrho_0^1)$ from $\check{\bgamma}_0$ and $(\vartheta_0^1, \alpha_0^1, \beta_0^1)$ from $\check{\bsigma}_0$. By the definition of the Wasserstein distance it follows
\[ W_{t,1}( \Psi_T^1(\check{\pathp\gamma}^1, \check{\pathp \sigma}^1),\Psi_T^1(\check{\pathp{\gamma}}^2, \check{\pathp \sigma}^2) ) + W_{t,1}( \Psi_T^2( \check{\pathp\gamma}^1, \check{\pathp \sigma}^1),\Psi_T^2(\check{\pathp{\gamma}}^2, \check{\pathp \sigma}^2) )\leq   \E \sup_{0\leq s\leq t}|\zeta^1_s-\zeta^2_s|.\]
Now, since $\check{\pathp \gamma}^1_0= \check{\pathp \gamma}^2_0 $ and $(\check{\pathp \gamma}^1, \check{\pathp \sigma}^1), (\check{\pathp \gamma}^2, \check{\pathp \sigma}^2) \in \A_{T}$, the optimal couplings $\upsilon_s$ in $\tilde{W}_{t,1}(\pathp \pi ^1 , \pathp \pi^2 )$ are all the identity coupling for the measure $\pi^1_{0,z} $; the exact same is true for $\tilde{W}_{t,1} (  \Psi_T^1(\pathp \pi ^1) , \Psi_T^1(\pathp \pi^2)   )$. It follows that
\begin{align*}
 &W_{t,1}( \Psi_T^1(\check{\pathp\gamma}^1, \check{\pathp \sigma}^1),\Psi_T^1(\check{\pathp{\gamma}}^2, \check{\pathp \sigma}^2) ) + W_{t,1}( \Psi_T^2(\check{\pathp\gamma}^1, \check{\pathp \sigma}^1),\Psi_T^2(\check{\pathp{\gamma}}^2, \check{\pathp \sigma}^2) ) \\ & \leq \E \sup_{0\leq s\leq t}|\zeta^1_s-\zeta^2_s| 
\leq K_{T,D}   \int_0^t \{W_{1}( \check{\bgamma}^1_s, \check{\bgamma}_s^2  ) + W_{1}(\check{\bsigma}^1_s, \check{\bsigma}_s^2   ) + \tilde{W}_{s,1}(\pathp \pi^1, \pathp \pi^2)  \} ds  
 \\ & \leq  K_{T,D}   \int_0^t \{W_{s,1}(\check{\pathp \gamma}^1, \check{\pathp \gamma}^2  ) + W_{s,1}(\check{\pathp \sigma}^1,\check{\pathp \sigma}^2   ) + \tilde{W}_{s,1}(\pathp \pi^1, \pathp \pi^2)  \} ds  
 \\ & \leq  t K_{T,D}  \{  W_{t,1}(\check{\pathp \gamma}^1, \check{\pathp \gamma}^2  ) + W_{t,1}(\check{\pathp \sigma}^1,\check{\pathp \sigma}^2   ) + \tilde{W}_{t,1}(\pathp \pi^1, \pathp \pi^2) \}.  
\end{align*}
Likewise,  
\[ \tilde{W}_{t,1} (  \Psi_T^1(\pathp \pi ^1) , \Psi_T^1(\pathp \pi^2)   ) \leq  \E \sup_{0\leq s\leq t}|\zeta^1_s-\zeta^2_s| .   \]
Putting together the above estimates we obtain the desired result.
\end{proof}

\subsection{Well-posedness of mean-field PDE}
\label{sec:WellPosednessMeanFieldPDE}

We now look for a system of ODEs characterizing the solution of the system \eqref{WFR lifted dynamics}. The natural candidate is given by the mean-field equation
\begin{equation}
  \begin{split}
    (Z^{mf}, \tilde Z^{mf}, \omega^{mf}, \vartheta^{mf},\alpha^{mf}, \beta^{mf},\varrho^{mf}):=  (Z^{\check{\bgamma},\check{\bsigma}}, \tilde Z^{\check{\bgamma},\check{\bsigma}}, \omega^{\check{\bgamma},\check{\bsigma}}, \vartheta^{\check{\bgamma},\check{\bsigma}},\alpha^{\check{\bgamma},\check{\bsigma}},\beta^{\check{\bgamma},\check{\bsigma}},\varrho^{\check{\bgamma},\check{\bsigma}})   ;  \\
    \text{with } \quad \check{\bgamma} = \mathrm{Law}[((Z^{mf}, \tilde Z^{mf}), \omega^{mf}, \varrho^{mf} )], \quad  \check{\bsigma} = \mathrm{Law}[(\vartheta^{mf}, \alpha^{mf}, \beta^{mf})].
  \end{split}
  \label{MeanFieldEquation}
\end{equation}
Indeed, assuming that such mean-field equation exists, we verify that setting 
\[ \bgamma = \mathrm{Law}[((Z^{mf}, \tilde Z^{mf}),\omega^{mf})], \text{ and } \bsigma = \mathrm{Law}[(\vartheta^{mf},\alpha^{mf})] \]
we satisfy \eqref{WFR lifted dynamics}. Consider two arbitrary testing functions $\phi, \varphi$. We have
\begin{align*}
   \frac{d}{dt} \E[ \phi(Z^{mf}_t, \tilde Z^{mf}_t),\omega^{mf}_t)] & = \E \left[ \nabla_{(z,\tilde z),\omega} \phi ((Z^{mf}_t, \tilde Z^{mf}_t),\omega^{mf}_t) \cdot [ (\frac{d}{dt}Z^{mf}_t, \frac{d}{dt}\tilde Z^{mf}_t),\frac{d}{dt}\omega^{mf}_t  ]^\top \right],
\end{align*}
and
\begin{align*}
   \frac{d}{dt} \E[ \varphi(\vartheta_t^{mf},\omega^{mf}_t)] & = \E \left[ \nabla_{\theta,\alpha} \varphi (\vartheta_t^{mf},\alpha^{mf}_t) \cdot [ \frac{d}{dt}\vartheta^{mf}_t ,\frac{d}{dt}\alpha^{mf}_t  ]^\top \right].
\end{align*}
Using the dynamics \eqref{DecoupledEquation} with $\pi,\nu$ as defined, we obtain precisely \eqref{WFR lifted dynamics} in a weak sense. Note that, since the measure driving the dynamics comes from the distribution of the dynamics itself, our previous argument does not immediately apply.
However, the matter is settled by establishing the well-posedness of the system of ODEs \eqref{MeanFieldEquation}.

\begin{theorem}
\label{thm:WelPosednessSDE}
Let $D>0$, and suppose that $\check{\gamma}_0$ and $\check{\sigma}_0$ are two probability measures such that  $ \check{\gamma}_0 (  \Z^2 \times [0,D]^2  ) = 1, \quad   \check{\sigma}_0 (  \Theta \times [0,D]^2  ) = 1$, and such that $\F \check \sigma_0$ and $\F \check {\gamma}_0$ are probability measures. Then, under Assumption \ref{Hyp U}, there exists a (pathwise) unique solution to the mean-field system \eqref{MeanFieldEquation} with initial distributions $(\check{\gamma}_0,\check{\sigma}_0)$. 
\end{theorem}
\begin{proof}
 Consider the set
 $ \A_{T,D}(\check{\gamma}_0,\check{\sigma}_0):= \{ (\check{\bgamma}, \check{\bsigma} ) \in \A_{T,D} \text{ s.t. } \check{\bgamma}_0=\check{\gamma}_0 , \quad \check{\bsigma}_0=\check{\sigma}_0 \}$. It is straightforward to see that $\A_{T,D}(\check{\gamma}_0,\check{\sigma}_0)$ endowed with the metric $  W_{T,1}( \check{\bgamma}^1, \check{\bgamma}^2 ) 
 +  W_{T,1}( \check{\bsigma}^1, \check{\bsigma}^2 )$ is a complete metric space. Note also that by Lemma \ref{Corollary:control of regularity of Psi} one can find $T>0$ small enough so that $\Psi$ contracts the quantity $  W_{T,1}( \check{\bgamma}^1, \check{\bgamma}^2 ) 
 +  W_{T,1}( \check{\bsigma}^1, \check{\bsigma}^2 ) + \tilde{W}_{T,1}(\pathp \pi^1, \pathp \pi^2) $ in the space $\A_{T,D}(\check{\gamma}_0,\check{\sigma}_0)$; note that the latter quantity dominates the metric in the space $\A_{T,D}(\check{\gamma}_0,\check{\sigma}_0)$. Hence, there is a unique solution $(\check{\bgamma},\check{\bsigma}) \in \A_{T,D}(\check{\gamma}_0, \check{\sigma}_0)$ to the fixed point equation
\[ \Psi(\check \bgamma,\check\bsigma)=(\check\bgamma,\check\bsigma).   \]
By definition, the mean-field system \eqref{MeanFieldEquation} is then satisfied and is well-posed in the interval $[0,T]$. By continuation, well-posedness can be arbitrarily extended.
\end{proof}

\begin{remark}
Since \eqref{WFR lifted dynamics} is well-defined, we can also conclude that the system of mean-field particles \eqref{MFParticleSystem} is well-defined given that it can be obtained by plugging the mean-field law, except that the initial condition is not sampled from $\check \gamma_0,\check \sigma_0$ but taken as in the system \eqref{ParticleSystem}.
\end{remark}


\red

\medskip








\nc 

\subsection{Propagation of chaos}

Before stating our propagation of chaos result we first present a lemma.

\begin{lemma}
\label{lem:ControlConditionals}
Let $(\check{\gamma}_0,\check{\sigma}_0)$ be such that $\check{\gamma}_0(\Z^2 \times [0,D]^2)=1 $, $\check{\sigma}_0(\Theta \times [0,D]^2) =1$, and such that $\F \check \sigma_0$ and $\F \check {\gamma}_0$ are probability measures. Let $(\check{\bgamma}, \check{\bsigma})$ be the law of \eqref{MeanFieldEquation} with $(\check{\bgamma}_0,\check{\bsigma}_0)=(\check{\gamma}_0,\check{\sigma}_0)$. Let $z_0'$ and $z_0$ be two arbitrary points in the support of $\pi_{0,z}$. Then for every $t \in [0,T]$ we have
\begin{equation}
\sup_{s \in [0,t]} W_1( \pi_s(\cdot | z_0' ), \pi_s(\cdot| z_0 )  )\leq   K_{T,D} (W_1( \pi_0(\cdot| z_0'), \pi_0(\cdot| z_0) ) + |z_0-z_0'|).  
\label{eq:ControlConditionalsInitial}
\end{equation}
\end{lemma}

\begin{proof}
Consider one particle as in \eqref{MFParticleSystem} that we denote by $\zeta$ and that we initialize at $Z_0= z_0$ and $(\tilde{Z}_0, \omega_0, \varrho_0) \sim \check{\bgamma}_0(\cdot| z_0)$ and $(\vartheta_0, \alpha_0, \beta_0) \sim \check{\bsigma}_0$. Likewise, consider another particle as in \eqref{MFParticleSystem} that we denote by $\zeta'$ and that we initialize at $Z'_0= z_0'$, $(\tilde{Z}_0', \omega_0', \varrho_0') \sim \check{\bgamma}_0(\cdot| z_0)$, and $(\vartheta'_0, \alpha'_0, \beta'_0) =(\vartheta_0, \alpha_0, \beta_0)$. At this point we leave unspecified the joint distribution for the initializations of the variables $\tilde{z}, \omega, \varrho$, but it is understood that one such coupling has been fixed in the computations below.   

An application of Lemma \eqref{lem:apriori control zeta} deduces that for every $t \in [0,T]$ 
  \[ \E [\sup_{0\leq s\leq t} |\zeta'-\zeta| ]\leq  K_{T,D} \E|\zeta'_0-\zeta_0| +  K_{T,D}\int_0^t W_1(\pi_s(\cdot |z_0' ),\pi_s(\cdot |z_0 ) )  ds. \]
By definition of the Wasserstein distance, the left hand side of the above expression can be bounded from below by $W_1(\pi_s(\cdot | z_0' ), \pi_s(\cdot| z_0))$ 
for any $s\in [0,t]$, and thus 
\[ \sup_{s \in [0,t]}  W_1(\pi_s(\cdot | z_0' ), \pi_s(\cdot| z_0))  \leq  K_{T,D} \E|\zeta'_0-\zeta_0| +  K_{T,D}\int_0^t  W_1(\pi_s(\cdot |z_0' ),\pi_s(\cdot |z_0 ) )  ds.  \]
By using the fact that the coupling between the distributions for the variables $(\tilde z, \omega, \varrho)$ was arbitrary we can conclude that
\[ \sup_{s \in [0,t]}  W_1(\pi_s(\cdot | z_0' ), \pi_s(\cdot| z_0))  \leq  K_{T,D}(W_1(\pi_0(\cdot|z_0'), \pi_0(\cdot| z_0) ) +|z_0-z_0'|) +  K_{T,D}\int_0^t  W_1(\pi_s(\cdot |z_0' ),\pi_s(\cdot |z_0 ) )  ds.  \]
At this stage we can apply a Gronwall-type argument to obtain the desired result. 

\end{proof}

\begin{theorem}[Propagation of chaos]
Let $T,D>0$, and suppose that Assumption \ref{Hyp U} holds. Let $(\check{\gamma}_0, \check{\sigma}_0)$ be such that $\check{\gamma}_0(\Z^2 \times [0,D]^2)=1$ and $\check{\sigma}_0(\Theta \times [0,D]^2)=1$, and suppose that $\F \check \sigma_0$ and $\F \check {\gamma}_0$ are probability measures.

For $N \in \N$ consider the system \eqref{ParticleSystem} associated to a sequence $\{ (\check{\gamma}_0^N, \check{\sigma}_0^N) \}_{N\in \N}$ satisfying $\check{\gamma}_0^N(\Z^2 \times [0,D]^2)=1$ and $\check{\sigma}_0^N(\Theta \times [0,D]^2)=1$ for all large enough $N$, and suppose that $\F \check \sigma^N_0$ and $\F \check {\gamma}^N_0$ are probability measures. We also assume that the $Z_0^i$ belong to the support of the measure $\pi_{0,z}$.

Assume further that as $N \rightarrow \infty$ we have
\begin{equation}  
\inf_{\upsilon_z \in \Gamma_{\text{Opt}}(\pi_{0,z}^N, \pi_{0,z}) } \int W_1(\check{\gamma}_0^N(\cdot| z_0'),\check{\gamma}_0(\cdot| z_0) ) d \upsilon_z(z_0',z_0) \rightarrow 0, \text{ and }  W_1( \check{\sigma}_0^N, \check{\sigma}_0) \rightarrow 0. 
\label{eq:InitialConditionalsConditionalPROPCHAOS}
\end{equation}

Then
\[ W_{T,1}(\check{\pathp{\gamma}}^N ,\check{\pathp \gamma})\rightarrow 0 , \qquad  W_{T,1}(\check{\bsigma}^N ,\check{\pathp{\sigma}})\rightarrow 0,\quad \tilde{W}_{T,1}({\pathp{\pi}}^N, {\pathp{\pi}})\rightarrow 0,\]
where $(\check{\bgamma},\check{ \bsigma})$ are the laws of the mean-field system \eqref{MeanFieldEquation} with initial conditions drawn from $(\check{\gamma}_0,\check{\sigma}_0)$.
\label{thm:PropChaos}
\end{theorem}
\begin{proof}
From the assumptions, we can assume without the loss of generality that for every $N$ and every $i=1, \dots, N$, the weights $\omega^i_0, \alpha^i_0$ belong to $[0,D]$ . From Gronwall's inequality and Assumptions \eqref{Hyp U} we can then see that the weights $\omega^{i}_t, \alpha_t^i$ belong to $[0,D e^{2M\kappa t}]$ . It follows that $(\check{\bgamma}^N,\check{\bsigma}^N) \in \A_{T, D}$.

In what follows we let $\zeta^i$ denote the path of all variables of the $i$-th particle in the system \eqref{ParticleSystem}, and $\zeta^{mf,i}$ the corresponding particle in \eqref{MFParticleSystem}; we recall that these particles are assumed to be initialized at the same location. We consider  
  \[ \check{\bgamma}_t^{N,mf}:= \frac 1 N \sum_{i=1}^N \delta_{(Z_t^{mf,i},\tilde Z_t^{mf,i} ), \omega_t^{mf,i}, \varrho_t^{mf,i}} \quad \text{and} \quad \check{\bsigma}_t^{N,mf}:= \frac 1 N \sum_{i=1}^N \delta_{\vartheta_t^{mf,i},\alpha_t^{mf,i}, \beta_{t}^{mf,i}},\] that is, the empirical measures of the mean-field system of particles.

 From the triangle inequality we have
  \begin{equation}
    W_{t,1}(\check{\pathp \gamma},\check{\pathp\gamma}^N) + W_{t,1}(\check{\pathp \sigma},\check{\pathp\sigma}^N) \leq \{W_{t,1}(\check{\pathp \gamma}^{N,mf},\check{\pathp\gamma}^N) + W_{t,1}(\check{\pathp \sigma}^{N,mf},\check{\pathp\sigma}^N)\} + \{W_{t,1}(\check{\pathp \gamma},\check{\pathp\gamma}^{N,mf}) + W_{t,1}(\check{\pathp \sigma},\check{\pathp\sigma}^{N,mf})  \}.
    \label{eq: decompose chaos}
  \end{equation}

We now claim that a similar inequality holds for the term $\tilde{W}_{t,1}(\pathp \pi , \pathp \pi^N )$. Namely, we prove that 
\begin{equation}
\tilde{W}_{t,1}(\pathp \pi , \pathp \pi^N) \leq    \tilde{W}_{t,1}(\pathp \pi , {\pathp \pi}^{N,mf}) +   \tilde{W}_{t,1}({\pathp \pi}^{N,mf} , \pathp \pi^N). 
\label{eqn:TriangleIneqWTilde}
\end{equation}
To see this, recall that the $z$ coordinates of all dynamics remain unchanged and that the initializations of $\zeta^i$ and $\zeta^{mf,i}$ are the same. It follows that $\pi_{s,z}^N= \pi_{0,z}^N= {\pi}_{0,z}^{N,mf} ={\pi}_{s,z}^{N,mf}$, and thus $\Gamma_{\text{Opt}}({\pi}_{s,z}^{N,mf}, \pi_{s,z}^N)$ consists exclusively of the identity coupling. Let $\upsilon \in  \Gamma_{\text{Opt}}(\pi_{s,z}, \pi_{s,z}^N) = \Gamma_{\text{Opt}}(\pi_{0,z}, \pi_{0,z}^N) $. From the triangle inequality for $W_1$ we deduce
\begin{align*}
 \int W_1(\pi_s(\cdot| z ),  \pi_s^N(\cdot| z' )  )d \upsilon(z,z') & \leq     \int (W_1( {\pi}_s(\cdot| z )  , {\pi}_s^{N,mf}(\cdot| z' )  ) +  W_1( {\pi}_s^{N,mf}(\cdot| z' )  , \pi_s^N(\cdot| z' )  ) )d \upsilon(z,z')    
 \\& \leq   \int W_1( {\pi}_s(\cdot| z )  , {\pi}_s^{N,mf}(\cdot| z' ))  d \upsilon(z,z')   + \int W_1(  {\pi}_s^{N,mf}(\cdot| z' ), {\pi}_s^N(\cdot| z' )  )  d \pi_{s,z}^N(z')
 \\& \leq   \int W_1( {\pi}_s(\cdot| z )  ,{\pi}_s^{N,mf}(\cdot| z' ))  d \upsilon(z,z')   +  \tilde{W}_{t,1}({\pathp \pi}^{N,mf}, {\pathp \pi}^N ),
\end{align*}
for every $s\in [0,t]$. Taking the $\inf$ over $\upsilon$ on both sides and then the $\sup$ over $s\in [0,t]$, we obtain \eqref{eqn:TriangleIneqWTilde}.

 We use again the fact that the particles $\zeta^i$ and $\zeta^{mf,i}$ have the same initialization to proceed as in the proof of Corollary \ref{Corollary:control of regularity of Psi} and conclude that for every $t\in [0,T]$
 \begin{align*}  
 W_{t,1}( {\check \bgamma}^{N,mf}, {\check \bgamma}^N) + W_{t,1}( \check{\bsigma}^{N,mf}, \check \bsigma^N ) \leq 
   K_{T,D}   \int_0^t \{W_{s,1}({\check \bgamma}, {\check \bgamma}^N  ) + W_{s,1}({\check \bsigma}, \check \bsigma^N    ) + \tilde{W}_{s,1}(\pathp \pi, \pathp \pi^N)  \} ds, 
 \end{align*}
  as well as
  \begin{align*}
\tilde{W}_{t,1}( {\pathp \pi}^{N,mf}, {\pathp \pi}^N) \leq 
   K_{T,D}   \int_0^t \{W_{s,1}(\check{\bgamma}, \check{\bgamma}^N  ) + W_{s,1}(\check{\bsigma}, \check{\bsigma}^N    ) + \tilde{W}_{s,1}(\pathp \pi, \pathp \pi^N)  \} ds.
  \end{align*}
 We can now combine the previous two inequalities with \eqref{eq: decompose chaos} and \eqref{eqn:TriangleIneqWTilde} to conclude that
 \begin{align*}
 W_{t,1}( \check{\bgamma}, \check{\bgamma}^N) + W_{t,1}( \check{\bsigma}, \check\bsigma^N ) &+  \tilde{W}_{t,1}( {\pathp \pi}, {\pathp \pi}^N)  \leq  
 \\ &  \{ W_{t,1}( \check{\bgamma}^{N,mf}, \check{\bgamma}) + W_{t,1}( \check{\bsigma}^{N,mf}, \check \bsigma ) +  \tilde{W}_{t,1}( {\pathp \pi}^{N,mf}, {\pathp \pi}) \} 
  \\&+  K_{T,D}   \int_0^t \{W_{s,1}(\check{\bgamma}, \check{\bgamma}^N  ) + W_{s,1}(\check{\bsigma}, \check\bsigma^N    ) + \tilde{W}_{s,1}(\pathp \pi, \pathp \pi^N)  \} ds.
 \end{align*}
  Combining with Gronwall's inequality, the above implies
  \[ W_{t,1}(\check{\pathp \gamma},\check{\pathp\gamma}^N) + W^1_{t,1}(\check{\pathp \sigma},\check{\pathp\sigma}^N) + \tilde{W}_{t,1}( {\pathp \pi}, {\pathp \pi}^N) \leq e^{ t K_{T,D}}  \{W_{t,1}(\check{\pathp \gamma},{\check{\pathp\gamma}}^{N,mf}) + W_{t,1}(\check{\pathp \sigma},{\check{\pathp\sigma}}^{N,mf}) +  \tilde{W}_{t,1}( \pathp \pi, \pathp \pi^{N,mf}) \}.\]

To complete the proof we must show that the right hand side of the above expression goes to zero as $N \rightarrow \infty$. For that purpose we compare the evolutions of $\zeta_\Z^{mf,i} :=  (Z^{mf,i}, \tilde{Z}^{mf,i}, \omega^{mf,i}, \varrho^{mf,i} )$ and $\zeta_{\Z}^{mf}:=(Z^{mf}, \tilde{Z}^{mf}, \omega^{mf}, \varrho^{mf} )$, and then, separately, compare the evolutions of $\zeta_{\Theta}^{mf,i}:=(\vartheta^{mf,i}, \alpha^{mf,i}, \beta^{mf,i})$ and $\zeta_{\Theta}^{mf}:=(\vartheta^{mf}, \alpha^{mf}, \beta^{mf})$. For the first pair of evolutions, we proceed as in the proof of Lemma \ref{lem:apriori control zeta} to conclude
\[\sup_{0\leq s\leq t} |\zeta^{mf,i}_{\Z,s}-\zeta^{mf}_{\Z,s}| \leq  K_{T,D} |\zeta^{mf,i}_{\Z,0}-\zeta^{mf}_{\Z,0}| +  K_{T,D}\int_0^t  W_1(\pi_s(\cdot |Z_0^{mf} ),\pi_s(\cdot |Z_0^{mf,i } ) ) ds. \]
We can then use Lemma \ref{lem:ControlConditionals} to obtain 
\[\sup_{0\leq s\leq t} |\zeta^{mf,i}_{\Z,s}-\zeta^{mf}_{\Z,s}| \leq  K_{T,D} |\zeta^{mf,i}_{\Z,0}-\zeta^{mf}_{\Z,0}| +  K_{T,D}  W_1(\pi_0(\cdot |Z_0^{mf} ),\pi_0(\cdot |Z_0^{mf,i } )).\]
Combining the above pathwise estimate with the freedom to choose the coupling for the initializations, we can conclude  that
\begin{align*}
W_{T,1}(\check{\bgamma}, \check{\bgamma}^{N,mf}), W_{T,1}({\pathp \pi}, {\pathp \pi}^{N,mf})  &\leq  K_{T,D}  W_{1}(\pi_{0,z}, \pi_{0,z}^{N}) 
\\ & +K_{T,D} \inf_{\upsilon_z \in \Gamma_{\text{Opt}}(\pi_{0,z}^N, \pi_{0,z}) } \int W_1(\check{\gamma}_0^N(\cdot| z_0'),\check{\gamma}_0(\cdot| z_0) ) d \upsilon_z(z_0',z_0).  
\end{align*}
By assumption \eqref{eq:InitialConditionalsConditionalPROPCHAOS}, Remark \ref{rem:App1}, and Lemma \ref{Lem:F_lipschitz}, it follows that the right hand side of the above expression goes to zero as $N \rightarrow \infty$. For the pair of evolutions $\zeta_{\Theta}^{mf,i}$ and $\zeta_{\Theta}^{mf}$ we proceed as in the proof of Lemma \ref{lem:apriori control zeta}, this time noticing that we can write
\[\sup_{0\leq s\leq t} |\zeta^{mf,i}_{\Theta,s}-\zeta^{mf}_{\Theta,s}| \leq  K_{T,D} |\zeta^{mf,i}_{\Theta,0}-\zeta^{mf}_{\Theta,0}|. \]
Combining the above pathwise estimate with the freedom to choose the coupling for the initializations, we can conclude  that
\[ W_{T,1}(\check{\sigma}, \check{\sigma}^{N,mf}) \leq K_{T,D} W_1(\check \sigma_0, \check \sigma_0^N)\rightarrow 0. \]

\nc

  
\end{proof}

\subsection{Corollaries of Theorem \ref{thm:PropChaos}} 
\label{sec:CorollariesPropChaos}

In this section we establish some important results that are implied by Theorem \ref{thm:PropChaos}. The first one is  
Theorem \ref{thm:Main1}. 

\begin{proof}[Proof of Theorem \ref{thm:Main1}]

Let $(\gamma_0, \sigma_0 )$ be such that $\F \gamma_0= \overline{\pi}_0$  and $\F \sigma_0= \overline{\nu}_0$ and such that \eqref{eq:InitialConditionConditionals} holds. We introduce the measures $\check{\gamma}_0:= \gamma_0 \otimes \delta_{1}(d\varrho)$, and $\check{\sigma}_0:= \sigma_0 \otimes \delta_{1}(d \beta)$. That is, $\check{\gamma}_0$ is the product of $\gamma_0$ and a Dirac delta at the value $1$ for the $\varrho$ coordinate; $\check{\sigma}_0$ is defined analogously.  Likewise, we define $\check{\gamma}_0^N := \gamma_0^N \otimes \delta_1$ and $\check{\sigma}_0^N:= \sigma_0^N \otimes \delta_1$. It is clear that with these definitions we have \eqref{eq:InitialConditionalsConditionalPROPCHAOS} and thus we can invoke Theorem \ref{thm:PropChaos} to deduce
\[  W_{T,1}( \check{\bgamma}^N, \check{\bgamma}  ) + W_{T,1}( \check{\bsigma}^N, \check{\bsigma}   ) \rightarrow 0,  \]
where $(\check{\bgamma}^N, \check{\bsigma}^N)$ is the measure in path space induced by the particle system \eqref{ParticleSystem} with initializations as described in the statement of the theorem and with $\beta^i_0 = \varrho_0^i =1$ for all $i$; $(\check{\bgamma}, \check{\bsigma})$, on the other hand, is the law in \eqref{MeanFieldEquation} with intialization
$(\check{\bgamma}_0, \check{\bsigma}_0) = (\check{\gamma}_0, \check{\sigma}_0)$.


Using Lemma \ref{Lem:F_lipschitz}, we conclude that for every $t \in [0,T]$
\[  W_1(\pi_t^N , \pi_t) =  W_{1}( \F \bgamma^N_t , \F \bgamma_t )  \leq  K_{T,D} W_1 ( \bgamma^N_t , \bsigma^N_t )  \leq K_{T,D} W_1 ( \check{\bgamma}_t^N ,\check{\bgamma}_t)  \leq  K_{T,D} W_{T,1} ( \check{\bgamma}^N ,\check{\bgamma}). \]
Likewise,
\[  W_1(\nu_t^N , \nu_t) \leq  K_{T,D} W_{T,1} ( \check{\bsigma}^N ,\check{\bsigma}). \]
Taking the $\sup$ over all $t \in [0,T]$ in the sum of the above two expressions we get
\[ \sup_{t \in [0,T]} \{ W_{1}(\pi_t^N , \pi_t) + W_1(\nu_t^N, \nu_t)   \} \leq K_{T,D} (  W_{T,1} ( \check{\bgamma}^N ,\check{\bgamma}) + W_{T,1} ( \check{\bsigma}^N ,\check{\bsigma}) ), \]
from where the desired result now follows. 

\end{proof}

Corollary \ref{cor:ParticlesHat} and Remark \ref{rem:ApproxRelEntropy} below, which we will use in section \ref{sec:LongTimeNash}, are the other important consequences of Theorem \ref{thm:PropChaos} that we discuss in this section. There, we consider an evolution $\{ (\hat{\nu}_t, \hat{\pi}_t )\}_{t}$  closely related to \eqref{WFR dynamics} that is given by
\begin{equation}
   \partial_t \hat{\nu}_t =  \eta_t\divv( \hat{\nu}_t \nabla_\theta \U_\nu(\pi_t, \nu_t; \theta) ), \quad  \partial_t \hat{\pi}_t =  -\eta_t\divv( \hat{\pi}_t (0,\nabla_{\tilde z} \U_\pi(\pi_t, \nu_t; (z, \tilde z))) ),
   \label{eq:ModifiedDynamicsMF}
\end{equation}
 with initializations $\hat{\nu}_0, \hat{\pi}_0$ that are absolutely continuous with respect to $\nu_0$ and  $\pi_0$, respectively. It is at this stage that we use the extra coordinates $\beta, \varrho$ in \eqref{ParticleSystem}. Indeed, these variables have been introduced to accommodate for the changes of measure between $\hat{\nu}_0 $ and $\nu_0$ and between $\hat{\pi}_0$ and $\pi_0$. We will be able to use the general purpose Theorem \ref{thm:PropChaos} to prove the consistency of particle approximations for the system \eqref{eq:ModifiedDynamicsMF}.     \nc

We start with a preliminary result.

\begin{proposition}
\label{prop:NuHatEmpirical}
Let $\nu_t^N = \frac{1}{N}\sum_{i=1}^n \alpha_i(t) \delta_{\vartheta_i(t)} $ and $\pi_t^N = \frac{1}{N}\sum_{i=1}^N \omega_i(t) \delta_{(Z_i(t), \tilde Z _i(t))} $ be as in \eqref{eqn:InitialiCondEmpirirical}.
Let $\beta_1, \dots, \beta_{N}$ and $\varrho_1, \dots, \varrho_{N}$ be two collections of non-negative scalars satisfying
\[ \frac{1}{N}\sum_{i=1}^N \beta_i \alpha_i(0) =1, \quad \frac{1}{N}\sum_{i=1}^N \varrho_i \omega_i(0) =1.  \] 
Let $\hat{\nu}_t^N$ and $\hat{\pi}_t^N$ be the probability measures defined as
\[ \hat{\nu}_t^N:= \frac{1}{N}\sum_{i=1}^n \beta_i \alpha_i(0) \delta_{\vartheta_i(t)}, \quad \pi_t^N := \frac{1}{N}\sum_{i=1}^N \varrho_i \omega_i(0) \delta_{(Z_i(t), \tilde Z _i(t))}  \quad t \geq 0. \]
Then 
\begin{equation}
   \partial_t \hat{\nu}^N_t =  \eta_t\divv( \hat{\nu}^N_t \nabla_\theta \U_\nu(\pi_t^N, \nu_t^N; \theta) ), \quad  \partial_t \hat{\pi}^N_t =  -\eta_t\divv( \hat{\pi}^N_t (0,\nabla_{\tilde z} \U_\pi(\pi_t^N, \nu_t^N; (z, \tilde z))) )
   \label{eq:ModifiedDynamics}
\end{equation}
in the weak sense. 
\end{proposition}
\begin{proof}
Let $\phi(\theta)$ be an arbitrary test function. From \eqref{ParticleSystem} we see that
\begin{align*}
\frac{d}{dt} \int \phi(\theta) d\hat{\nu}_t^N(\theta) & =  \frac{1}{N} \sum_{i=1}^N \beta_i \alpha_i(0) \frac{d}{dt} \phi(\vartheta_i(t)) = \frac{1}{N} \sum_{i=1}^N \beta_i \alpha_i(0) \nabla  \phi(\vartheta_i(t)) \cdot \dot{\vartheta}_i(t)
\\&= - \frac{\eta_t}{N} \sum_{i=1}^N \beta_i \alpha_i(0) \nabla  \phi(\vartheta_i(t)) \cdot \nabla_\theta \U_\nu(\pi_t^N, \nu_t^N; \vartheta_i(t))
\\&= \eta_t \int \nabla  \phi(\theta) \cdot \nabla_\theta \U_\nu(\pi_t^N, \nu_t^N; \theta)d \hat{\nu}_t^N(\theta).
\end{align*}
This shows that $\hat{\nu}^N$
solves equation \eqref{eq:ModifiedDynamics} in the weak sense. The equation for $\hat{\pi}^N$ is deduced similarly.
\end{proof}

\nc


\nc

We will now proceed to relate \eqref{eq:ModifiedDynamics} with \eqref{eq:ModifiedDynamicsMF}. We first introduce some additional mathematical tools that will help us in this aim. 

Let $\check{\F} : \mathcal{P}(\Z^2 \times \R_+^2) \rightarrow \mathcal{M}_+(\Z^2 ) $ be the map defined via the identity
\[  \int \phi(\theta) d(\check{\F} \check \sigma)(\theta)    = \int \alpha \beta \phi(\theta) d\check \sigma(\theta,\alpha, \beta),    \]
for all test functions $\phi$. Analogously, define $\check{F}$ as a map $\check{\F} : \mathcal{P}(\Theta \times \R_+^2) \rightarrow \mathcal{M}_+(\Theta ) $, substituting any appearance of $\check{\sigma}, \theta, \alpha, \beta$ in the above with $\check{\gamma},(z, \tilde z),\omega, \varrho$. Notice that $\check{\F} \check \sigma $ is a probability measure provided that $\int  \alpha \beta d\check{\sigma}(\theta, \alpha, \beta)=1  $, while an analogous statement holds when $\check \F$ acts on $\check{\gamma}$.


Let us now introduce a map 
$\G: \C([0,T], \Z^2 \times \R_+^2 ) \rightarrow \C([0,T], \Z^2 \times \R_+^2 )$ defined as:
\begin{equation}
 \G : \{ ( z_t, \tilde z _t , \omega_t, \varrho_t) \}_{t \in [0,T] } \mapsto \{ ( z_t, \tilde z _t , \omega_0, \varrho_0) \}_{t \in [0,T] }.
 \label{eqn:DefG}
\end{equation}
 That is, $\G$ is the map that freezes the coordinates $\omega, \varrho$ of a given path, setting them to be equal to their initializations. Naturally, $\G$ induces, via pushforward, a map from $\mathcal{P}(\C([0,T], \Z^2 \times \R_+^2 )) $ into itself; we will abuse notation slightly and will also use $\G$ to denote this induced map. Furthermore, we will also think of $\G$ as a map $\G: \C([0,T], \Theta \times \R_+^2 ) \rightarrow \C([0,T], \Theta \times \R_+^2 )$ that freezes the coordinates $\alpha, \beta$ of a given path, setting them to be equal to their initializations; we will also denote by $\G$ the map induced via pushforward from $\mathcal{P}(\C([0,T], \Theta \times \R_+^2 ))$ into itself. Which of the interpretations for $\G$ will be used in each instance should be clear from context. 

\begin{remark}
 Notice that $\hat{\pi}^N_t$ and $\hat{\nu}_t^N$ in Proposition \ref{prop:NuHatEmpirical} can be written as $\check{\F} ( (\G \check \bgamma^N)_t  ) $ and   $\check{\F} ( (\G \check \bsigma^N)_t  ) $, respectively. 
 \label{rem:HatParticleSystem}
 \end{remark}

\begin{lemma}
\label{lem:HatProcessMeanField}
Let $(\check{\bgamma } , \check{\bsigma})$ be the law of the process \eqref{MeanFieldEquation} initialized at a pair $(\check{\gamma }_0 , \check{\sigma}_0)$. Then $\{ \hat{\nu}_t:= \check \F  (  (\G  \check \bsigma)_t  )  \}_{t \in [0,T]}$ and $\{\hat{\pi}_t:= \check \F  (  (\G  \check \bgamma)_t  )  \}_{t\in [0,T]}$ solve the PDEs \eqref{eq:ModifiedDynamicsMF}, where $\pi_t = \F (\bgamma_t )  $ and $\nu_t= \F(\bsigma_t)$.

\end{lemma}

\begin{proof}
Consider the mean-field ODE \eqref{MeanFieldEquation}. For every smooth test function $\phi$ we have
\[ \int \phi(\theta) d \hat {\nu}_t(\theta) = \int 
\alpha \beta \phi(\theta) d (\G  \sigma )_t(\theta, \alpha, \beta) =   \E [\alpha_0 \beta_0\phi(\vartheta_t) ].   \]
In particular,
\begin{align*}
\frac{d}{dt}\int \phi(\theta) d \hat {\nu}_t(\theta) = \frac{d}{dt}  \E [\alpha_0 \beta_0\phi(\vartheta_t) ] & = -\E [ \eta_t \alpha_0 \beta_0  \nabla \phi(\vartheta_t)\cdot \nabla_\theta \U_\nu (\pi_t,\nu_t; \vartheta_t ) ]
\\& = -\eta_t \int \nabla \phi(\theta) \cdot \nabla_\theta \U_\nu(\pi_t, \nu_t; \theta)d \hat{\nu}_t(\theta).  \end{align*}
This proves that $\hat{\nu}$ satisfies the desired equation. The equation for $\hat{\pi}$ is obtained similarly. 

\end{proof}

\begin{remark}
Notice that $\hat{\nu}_t$ and $\hat{\pi}_t$ are probability measures if $\hat{\nu}_0$ and $\hat{\pi}_0$ are. 
\end{remark}

In what follows we use Theorem \ref{thm:PropChaos} to show that, under appropriate assumptions on initializations, the system in \eqref{eq:ModifiedDynamics} can be recovered from suitable particle approximations.

\begin{corollary}
\label{cor:ParticlesHat}
 
Let $\overline{\nu}_0$ and $\overline{\pi}_0$ be arbitrary, and let $\hat{\nu}_0$ and $\hat{\pi}_0$ be probability measures such that $\hat{\nu}_0 \ll \overline{\nu}_0 $, $\hat{\pi}_0 \ll \overline{\pi}_0 $, with $\frac{d \hat{\nu}_0 }{d\overline{\nu}_0} \in L^\infty(\overline {\nu}_0)$ and $\frac{d \hat{\pi}_0 }
{d\overline{\pi}_0} \in L^\infty(\overline{ \pi}_0)$. Let
$\check{\gamma}_0$ and $\check{\sigma}_0$ be as in Theorem \ref{thm:PropChaos} and additionally assume they satisfy
$\F \check{\gamma}_0 = \overline{\pi}_0$, $ \F \check{\sigma}_0 = \overline{\nu}_0 $, and $\check \F \check{\gamma}_0 = \hat{\pi}_0$, $ \check\F \check{\sigma}_0 = \hat{\nu}_0 $.

Consider approximating particle systems as in Theorem \ref{thm:PropChaos} with the additional assumption that $\hat{\nu}_0^N, \hat{\pi}_0^N$ are probability measures.
Then
\[  \sup_{t\in [0,T] } \{   W_1( \hat{\nu}_t^N , \hat{\nu}_t )  +   W_1( \hat{\pi}_t^N , \hat{\pi}_t )  \} \rightarrow 0,  \quad  \sup_{t\in [0,T] } \{   W_1( {\nu}_t^N , {\nu}_t )  +   W_1( {\pi}_t^N , {\pi}_t )  \} \rightarrow 0,  \]
as $N \rightarrow \infty$. In the above, we use the same notation as in Lemma \ref{lem:HatProcessMeanField} and Remark \ref{rem:HatParticleSystem}.


\nc 


\end{corollary}

\begin{proof}

 First of all, let us notice that the condition $\frac{d \hat{\nu}_0 }{d\overline{\nu}_0} \in L^\infty(\overline {\nu}_0)$ and $\frac{d \hat{\pi}_0 }
{d\overline{\pi}_0} \in L^\infty(\overline{ \pi}_0)$ is used to guarantee that we can indeed build
  $\check{\sigma}_0$ and $\check{\gamma}_0$ with bounded supports; see the first part of Remark \ref{rem:ApproxRelEntropy} below.

It is straightforward to check that $\G$ is a Lipschitz map, i.e., 
\[ W_{T,1}(\G \check \bsigma, \G \check \bsigma' ) \leq 2 W_{T,1} (\check \bsigma, \check \bsigma'), \quad  W_{T,1}(\G \check \bgamma, \G \check \bgamma' ) \leq 2 W_{T,1} (\check \bgamma, \check \bgamma') .   \]
In addition, we can find a constant $C_{T,D}$ such that for every $t \in [0,T]$ 
\[ W_1( \check{\F} (  (\G \check \bsigma )_t   ), \check{\F} (  (\G \check \bsigma^N )_t )   ) \leq C_{T,D} W_1((\G \check \bsigma )_t   ),   (\G \check \bsigma^N )_t ) \leq C_{T,D} W_{T,1}(\G \check \bsigma    ,   \G \check \bsigma^N )  ,\]
where the first inequality follows from a very similar approach to the one in Lemma \ref{Lem:F_lipschitz}. Similarly,
\[ W_1( \check{\F} (  (\G \check \bgamma )_t   ), \check{\F} (  (\G \check \bgamma^N )_t )   ) \leq C_{T,D} W_{T,1}(\G \check \bgamma    ,   \G \check \bgamma^N ).\]

We may now combine the above inequalities with
Theorem \ref{thm:PropChaos}, which allows us to obtain
\[ W_{T,1}( \check{\bgamma}^N , \check{\bgamma}  ) +  W_{T,1}( \check{\bsigma}^N , \check{\bsigma}  )  \rightarrow 0, \]
to deduce the desired convergence.




	\nc

\end{proof}

\begin{remark}[Constructing initializations]
\label{rem:ApproxRelEntropy}
Let $\overline{\pi}_0$ and $\overline{\nu}_0$ be arbitrary, and let $\rho_\nu= \frac{d \hat{\nu}_0 }{d\overline{\nu}_0}$ and $\rho_\pi= \frac{d \hat{\pi}_0 }{d\overline{\pi}_0}$, which we assume satisfy $\rho_{\nu}\in L^\infty(\overline {\nu}_0)$ and $\rho_{\pi}\in L^\infty(\overline{ \pi}_0)$; we further assume that $\hat{\pi}_{0,z} = \pi_{0,z}$. The latter assumption implies that $\int \rho_\pi(z, \tilde z) d\pi_0(\tilde z| z) =1,  $ 
for all $z$ in the support of $\pi_{0,z}$.

Let $\check{\gamma}_0$ and $\check{\sigma}_0$ be the measures $\check{\gamma}_0:=  h_{\gamma\sharp} \overline{\pi}_0$, $\check{\sigma}_0:=  h_{\sigma\sharp} \overline{\nu}_0$, $h_\gamma: (z,\tilde{z})\mapsto (z, \tilde z, 1, \rho_\pi(z,\tilde z))$, $h_\sigma: \theta\mapsto (\theta, 1, \rho_\nu(\theta))$. Notice that $\F \check{\gamma}_0 = \overline{\pi}_0$ and  $\check{\F} \check{\gamma}_0 = \hat{\pi}_0$, while $\F \check{\sigma}_0  = \overline{\nu}_0 $ and $\check{\F} \check{\sigma}_0 = \hat{\nu}_0 $.

We use the same objects and notation as in Remark \ref{rem:Initial} and introduce the extra variables $\beta_{ij}= \rho_{\nu}(\vartheta_{ij})$ and $\varrho_{ij} = \rho_{\pi}( \tilde{Z}_{ij}, Z_{ij})$; notice that the uniform boundedness on $\rho_\nu$ and $\rho_{\pi}$ is imposed to guarantee that the weights $\varrho_{ij}$ and $\beta_{ij}$ are uniformly bounded. Consider the measures 
\[ \check{\gamma}_0^{n,m}:= \frac{1}{nm}\sum_{i=1}^{n}\sum_{j=1}^{m}  \delta_{(Z_{ij}, \tilde{Z}_{ij}, \omega_{ij}, \varrho_{ij} )}, \quad  \check{\sigma}_0^{n,m}:= \frac{1}{nm}\sum_{i=1}^{n}\sum_{j=1}^{m}  \delta_{(\vartheta_{ij}, \alpha_{ij}, \beta_{ij} )}. \]

From Lemma \ref{lem:SatisfyingAssumptions} we can find a sequence $\{(n_k, m_k)\}_{k \in \N}$ such that, almost surely, the induced sequence of pairs $\check{\gamma}^{N_k}_0:= \check{\gamma}^{n_k,m_k}_0$, $\check{\sigma}^{N_k}_0:= \check{\sigma}^{n_k,m_k}_0$ satisfies conditions \eqref{eq:InitialConditionalsConditionalPROPCHAOS}. Moreover, thanks to the law of large numbers and Lemma \ref{lem:Aux2App} in Appendix \ref{app1} this subsequence can be assumed to be such that
\begin{equation}
   \lim_{k \rightarrow \infty } \frac{1}{n_k}\sum_{i=1}^{n_k} \left|  \frac{1}{\frac{1}{m_k} \sum_{j=1}^{m_k} \rho_\pi(Z_{ij},\tilde{Z}_{ij}) } - 1 \right| =0, \quad \lim_{k \rightarrow \infty } \frac{1}{n_k}\sum_{i=1}^{n_k} \left|  \frac{1}{\frac{1}{m_k} \sum_{j=1}^{m_k} \rho_\nu(\vartheta_{ij}) } - 1 \right| =0. 
   \label{eq:AuxApproximatinSequence}
\end{equation}

We make a slight modification to the weights $\varrho_{ij}$ and $\beta_{ij}$, normalizing them so that $\frac{1}{m_k}\sum_{j} \varrho_{ij} =1$ for all $i$, as well as $\frac{1}{n_km_k} \sum_{ij} \beta_{ij} =1$. From \eqref{eq:AuxApproximatinSequence} we can directly show that condition
\eqref{eq:InitialConditionalsConditionalPROPCHAOS} continues to hold after the normalization of weights. The resulting measures $\hat{\pi}_0^{N_k}= \sum_{j} \varrho_{ij} \delta_{(Z_{ij}, \tilde{Z}_{ij})} $ and $\hat{\nu}_0^{N_k}= \sum_{ij} \beta_{ij} \delta_{\vartheta_{ij}} $ can be seen to converge, in the Wasserstein sense, respectively, toward $\hat{\nu}_0$ and $\hat{\pi}_0$, while the measures $\pi_0^{N_k}= \frac{1}{n_k m_k} \sum_{ij}\delta_{(Z_{ij}, \tilde{Z}_{ij})}$ and $\nu_0^{N_k}= \frac{1}{n_k m_k} \sum_{ij}\delta_{\vartheta_{ij}}$ converge toward $\overline{\pi}_0$ and $\overline{\nu}_0$, respectively. 

Moreover, another application of the law of large numbers implies that  \[\H(\hat{\nu}_0^{N_k} || {\nu}_0^{N_k})   = (\frac{1}{n_km_k}\sum_{ij} \rho_{\nu}(\vartheta_{ij}))^{-1} \frac{1}{n_km_k}\sum_{ij} \log(\rho_\nu(\vartheta_{ij})) \rho_\nu(\vartheta_{ij}) - \log\left( \frac{1}{n_km_k}\sum_{ij} \rho_\nu(\vartheta_{ij}) \right)  \]   
converges, as $k \rightarrow \infty$, toward $\int_\Theta \log(\rho_\nu(\theta)) \rho_\nu(\theta) d\overline{\nu}_0(\theta)$, 
which is precisely $\H(\hat \nu_0 || \overline{\nu}_0) $. Likewise, we can see that $\H(\hat{\pi}_0^{N_k} || {\pi}_0^{N_k})  \rightarrow  \H(\hat{\pi}_0 || \overline{\pi}_0)$, as $k \rightarrow \infty$. 

The above convergence of relative entropies will be used in the next section.

\end{remark}
\nc

\nc

\section{Long term behavior of mean-field equation and approximate Nash equilibria of \eqref{min max problem couplings}}
\label{sec:LongTimeNash}

In this section we study the long time behavior of the system of equations \eqref{WFR dynamics} initialized at arbitrary measures $(\pi_0, \nu_0)$. Our aim is to study the ability of system \eqref{WFR dynamics} to generate approximate Nash equilibria for problem \eqref{min max problem couplings}. We separate our discussion into two distinctive cases: 1) a rather general non-convex non-concave setting, and 2) a non-convex concave setting. Recall that by non-convex/non-concave we really mean not \textit{geodesically} convex/concave relative to certain optimal transport geometry driving the dynamics, while we do assume convexity/concavity in the linear interpolation sense as in Assumption \ref{assump:ConvConcav}. 

To begin our analysis, we first discuss the relationship between the system \eqref{WFR dynamics} and an associated hat process as in Lemma \ref{lem:HatProcessMeanField}. The study of similar systems has been considered in works like \cite{domingo2020mean}. However, here we present an alternative approach that allows us to fully justify our derivations; see Remark \ref{rem:Entropy_ComparisonLiterature} below for more details. Our approach makes use of the larger structure that we studied in section \ref{sec:ConvergenceMeanField}, and, specifically, the content of Remark \ref{rem:ApproxRelEntropy}. Indeed, we will use the particle approximation in Remark \ref{rem:ApproxRelEntropy} to understand the time evolution of the relative entropy between $\hat{\nu}$ and $\nu$, and $\hat{\pi}$ and $\pi$, for arbitrary initializations. As a first step, we study the time evolutions of relative entropies when the measures $(\nu^{N_k}_t, \pi^{N_k}_t)$ and $(\hat \nu^{N_k}_t, \hat \pi^{N_k}_t)$ are initialized at empirical measures as in Remark \ref{rem:ApproxRelEntropy}.

\begin{proposition}
\label{prop:RelativeEntropyEmpirical}
Let $\overline{\pi}_0$ and $\overline{\nu}_0$ be arbitrary, and let $\hat{\pi}_0$ and $\hat{\nu}_0$ be as in Remark \ref{rem:ApproxRelEntropy}. For a fixed $k \in \N$, let $\nu_t^{N_k},\hat{\nu}_t^{N_k}, \pi_t^{N_k}, \hat{\pi}_t^{N_k}$ be as in Proposition \ref{prop:NuHatEmpirical} when initialized as in Remark \ref{rem:ApproxRelEntropy}. 

Then
\begin{align}
\frac{d}{dt} \H(\hat{\nu}_t^{N_k} \lVert \nu_t^{N_k} )= \kappa\int_\Theta \U_\nu(\pi_t^{N_k}, \nu_t^{N_k}; \theta) d ( \hat{\nu}_t^{N_k} - \nu_t^{N_k})
\label{eqn:RelativeEntropyEmpirical}
\end{align}
and
\begin{align*}
\frac{d}{dt} \H(\hat{\pi}_t^{N_k} \lVert \pi_t^{N_k} )=-\kappa \int_{\Z \times \Z} \U_\pi(\pi_t^{N_k}, \nu_t^{N_k}; z,\tilde z) d ( \hat{\pi}_t^{N_k} - \pi_t^{N_k}).
\end{align*}
\end{proposition}
\begin{proof}
Notice that
\begin{align*}
\frac{d}{dt} \H(\hat{\nu}_t^{N_k} \lVert \nu_t^{N_k} ) &= \frac{d}{dt} \left(\frac{1}{n_km_k} \sum_{i=1}^{n_k} \sum_{j=1}^{m_k}\log\left(\frac{\beta_{ij}(0)\alpha_{ij}(0)}{ \alpha_{ij}(t)}\right) \beta_{ij}(0) \alpha_{ij}(0) \right)
\\ & = - \frac{1}{n_km_k}\sum_{i=1}^{n_k}\sum_{j=1}^{m_k}   \frac{d}{dt} \log(\alpha_{ij}(t)){\beta}_{ij}(0) \alpha_{ij}(0)
\\&= \frac{\kappa}{n_km_k}\sum_{i=1}^{n_k} \sum_{j=1}^{m_k}(\U_\nu(\pi_t^{N_k}, \nu_t^{N_k}; \vartheta_{ij}(t)) -\overline{\U}_\nu) \beta_{ij}(0) \alpha_{ij}(0),
\end{align*}
where to go from the second to the third line we have used equation \eqref{ParticleSystem} for $\alpha_{ij}(t)$. We have also used the shorthand notation $\overline{\U}_\nu= \int_\Theta  \U_\nu(\pi_t^{N_k}, \nu_t^{N_k}; \theta) d \nu_t^{N_k}(\theta).$ Identity \eqref{eqn:RelativeEntropyEmpirical} now follows. 

The identity for $\frac{d}{dt} \H(\hat{\pi}_t^N \lVert \pi_t^N )$ follows from similar considerations, but now we rely on the fact that the weights $\varrho_{ij}$ are normalized along every row: 
\begin{align*}
\frac{d}{dt} \H(\hat{\pi}_t^{N_k} \lVert \pi_t^{N_k} ) &= \frac{d}{dt} \left(\frac{1}{n_km_k} \sum_{i=1}^{n_k} \sum_{j=1}^{m_k}\log\left(\frac{\varrho_{ij}(0)\omega_{ij}(0)}{ \omega_{ij}(t)}\right) \varrho_{ij}(0) \omega_{ij}(0) \right)
\\ & = - \frac{1}{n_km_k}\sum_{i=1}^{n_k}\sum_{j=1}^{m_k}   \frac{d}{dt} \log(\omega_{ij}(t)){\varrho}_{ij}(0) \omega_{ij}(0)
\\&= - \frac{\kappa}{n_km_k}\sum_{i=1}^{n_k} \sum_{j=1}^{m_k}(\U_\pi(\pi_t^{N_k}, \nu_t^{N_k}; Z_{ij}, \tilde{Z}_{ij}) -\overline{\U}_{\pi, i}) \varrho_{ij}(0) \omega_{ij}(0).
\end{align*}
In the above we have used the shorthand notation $\overline{\U}_{\pi, i}= \int_{\Z \times \Z}  \U_\pi(\pi_t^{N_k}, \nu_t^{N_k}; Z_{ij}, \tilde z) d \pi_t^{N_k}( \tilde z | Z_{ij})$;
recall that in our construction $Z_{ij}$ does not depend on $j$. 
\end{proof}
\nc

Next, we add one ingredient to the approximation result from Corollary \ref{cor:ParticlesHat} in search of a relationship similar to \eqref{prop:RelativeEntropyEmpirical} but for general initializations.

\begin{proposition}
\label{prop:ConvergenceRelEntropyUpperBound}

Let $\overline{\pi}_0$ and $\overline{\nu}_0$ be arbitrary, and let $\hat{\pi}_0$ and $\hat{\nu}_0$ be as in Remark \ref{rem:ApproxRelEntropy}. Let $(\hat{\nu}, \hat{\pi})$ be the dynamics in Lemma \ref{lem:HatProcessMeanField} when initialized as in Remark \ref{rem:ApproxRelEntropy}. For every $k \in \N$, let $\nu_t^{N_k},\hat{\nu}_t^{N_k}, \pi_t^{N_k}, \hat{\pi}_t^{N_k}$ be as in Proposition \ref{prop:NuHatEmpirical} when initialized as in Remark \ref{rem:ApproxRelEntropy}.

Then
\begin{equation}
\lim_{k \rightarrow \infty} \int \U_\nu (\pi_s^{N_k}, \nu_s^{N_k}; \theta ) d(\hat \nu_s^{N_k} - \nu_s^{N_k})   =   \int \U_\nu (\pi_s, \nu_s; \theta ) d(\hat \nu_s - \nu_s) 
\label{eq:Aux}
\end{equation}
as well as
\begin{equation}
\lim_{k \rightarrow \infty} \int_{\Z \times \Z} \U_\pi (\pi_s^{N_k}, \nu_s^{N_k}; z, \tilde z) d(\hat \pi_s^{N_k} - \pi_s^{N_k})  =   -\int_{\Z \times \Z} \U_\pi (\pi_s, \nu_s; z, \tilde z ) d(\hat \pi_s - \pi_s). 
\label{eq:Aux2}
\end{equation}

\end{proposition}

\begin{proof}
From Assumptions \ref{Hyp U} and Corollary \ref{cor:ParticlesHat} we have
\[ \left| \int \U_\nu (\pi_s^{N_k}, \nu_s^{N_k}; \theta )  d(\hat \nu_s^{N_k} - \nu_s^{N_k})  - \int \U_\nu (\pi_s, \nu_s; \theta ) d(\hat \nu_s^{N_k} - \nu_s^{N_k})  \right| \leq   L (W_1(\nu_s, \nu_s^{N_k}) + W_1(\pi_s, \pi_s^{N_k})) \rightarrow 0, \]
as $k \rightarrow \infty$. On the other hand, since the function $\U_\nu(\pi_s, \nu_s, \cdot)$ is continuous with at most linear growth in $|\theta|$, and since $W_1({\nu}_s^{N_k}, {\nu}_s ) \rightarrow 0$,  $W_1(\hat{\nu}_s^{N_k}, \hat{\nu}_s ) \rightarrow 0$ as $k \rightarrow \infty$ by Corollary \ref{cor:ParticlesHat}, it follows that
\[ \lim_{k \rightarrow \infty }\left| \int \U_\nu (\pi_s, \nu_s; \theta )  d(\hat \nu_s^{N_k} - \nu_s^{N_k})  - \int \U_\nu (\pi_s, \nu_s; \theta ) d(\hat \nu_s - \nu_s)  \right|=0.  \]
Equation \eqref{eq:Aux} readily follows. \eqref{eq:Aux2} is obtained similarly.

\end{proof}

\nc

\begin{proposition}
\label{cor:EntropyBounds}

Let $\overline{\pi}_0$ and $\overline{\nu}_0$ be arbitrary, and let $\hat{\pi}_0$ and $\hat{\nu}_0$ be as in Remark \ref{rem:ApproxRelEntropy}. Let $(\hat{\nu}, \hat{\pi})$ be the dynamics in Lemma \ref{lem:HatProcessMeanField} when initialized as in Remark \ref{rem:ApproxRelEntropy}.

Then the following inequalities hold:
\begin{equation}
   \H(\hat{\nu}_t || \nu_t) - \H(\hat{\nu}_0|| \overline{\nu}_0) \leq \kappa \int_{0}^t \left(    \int \U_\nu (\pi_s, \nu_s; \theta ) d(\hat \nu_s - \nu_s) (\theta)  \right) ds, \quad \forall t \geq 0, \label{eqn:EntropyBound_nu} 
\end{equation}

and

\begin{equation}
   \H(\hat{\pi}_t || \pi_t) - \H(\hat{\pi}_0 || \overline{\pi}_0)  \leq  -\kappa \int_{0}^t \left(    \int \U_\pi (\pi_s, \nu_s; z,\tilde z ) d(\hat \pi_s - \pi_s) (z, \tilde z)  \right) ds, \quad \forall t \geq 0. 
   \label{eqn:EntropyBound_pi}
\end{equation}

\end{proposition}

\begin{proof}

 For every $k \in \N$, consider $\nu_t^{N_k},\hat{\nu}_t^{N_k}, \pi_t^{N_k}, \hat{\pi}_t^{N_k}$ be as in Proposition \ref{prop:NuHatEmpirical} when initialized as in Remark \ref{rem:ApproxRelEntropy}. Notice that thanks to Corollary \ref{cor:ParticlesHat} we have $W_1({\nu}_s^{N_k}, {\nu}_s ) \rightarrow 0$, $W_1(\hat{\nu}_s^{N_k}, \hat{\nu}_s ) \rightarrow 0$, as $k \rightarrow \infty$.
 

From Proposition \eqref{prop:RelativeEntropyEmpirical} we have 
\[
\H(\hat{\nu}_t^{N_k} \lVert \nu_t^{N_k} )  = \H(\hat{\nu}_0^{N_k} \lVert \nu_0^{N_k} )  +  \int_0^t\int_\Theta \U_\nu(\pi_s^{N_k}, \nu_s^{N_k}; \theta) d ( \hat{\nu}_s^{N_k} - \nu_s^{N_k}) ds .
 \]

We may now use the joint lower semi-continuity of the relative entropy w.r.t weak convergence to obtain:
\begin{align}
\begin{split}
\H(\hat \nu_t || \nu_t) \leq \liminf_{k \rightarrow \infty} \H(\hat \nu^{N_k}_t || \nu^{N_k}_t ) &= \liminf_{k \rightarrow \infty} \kappa \int_{0}^t \left(    \int \U_\nu (\pi_s^{N_k}, \nu_s^{N_k}; \theta ) d(\hat \nu_s^{N_k} - \nu_s^{N_k}) (\theta)  \right) ds 
\\&+ \lim_{k \rightarrow \infty} \H(\hat \nu^{N_k}_0 || \nu^{N_k}_0).
\\& =\liminf_{k \rightarrow \infty} \kappa \int_{0}^t \left(    \int \U_\nu (\pi_s^{N_k}, \nu_s^{N_k}; \theta ) d(\hat \nu_s^{N_k} - \nu_s^{N_k}) (\theta)  \right) ds 
\\&+\H(\hat \nu_0|| \overline{\nu}_0).
\end{split}
\end{align}
Using Proposition \eqref{prop:ConvergenceRelEntropyUpperBound} and the approximation properties discussed in Remark \ref{rem:ApproxRelEntropy} we obtain \eqref{eqn:EntropyBound_nu}. Inequality \eqref{eqn:EntropyBound_pi} is obtained similarly.
\end{proof}

\begin{remark}
\label{rem:Entropy_ComparisonLiterature}
In contrast to the analysis presented in \cite{domingo2020mean}, here we have used our mean-field limit results from section \ref{sec:ConvergenceMeanField} and the lower semi continuity properties of the relative entropy to fully justify the one-sided identities \eqref{eqn:EntropyBound_nu} and \eqref{eqn:EntropyBound_pi}. As we will see below, these one-sided identities are sufficient for our analysis. Following our approach, we can sidestep the strategy considered in \cite{domingo2020mean} for analyzing a similar problem. Their strategy relies on the assumption of existence and regularity of solutions to a certain PDE describing the evolution of the change of measure between processes similar to the $\nu$ and $\hat{\nu}$ considered here. Unfortunately, such PDE is not even well-defined in general, as it becomes apparent when one considers flows initialized at empirical measures. While this technical difficulty is acknowledged in \cite{domingo2020mean}, no solution for it is provided; see Page 29 in \cite{domingo2020mean}.
\end{remark}

\nc

\subsection{The non-convex non-concave case}
\label{sec:nonconvexnonconcave}

With Proposition \eqref{cor:EntropyBounds} in hand, and following similar steps as in \cite{domingo2020mean}, we can now derive the following results under Assumptions \ref{Hyp U} and \ref{assump:ConvConcav}. 

\begin{lemma}
\label{Lem: Control explitability} 
Let $\pi,\nu$ be the solution of equation \eqref{WFR dynamics} initialized at probability measures $\pi_0, \nu_0$ with $\pi_{0,z}=\mu$. Let $\pi^*,\nu^*$ be arbitrary probability measures over $\mathcal Z \times \mathcal Z $ and $\Theta$, respectively, and suppose that $\pi^*_{z}=\mu$. Let 
\[ \mathcal Q_\pi ( \pi_0,\pi^*; \tau ) := \inf_{\hat{\pi} \in \mathcal P(\Z \times \Z) , \: \hat{\pi}_z= \mu  } \{ \| \pi^*-\hat{\pi}  \|^*_{BL} + \frac 1 \tau \mathcal H (\hat{\pi} || \pi_0) \},\]
where $\lVert \cdot \rVert_{BL}^*$ denotes the dual of the BL (Bounded Lipschitz) norm. Consider also $\mathcal Q_\nu ( \nu_0,\nu^*; \tau )$ defined as
\[ \mathcal Q_\nu ( \nu_0,\nu^*; \tau ) := \inf_{\hat{\nu} \in \mathcal P(\Theta)  } \{ \| \nu^*-\hat{\nu}  \|^*_{BL} + \frac 1 \tau \mathcal H (\hat{\nu} || \nu_0) \}.\]

Suppose that Assumptions \ref{Hyp U} and \ref{assump:ConvConcav} hold. Then
 $$ \U(\pi^*, \bar \nu(t)) -  \U(\bar \pi(t), \nu^*)\leq  B ( \mathcal Q_\pi ( \pi_0,\pi^*; \kappa B t )  + \mathcal Q_\nu ( \nu_0,\nu^*; \kappa B t ) ) +  \frac{2B^2}{t} \int_0^t\int_0^s \eta_{ \tau \nc} d\tau ds, $$
 where $B:=M+L$ (see Assumption \ref{Hyp U} for the meaning of $L$ and $M$). In the above, $\overline{\pi}_t:= \frac{1}{t}\int_0^t \pi_s ds$ and $\overline{\nu}_t:= \frac{1}{t}\int_0^t \nu_s ds$.
\end{lemma}

\begin{proof}

Consider two arbitrary probability measures $\hat{\pi}_0$ and $\hat{\nu}_0$ with $\hat \pi_0 \ll \pi_0, \hat \nu_0\ll \nu_0$, $\hat \pi_{0,z} = \mu$, $\frac{d \hat{\nu}_0 }{d{\nu}_0} \in L^\infty( {\nu}_0)$, and $\frac{d \hat{\pi}_0 }
{d{\pi}_0} \in L^\infty({ \pi}_0)$.  We consider the dynamics $(\hat{\pi}_t, \hat{\nu}_t) $ and $({\pi}_t, {\nu}_t) $ as in Proposition \ref{cor:EntropyBounds}.

  

\begin{enumerate}
  \newcounter{mysteps}
  \setcounter{mysteps}{0}

  \item \stepcounter{mysteps} Step \themysteps: From the concavity of $\U$ in its first coordinate (with respect to linear interpolation) it follows that
  \begin{align*} 
    \U(\pi^*,\nu_t) \leq & \U(\pi_t,\nu_t) + \int \U_{\pi}(\pi_t,\nu_t;z,\tilde z) d(\pi^*-\pi_t) \\
    = & \U(\pi_t,\nu_t) + \int \U_{\pi}(\pi_t,\nu_t;z,\tilde z) d(\pi^*-\hat \pi_t)  \\
    &+ \int \U_{\pi}(\pi_t,\nu_t;z,\tilde z) d(\hat \pi_t-\pi_t) .    
  \end{align*}
 Using the BL (bounded Lipschitz) norm, we get from Proposition \ref{cor:EntropyBounds} that
  \begin{align}
  \begin{split}
    \int_{0}^t \U(\pi^*,\nu_s) ds - \int_0^t\U(\pi_s,\nu_s) ds \leq &\int_{0}^t (\|\U_\pi(\pi_s,\nu_s;\cdot )\|_{BL} \|\pi^*-\hat \pi_s\|^*_{BL}) ds \\
    & +  \int_{0}^t\int \U_{\pi}(\pi_s,\nu_s;z,\tilde z) d(\hat \pi_s-\pi_s)ds \\\
    \leq  & \int_{0}^t  (\|\U_\pi(\pi_s,\nu_s;\cdot )\|_{BL} \|\pi^*-\hat \pi_s\|^*_{BL}) ds 
    \\ & - \frac{1}{\kappa} (\H(\hat \pi_t|| \pi_t) -\H(\hat \pi_0|| \pi_0) ).
    \end{split}
    \label{Diff L 1}
\end{align}
A similar argument using the convexity of $\U$ in its second coordinate deduces
\begin{align}
\begin{split}
\int_{0}^t \U(\pi_s,\nu^*)ds - \int_{0}^t \U(\pi_s,\nu_s)ds  \geq & -\int_{0}^t (|\U_\nu(\pi_s,\nu_s;\cdot )\|_{BL} \|\nu^*-\hat \nu_s\|^*_{BL}) ds  
\\& + \frac{1}{\kappa} (\H(\hat \nu_t|| \nu_t) - \H(\hat \nu_0|| \nu_0) ). 
\end{split}
\label{Diff L 2}
\end{align}
Using again the concavity and convexity of $\U$, we get:
\begin{align*}
  \U(\bar \pi_t, \nu^*) &\geq \frac 1 t \int_0^t \U(\pi_s,\nu^*) d s, 
  & \U(\pi^*,\bar \nu_t) & \leq \frac 1 t \int_0^t \U(\pi^*,\nu_s) d s.
\end{align*} 
Combining the above with \eqref{Diff L 1}, \eqref{Diff L 2}, and the fact that $\H(\hat{\nu}_t || \nu_t), \H(\hat{\pi}_t || \pi_t) \geq 0 $ we conclude that
\begin{align}
  \U(\pi^*,\bar \nu_t) - \U(\bar \pi_t, \nu^*) \leq & \frac 1 t \int_0^t (\|\U_\nu(\pi_s,\nu_s;\cdot )\|_{BL} \|\nu^*-\hat \nu_s\|^*_{BL}  + \|\U_\pi(\pi_s,\nu_s;\cdot)\|_{BL} \|\pi^*-\hat \pi_s\|^*_{BL} ) ds \notag \\
  & + \frac{1}{\kappa t} \left(\mathcal H(\hat \nu_0 || \nu_0) + \mathcal H(\hat \pi_0 || \pi_0)      \right) \notag \\
  \leq & \frac B t \int_0^t ( \|\nu^*-\hat \nu_s\|^*_{BL}  + \|\pi^*-\hat \pi_s\|^*_{BL} ) ds + \frac{1}{\kappa t} \left(\mathcal H(\hat \nu_0 || \nu_0) + \mathcal H(\hat \pi_0|| \pi_0)      \right). 
  \label{bound on L differences}
\end{align}

\item \stepcounter{mysteps} Step \themysteps:  Observe that both $\U_\pi$ and $\U_\nu$ have their BL norm bounded by $B=M+L$. To conclude, it remains to remark that
\begin{align*}
  \frac{1}{t} \int_0^t \|\pi^*-\hat \pi_s\|^*_{BL} ds \leq & \|\pi^*-\hat \pi_0\|^*_{BL} + \frac{1}{t} \int_0^t \|\hat\pi_0-\hat \pi_s\|^*_{BL} ds\\
  = & \|\pi^*-\hat \pi_0\|^*_{BL} + \frac 1 t \int_0^t \left\{  \sup_{\|f\|_{BL}\leq 1 ; f \in C^1} \int f d(\hat\pi_s-\hat \pi_0) \right\} ds\\ 
  \leq & \|\pi^*-\hat \pi_0\|^*_{BL} + \frac{B}{t} \int_0^t\int_0^s \eta_\tau d \tau d s,
\end{align*}
and similarly,
\[\frac{1}{t} \int_0^t \|\nu^*-\hat \nu_s\|^*_{BL} ds \leq  \|\nu^*-\hat \nu_0\|^*_{BL} + \frac{B}{t} \int_0^t\int_0^s \eta_\tau d \tau d s .\]
Replacing in \eqref{bound on L differences}, it follows that
\begin{align} 
  \U(\pi^*,\bar \nu_t) - \U(\bar \pi_t, \nu^*) & \leq  B ( \|\nu^*-\hat \nu_0\|^*_{BL}  + \|\pi^*-\hat \pi_0\|^*_{BL} )  \\
  &\quad  + \frac{1}{\kappa t} \left(\mathcal H(\hat \nu_0 || \nu_0) + \mathcal H(\hat \pi_0|| \pi_0)     \right) +  \frac{2B^2}{t} \int_0^t\int_0^s \eta_\tau d \tau d s. \notag
\end{align}

Recall that $\hat \pi_0$ and $\hat \nu_0$ were arbitrary measures with densities with respect to $\pi_0$ and $\nu_0$ belonging to $L^\infty$. From a simple density argument we may now conclude the desired estimate.

\end{enumerate}

\end{proof}

The following Lemma is taken from  \cite{domingo2020mean} which in turn follows the arguments in \cite{chizat_global_2018}. 

\begin{lemma}
  Suppose that Assumptions \ref{Hyp U} and \ref{assump:ConvConcav} hold. Assume further that there exists $k>0$  such that $\frac{d \nu_0}{d\theta} (\theta) > k$, and suppose that $ |\mc B_{\theta,\epsilon}\cap \Theta | \geq k' \epsilon^d$ uniformly in $\theta \in \Theta$, where $\mc B_{\theta,\epsilon}$ is the Euclidean ball of radius $\epsilon$ centered at $\theta$. Then,
  \[ \mathcal Q_\nu ( \nu_0,\nu^*; \tau ) \leq \frac{d}{\tau} \left\{1 - \log\left(\frac d \tau \right) \right\} + \frac{1}{\tau}  \{- \log ( k) - \log(k')\} . \]
  \label{Lem: Control on Q}
\end{lemma}

\begin{proof}
We obtain a bound for the $\min$ in the definition of $\mc Q_{\nu}$. For a fixed $\epsilon>0$ we introduce a probability measure $\nu^\epsilon$ given by
\[\nu^\epsilon(A) := \int_\Theta \frac{|\mc B_{\theta,\epsilon} \cap A|}{|\mc B_{\theta,\epsilon} \cap \Theta|}    d \nu^*(\theta).\]

We now calculate $W_1(\nu^*,\nu^\epsilon)$. Consider the coupling: 
\[ \Upsilon(A,A') := \int_A \frac{|\mc B_{\theta,\epsilon} \cap A'|}{|\mc B_{\theta,\epsilon} \cap \Theta|}   d \nu^*(\theta).\]
Indeed, one easily verifies $\Upsilon(\Theta,A') = \nu^\epsilon(A')$ and $\Upsilon(A,\Theta) = \nu^*(A)$. Thus, 
\[ W_1(\nu^*,\nu^\epsilon) \leq \int_{\Theta} \int_\Theta  |\theta-\theta'| d\zeta(\theta,\theta') = \int_\Theta \int_{\theta' \in \mc B_{\theta,\epsilon} \cap \Theta } |\mc B_{\theta,\epsilon} \cap \Theta|^{-1} |\theta-\theta'| d\theta' d \nu^*(\theta) \leq  \epsilon. \]
Since for any measure $\nu \in \mathcal{P}(\Theta)$ we have that $\lVert \nu^*-\nu \rVert_{BL}^* \leq W_1(\nu^*,\nu)  $, we obtain a bound of $\epsilon$ for the first term in $\mc Q_\nu$. 

We now turn to the relative entropy term. 
Observe that the definition of $\nu^\epsilon$ and Fubini's theorem gives
\[ \frac{d \nu^\epsilon}{d \theta} (\theta) =  \int_\Theta |\mc B_{\theta',\epsilon} \cap \Theta|^{-1} \one_{\mc B_{\theta',\epsilon}\cap \Theta}(\theta) d \nu^*(\theta') ; \]
thus, by convexity of the function $x \mapsto x\log(x)$, Jensen's inequality and Fubini's theorem, we have 
\begin{align}
   \int_\Theta \frac{d \nu^\epsilon}{d \theta}(\theta) \log\left( \frac{d \nu^\epsilon}{d \theta} (\theta) \right) d\theta & \leq \int_\Theta  \int_\Theta |\mc B_{\theta',\epsilon} \cap \Theta|^{-1} \one_{\mc B_{\theta',\epsilon} } (\theta) \log(|\mc B_{\theta',\epsilon} \cap \Theta|^{-1}) d \nu^*(\theta') d\theta  \\
   & \leq -  \int_\Theta  \log(|\mc B_{\theta',\epsilon}\cap \Theta|) d \nu^*(\theta')
   \leq -\log(k') - d \log(\epsilon), \notag 
\end{align}
where we have used the convention $0\times -\infty = 0$.
On the other hand, by assumption, 
\begin{align}
  \int_\Theta \frac{d \nu^\epsilon}{d \theta}(\theta) \log\left( \frac{d \nu_0}{d \theta} (\theta) \right) d\theta = \int_\Theta \log\left( \frac{d \nu_0}{d \theta} (\theta) \right) d \nu^\epsilon(\theta) \geq \log(k).
\end{align}
From the above it follows
\[\H(\nu^\epsilon|| \nu^0) \leq  -\log(k) - \log(k') - d \log(\epsilon). \]
Hence,
\[ \mathcal Q ( \nu_0,\nu^*; \tau ) \leq \epsilon  - \frac{1}{\tau}(d \log(\epsilon) +\log(k') + \log(k)), \]
for every $\epsilon>0$.
Choosing $\epsilon= \frac{d}{\tau}$, the minimizer of the right hand side of the above expression, we get the desired result.
\end{proof}

\begin{remark}
The condition $ |\mc B_{\theta,\epsilon}\cap \Theta | \geq k' \epsilon^d$ uniformly over $\theta \in \Theta$ is implied by the fact that the boundary of $\Theta$ was assumed to be Lipschitz; see Assumptions \ref{Hyp U}.
\end{remark}

\begin{lemma}
  Suppose that Assumptions \ref{Hyp U} and \ref{assump:ConvConcav} hold. Let $\pi_0$ be such that $\pi_{0,z}=\mu$ and suppose that 
 there exists $k>0$  such that $\frac{d \pi_0(\tilde z |z)}{d\tilde z} (\tilde z) > k$ for all $z$ in the support of $\mu$. Suppose further that $ |\mc B_{\tilde z,\epsilon}\cap \Z | \geq k' \epsilon^{d'}$ uniformly in $\tilde z \in \Z$. Then, for all $\pi^*$ with $\pi^*_z=\mu$, we have
  \[ \mathcal Q_\pi ( \pi_0,\pi^*; \tau ) \leq \frac{d'}{\tau} \left\{1 - \log\left(\frac{d'} {\tau} \right) \right\} + \frac{1}{\tau}  \{- \log ( k) - \log(k')\} . \]
  \label{Lem: Control on Q2}
\end{lemma}
\begin{proof}
Since all measures of interest must have the same first marginal (i.e., $\mu$) we proceed as in Lemma \ref{Lem: Control on Q}, but this time only regularizing conditional distributions. More precisely, for a given $z$ in the support of $\Z$ we define the measure $\pi^\epsilon(\cdot|z)$ as follows:
\[\pi^\epsilon(A|z) := \int_\Z \frac{|\mc B_{\tilde z,\epsilon} \cap A|}{|\mc B_{\tilde z,\epsilon} \cap \Z|}    d \pi^*(\tilde z |z ).\]
We then define the measure $\pi^\epsilon$ as:
\[ d \pi^\epsilon(z, \tilde z)= d\pi^\epsilon(\tilde z| z)d\mu(z). \]
Notice that the measure $\pi^\epsilon$ is such that $\pi^\epsilon_{z}=\mu$. Moreover, it is straightforward to show (repeating similar computations as in the proof of Lemma \ref{Lem: Control on Q}) that $W_1(\pi^\epsilon, \pi^*) \leq \epsilon$ and $\H(\pi^\epsilon|| \pi_0) \leq -\log(k) - \log(k') - d' \log(\epsilon).$ The desired result now follows as in Lemma \ref{Lem: Control on Q}.
\end{proof}

\begin{proof}[Proof of Theorem \ref{thm:Main2}]
  On the one hand, by assumption, we can find $T_1$ such that for all $t>T^*$
  \[ |2B^2 \frac{1}{t} \int_0^t\int_0^s \eta_u du ds | \leq \frac{3}{4}\epsilon.\]
  On the other hand, Lemmas  \ref{Lem: Control on Q} and \ref{Lem: Control on Q2} imply that there exists $T_2$ such that, for all $t>T^*$ and arbitrary $\pi^*$  with $\pi^*_z=\mu$ and $\nu^*$, we have
  \[ ( \mathcal Q ( \pi_0,\pi^*; \kappa B t ) + ( \mathcal Q ( \nu_0,\nu^*; \kappa B t ) \leq \frac \epsilon {4B}. \]
  We conclude by taking $T^*=T_1 \vee T_2$ and using Lemma \ref{Lem: Control explitability}.
\end{proof}

\subsection{The non-convex and strongly concave case}
\label{sec:stronglyConcave}



In this section we present the proof of Theorem \ref{thm:mainStrongConvex}.

\begin{proof}[Proof of Theorem \ref{thm:mainStrongConvex}]
Throughout this proof we use $m^*_{t}$ to denote the quantity 
\[m^*_t:= \sup_{ \pi \: \text{s.t.} \: \pi_z=\mu  } \U(\pi,\nu_t) .\]

From concavity-convexity of $\U$ in the linear interpolation sense we have for all arbitrary $\tilde \pi$ (with $\tilde \pi_z =\mu$) and $\tilde \nu$:
\begin{align}
    \begin{split}
        \U(  \tilde \pi, \overline{\nu}_{t}) - \U( \overline{\pi}_t, \tilde \nu ) &  \leq \frac{1}{t}\int_0^t ( \U( \tilde \pi, \nu_{s}) - \U(\pi_s,\tilde \nu)   )   ds 
        \\ & = \frac{1}{t}\int_0^t ( \U( \tilde \pi,\nu_{s}) - m^*_{s})ds +  \frac{1}{t}\int_0^t ( m^*_{s} - \U(\pi_s,\tilde \nu)   )   ds 
         \\& \leq \frac{1}{t}\int_0^t ( m^*_s - \U( \pi_s,\tilde \nu)   )   ds 
\\& = \frac{1}{t}\int_0^t ( m^*_{s} - \U( \pi_s, \nu_{s} ) )ds + \frac{1}{t}\int_0^t ( \U( \pi_{s}, \nu_{s}) - \U(\pi_s,\tilde \nu)   )   ds.
\\& =: \I_1 + \I_2
    \end{split}
    \label{eq:AuxStrongConvex}
\end{align}
In the above, the second inequality follows from the definition of $m^*_{s}$. We will now control each of the terms $\I_1$ and $\I_2$ on the right-hand side of the above expression.

In order to control $\I_1$, we start by using the chain rule (e.g., see section 10.1.2 in \cite{AmbrosioGigliSavare}) to obtain an expression for $\frac{d}{ds} \U(\pi_s, \nu_s)$: 
\begin{align}
\begin{split}
\frac{d}{ds} \U(\pi_s, \nu_s) = & \int |\nabla_{\tilde z} \U_\pi(\pi_s, \nu_s; z, \tilde z)  |^2 d \pi_s(z, \tilde z) + \kappa \int  \U_\pi(\pi_s, \nu_s;z, \tilde z) (\U_\pi(\pi_s, \nu_s;z, \tilde z) - \overline{\U}_{\pi,z}) d \pi_s(z,\tilde z)   
\\&- \frac{\eta_t}{K} \int |\nabla_\theta \U_\nu(\pi_s, \nu_s; \theta)|^2 d\nu_s(\theta) - \frac{\kappa}{K} \int  \U_\nu(\pi_s, \nu_s;\theta) ( \U_\nu(\pi_s, \nu_s;\theta) - \U_\nu) d\nu_s(\theta); 
\\& 
\end{split}
\label{eq:ChainRule}
\end{align}
in the above, we use the shorthand notation $\overline{\U}_{\pi,z}$ to denote $\int \U_{\pi}(\pi_s, \nu_s; z, \tilde z' ) d \pi_s(\tilde z'|z)$, and $\overline{\U}_{\nu}$ to denote $\int \U_\nu(\pi_s, \nu_s; \theta')d \nu_s(\theta')$. Assumption \ref{assump:StrongConcavity} implies that the first term on the right-hand side of \eqref{eq:ChainRule} is bounded from below by $\lambda ( m_s^*- \U(\pi_s, \nu_s))$. On the other hand, Jensen's inequality implies that the second term is non-negative. Finally, Assumptions \ref{Hyp U} imply that the last terms can be bounded from below by $-\frac{\lVert \eta\rVert_\infty}{ K} M^2 - \frac{2\kappa}{K} M^2$. It follows that for all $t\geq 0$
\begin{align*}
\U(\pi_t, \nu_{t}) =& \U( \pi_0,\nu_{0}) +\int_0^t \frac{d}{dr} \U( \pi_r,\nu_{r}) dr
\\ \geq&  \U( \pi_0, \nu_{0}) - \frac{t}{K} \tilde B + \lambda  \int_0^t ( m_s^* - \U( \pi_s, \nu_{s}) ) ds,
\end{align*}
where $\tilde B := (\lVert \eta \rVert_\infty  + 2\kappa )M^2 $. Subtracting $m_t^*$ from both sides of the above inequality, we get, thanks to Assumptions \ref{Hyp U}, 
\[ \U(\pi_t, \nu_{t})  - m_t^* \geq  - 2M  - \frac{t}{K} \tilde B + \lambda  \int_0^t ( m^*_s - \U( \pi_s, \nu_{s}) ) ds.\] 
Equivalently,
\[  m_t^*- \U( \pi_t,\nu_{t})   \leq   2M  +  \frac{t}{K}\tilde  B - \lambda  \int_0^t ( m^*_s - \U( \pi_s,\nu_{s}) ) ds.\] 
We thus see that the function $f(t):= m^*_t - \U(\nu_{t}, \pi_t)  $ satisfies 
\[ f(t) \leq 2M + \frac{\tilde B }{K} t - \lambda \int_{0}^t f(s) ds, \]
and from Lemma \ref{lemm:Gronwall} in Appendix \ref{app2} we conclude that 
\[ \I_1  \leq 
\frac{\tilde B}{K \alpha} + \frac{A}{t}, \]
for $A:= \frac{1}{\lambda}|2M - \frac{\tilde B}{K \lambda}|$. 

To estimate $\I_2$ in \eqref{eq:AuxNonConvexConcave}, we follow similar computations to those in the proof of Lemma \ref{Lem: Control explitability} to conclude that 
\begin{align}
\begin{split}
 \int_{0}^t \U(\pi_s,\nu_{s})ds - \int_{0}^t \U(\pi_s,\tilde \nu)ds  \leq & \int_{0}^t(|\U_\nu(\pi_s,\nu_{s};\cdot )\|_{BL} \|\tilde \nu-\hat \nu_{s}\|^*_{BL}) ds  
\\ & -\int_{0}^t \left(    \int \U_\nu (\pi_s, \nu_{s}; \theta ) d(\hat \nu_{s} - \nu_{s}) (\theta)  \right) ds,
\end{split}
\label{eq:AuxNonConvexConcave}
\end{align}
where now we use a modified hat process $\hat \nu$ satisfying
\[  \partial_t \hat \nu_t = \frac{\eta_t}{K} \divergence_{\theta} (\hat \nu_t \nabla_{\theta} \U_\nu (\pi_{t}, \nu_t;\theta)), \]
initialised at an arbitrary $\hat \nu_0\ll \nu_0$ with density in $L^\infty (\nu_0)$. Following a straightforward adaptation of Proposition \ref{cor:EntropyBounds}, we can then see that 
\begin{equation*}
   \H(\hat{\nu}_t || \nu_t) - \H(\hat{\nu}^0|| \nu^0) \leq \frac{\kappa}{K} \int_{0}^t \left(    \int \U_\nu (\pi_s, \nu_s; \theta ) d(\hat \nu_s - \nu_s) (\theta)  \right) ds, \quad \forall t \geq 0,
\end{equation*}
from where it now follows that
\begin{align*}
\begin{split}
 \I_2  & \leq \frac{1}{t}\int_{0}^t(\lVert\U_\nu(\pi_s,\nu_{s};\cdot )\|_{BL} \|\tilde \nu-\hat \nu_{s}\|^*_{BL}) ds  +\frac{K}{\kappa t } \H(\hat{\nu}_0|| \nu_0)
 \\& \leq  B  \|\tilde \nu-\hat \nu_0\|^*_{BL} +  \frac{ B^2}{Kt} \int_0^t\int_0^s \eta_\tau d \tau d s   + \frac{K}{\kappa t} \mathcal H(\hat \nu_0 || \nu_0)  .
\end{split}
\label{eq:AuxNonConvexConcave2}
\end{align*}
From the above we can deduce
\[ \I_2 \leq B \mathcal Q_\nu ( \nu_0,\tilde \nu; \frac{\kappa}{K} B t ) ) +  \frac{B^2}{Kt} \int_0^t\int_0^s \eta_{ \tau \nc} d\tau ds. \]

Putting all our estimates together we obtain
\begin{align*}
\U( \tilde \pi, \overline{\nu}_{t}) - \U( \overline{\pi}_t,\tilde \nu )  \leq 
\frac{\tilde B}{K \lambda} + \frac{A}{t} + B  \mathcal Q_\nu ( \nu_0,\tilde \nu; \frac{\kappa}{K} B t ) ) +  \frac{B^2}{Kt} \int_0^t\int_0^s \eta_\tau d \tau d s. 
\end{align*}

At this stage we can use the specific properties of $\nu_0$ and use Lemma \ref{Lem: Control on Q} to conclude that there are $r_0(\epsilon), K_0(\epsilon), t_0(\epsilon),r_1(\epsilon)$ such that, if $\frac{K}{t} \leq r_0(\epsilon)$, $K \geq K_0(\epsilon), t \geq t_0(\veps)$, $\overline{\eta}/K \leq r_1(\epsilon)$, then 
\[ \sup_{\tilde \pi \in \mathcal{P}(\Z^2) \text{ s.t. } \tilde \pi_z=\mu} \U( \tilde \pi,\overline{\nu}_t) - \inf_{\tilde \nu \in \mathcal{P}(\Theta)} \U(\overline{\pi}_t, \tilde \nu) \leq \epsilon.   \]

\end{proof}

\section{Numerical examples}
\label{sec:Experiments}

We illustrate our results numerically in the context of image classification on the MNIST database. 

In this framework, we take the  particles representing the distribution $\nu$ to be the training parameters (i.e. weights and biases) for simple convolutional networks with fixed widths and depths\footnote{Two layers with a convolutional kernel of size 5 and output channel sizes of 32 and 64 respectively, ReLu activation functions, and maxpool; and two linear layers at the end with respective output sizes 1000 and 10.}; while the particles representing the distribution $\pi$ are pairs of images where the first component is an image from the original database, and the second is an adversarial image built during the training process. We set the loss function to be given by \[\Risk(\pi,\nu) = \int_{\mc Z \times \mc Z} \int_{\Theta} | h_\theta(\tilde x) - \tilde y  |^2  d\nu(\theta) d \pi(z,\tilde z) ; \quad \C(\pi) = c_a \int_{\mc Z \times \mc Z} |z-\tilde z|^2 d\pi(z,\tilde z)\] 
where, $h_\theta(x)$ is the outcome of the convolutional network for the input $x$ when setting the parameters of the network to be $\theta$. 

Given the nonlinear structure of the convolutional architecture, it would be extremely memory-consuming to apply directly the time average step \ref{AlgLine: time-avg_net}  in Algorithm \ref{Algo:WADA}, as it would require us to keep track of copies of all intermediate networks in the training process. A possible solution, proposed for example in \cite{domingo2020mean}, is to calculate the time average on the weights only, while keeping the last position of the network-parameter particles. 
We implement also an alternative approach based on the resampling methods used in particle filters (see \cite{li_resampling_2015} for a review): we keep in memory at most a maximum number of network parameters ($M'$). At each update time, we use residual systematic resampling  (RSR) to pick $M'$ parameters to keep from the list of the $M'$ already contained in memory and the new bunch of $M$ particles. Details of the (RSR) method can be found in \cite{li_resampling_2015} (see for example code 4 in Table 2). The time-average calculation of adversarial images is done similarly.

To illustrate our main result, we compute a proxy for the ratios 
\[ r_a:=  \frac{ \sup_{\tilde \pi \in \mathcal{P}(\Z^2)\text{ s.t. } \tilde \pi _z=\mu}\;\U(\tilde \pi, \nu^\star)}{ \U(\pi^\star, \nu^\star) }  , \text{  and  }  r_m:= \frac{\inf_{\tilde \nu \in \mathcal{P}(\Theta)}\;\U( \pi^\star, \tilde\nu)}{ \U(\pi^\star, \nu^\star) },\]
where $(\pi^\star, \nu^\star)$ are the time-averaged distributions for the networks and adversarial images obtained after training. According to our results, we should reach an approximate Nash equilibrium, so we expect both ratios to be closed to zero. 
The proxy is computed as follows: we approximate the supremum in $r_m$ by fixing $\nu^\star$ while training each one of the networks representing $\pi^\star$ with stochastic gradient descent  for a fixed number of epochs (weights are kept constant). We compute then the relative change in total risk after this procedure. The proxy for $r_a$ is computed analogously. We present a summary of the parameters used for the numerical experiments and the results obtained in Table \ref{tab:numerical_tests}. 
\begin{table}[h]
    \centering
    
    \begin{tabular}{|c|c|}
        \hline
        \multicolumn{2}{|c|}{Model parameters} \\
        \hline \hline
        $N$ & 4\\
        $M$ &  2\\
        $\eta_t$  &  $0.1(t+1)^{-1}$\\
        $\kappa$ & 0.25   \\
        $c_a$ & 10 \\
        \hline
    \end{tabular}
  \quad
    \begin{tabular}{|l|c|}
         \hline
         \multicolumn{2}{|c|}{Test implementation parameters} \\
         \hline \hline
         Dataset & MNIST \\
         Batch size & 64 \\
         Training epochs & 4 \\
         \hline
    \end{tabular}
     \qquad
    \begin{tabular}{|c|c|c|}
         \hline
         \multicolumn{3}{|c|}{Results} \\
         \hline \hline
         \multicolumn{3}{|c|}{Accuracy}\\
         \hline
         & Time avg. on weights   & Resampling\\
         Clean  & 96.34\%  &    93.53\% \\
         PGD (20 steps) & 62.21\% & 58.49 \%  \\
         \hline
         \multicolumn{3}{|c|}{ Relative change of loss at solution - 5 additional training epochs}\\
         \hline
         & Time avg. on weights   & Resampling\\
         $r_a$ & 0.21\% & 0.03 \% \\
         $r_m$ & 1.82\%  & 3.5\%  \\
         \hline
    \end{tabular}
    \caption{Parameters and results of numerical test}
    \label{tab:numerical_tests}
\end{table}

Intuitively, we expect that the classification provided by the final time-averaged distributions of networks should be both effective and robust. 
To test this idea, we evaluate the accuracy of this classifier with a \emph{clean} testing sample, independent of the original distribution, and with an \emph{adversarial} version constructed via modification of the latter using projected gradient descent (PGD) with 20 steps and a step size of 0.04.  PGD constructs adversarial images by repeatedly perturbing each pixel in the image by a fixed amount choosing the sign of the perturbation to be the same as the sign of the gradient of the loss function with respect to the entry. See for example \cite{madry2018towards}. Results of this test are also included in Table  \ref{tab:numerical_tests}.
We observe that the overall procedure has degraded a bit the clean performance of the network but significantly improved the resistance to adversarial attacks. 
For reference, a baseline obtained by a simple training of a network with the same characteristics obtains in the same number of epochs a clean accuracy of 98.41\% but an accuracy after the PGD attack of only 0.68\% (compare also with the results in \cite{garcia_trillos_regularized_2022}). 
Table \ref{tab:numerical_tests} shows that in the tested case, calculating the time average on the weights only is not only simpler but also has better results than the resampling procedure. However, there may be settings, not explored here, where the latter approach may be more advantageous. Exploring this would be an interesting research direction. 

\nc



\section{Conclusions}
\label{sec:Conclusions}

In this paper we have studied min-max  problems over spaces of probability measures with a payoff structure motivated by adversarial problems in supervised learning settings; we have studied gradient ascent-descent dynamics aimed at solving these problems. The dynamics that we have studied take the form of an evolutionary system of PDEs that can be discretized using systems of finitely many interacting particles. Under some reasonable assumptions on the payoff structure of the game, we can show that the proposed particle systems are consistent and recover the solution of the original PDE as the number of particles in the system scales up. We have also discussed the behavior of our evolutionary system of PDEs as time tends to infinity, showing that in a certain sense (see below) the system can produce approximate Nash equilibria for the adversarial game. Our results are stated under suitable assumptions on intializations in two settings of interest: 1) for nonconvex nonconcave payoffs (convexity and concavity understood in a suitable OT-sense), and 2) nonconvex strongly-concave problems (again, in a suitable OT sense). Both settings are realistic in adversarial learning games for supervised learning tasks, while in general convexity can only be enforced by introducing additional (exogenous) regularization penalties in the payoff function. 

Due to the lack of convexity of the payoff in our problem (w.r.t. the metric inducing the dynamics of our ascent-descent dynamics), we can only guarantee that time averages of the measures produced by our PDE system become approximate Nash equilibria in the $t \rightarrow \infty$ limit. For our algorithms to follow more closely our theoretical results, it was thus important to discuss strategies for constructing surrogate time averages that do not incur in memory overload and that can still recover approximate Nash equilibria for the game, at least in some benchmark learning tasks.


\medskip

There are several directions for research that our work motivates. Here we mention a few. 

\medskip

First, the theoretical analysis that we have presented in this paper presupposes that the optimization updates take into account all (perturbed) data points, but in practice a natural strategy is to use batches of data to compute the loss (and its gradient) at each iteration. We thus believe that it is of interest to study how the use of stochastic gradient descent (SGD) affects the resulting PDE system. 

Another interesting direction for future research is the exploration of broader frameworks for adversarial learning covering \textit{multiclass classification} settings (as opposed to regression problems as considered in this paper or just binary classification problems). In principle, one could even consider situations where prior information on the similarity of classes in a learning problem is available  (e.g. the class ``guitar" may be considered more similar to class ``violin" than to class ``baseball") as in those situations it may be beneficial to use such information to construct more nuanced models for risk and admissible adversarial attacks; for example, the work \cite{Silla2010} discusses the advantages of using similarity or hierarchical structures between classes in different learning tasks; the work \cite{LabelSimilarity1} explicitly discusses how to build similarities between labels using their semantic content. Our framework indeed seems better suited for regression problems, since in that setting the cost function $C$ for the adversary can be naturally defined using something like the Wasserstein distance over the feature space times the label space, where the latter space is simply the real line. When the response variable has a discrete structure, it is less obvious how one can still define a reasonable (from the modelling perspective) cost structure for the adversary in such a way that the resulting adversarial game can still be solved using an ascent-descent scheme as explored in this paper.

\section*{Acknowledgments}
The authors are thankful to Theodor Misiakiewicz, Katy Craig, Matt Jacobs, and L\'ena\"{i}c Chizat for enligthening discussions on topics related to the content of this paper. NGT was supported by NSF-DMS grant 2005797, and would also like to thank the IFDS at UW-Madison and NSF through TRIPODS grant 2023239 for their support.

\appendix

\section{Auxiliary results and computations}
\label{app1}

\subsection{On the PL condition of Assumption \ref{assump:StrongConcavity}}
\label{app0}

\begin{proposition}
\label{prop:app0}
Suppose that $\Z$ is a convex set and that we select an activation function and a loss function in the setting in  \ref{sec:RobustRegression} that are twice continuously differentiable. Then the function $\U$ with $\Risk$ and $\C$ as in \ref{Ex of h_nu} and \ref{eq:Cost} satisfies the condition in Assumption \ref{assump:StrongConcavity} for all large enough $c_a$.
\end{proposition}

\begin{proof}
We recall that in this case we have
\[ \U_{\pi}(\pi, \nu; z, \tilde z) = \ell(h_\nu(\tilde x ), \tilde y) - c_a|z-\tilde z |^2 =: \U(\nu; z, \tilde z);\]
see Example \ref{example:2}. Assuming that the loss function $\ell$ and the activation function $f$ are at least twice continuously differentiable, we can conclude that the function $\tilde z \in \Z  \mapsto \U(\nu; z, \tilde z)$ (for fixed $z$ and $\nu$) is strongly concave (for all $z$ and $\nu$), provided that $c_a$ is large enough. Indeed, this is simply because we can bound, uniformly over $z,\nu$, the second derivatives of the first term in $\U(\nu; z, \tilde z).$ Thanks to this and Theorem 5.15 ii) in \cite{VillaniBook}, we deduce that there is $\lambda>0$ such that for every $z\in \Z$ and $\Upsilon \in \mathcal{P}(\Z)$ we have
\[ \int_{\Z} | \nabla_{\tilde z} \U (\nu; z, \tilde z) |^2 d\Upsilon(\tilde z) \geq \lambda ( \sup_{\hat \Upsilon \in \mathcal{P}(\Z)} \int_\Z \U(\nu; z , \tilde z) d \hat{\Upsilon}(\tilde z)  -      \int_\Z \U(\nu; z , \tilde z) d {\Upsilon}(\tilde z)  ).   \]
In particular, for a given $\pi \in \mathcal{P}(\Z\times \Z)$ with $\pi_{0,z}=\mu$, we have
\[ \int_{\Z} | \nabla_{\tilde z} \U (\nu; z, \tilde z) |^2 d\pi(\tilde z|z) \geq \lambda ( \sup_{\hat \Upsilon \in \mathcal{P}(\Z)} \int_\Z \U(\nu; z , \tilde z) d \hat{\Upsilon}(\tilde z)  -      \int_\Z \U(\nu; z , \tilde z) d \pi(\tilde z|z)  ),   \]
for all $z \in \Z$ and all $\nu \in \mathcal{P}(\Theta)$. Integrating over $z$ with respect to $\mu$ on both sides of the above inequality, we get
\begin{align*}
\int_{\Z\times \Z} | \nabla_{\tilde z} \U_\pi (\pi,\nu; z, \tilde z) |^2 d\pi(z, \tilde z) &= \int_{\Z\times \Z} | \nabla_{\tilde z} \U (\nu; z, \tilde z) |^2 d\pi(z, \tilde z)
\\& \geq \lambda \left(  \int_{\Z} \left( \sup_{\hat \Upsilon \in \mathcal{P}(\Z)} \int_\Z \U(\nu; z , \tilde z) d \hat{\Upsilon}(\tilde z)  \right)d \mu(z)  - \U(\pi,\nu)  \right)
\\& \geq \lambda\left( \sup_{\hat \pi\in \mathcal{P}(\Z^2) \text{ s.t. } \hat \pi_z=\mu } \U(\hat \pi, \nu) -\U(\pi, \nu) \right).  
\end{align*}
\end{proof}

\nc 
\subsection{Auxiliary lemmas for the construction of approximate initializations in Theorems \ref{thm:Main1} and \ref{thm:PropChaos} }
\label{App:1}

\begin{proposition}
Let $A, B$ be two bounded Borel subsets of $\R^{d}$ and $\R^{d'}$, respectively. Let $\mu \in \mathcal{P}(A)$, and let $ u \in A \mapsto \mu_u (\cdot ) \in \mathcal{P}(B)  $ be a measurable map. 

For every sequence $\{\Upsilon_n \}_{n \in \N} \subseteq \Gamma(\mu, \mu)$ satisfying \[\lim_{n \rightarrow \infty } \int_{A \times A} |u - u' | d \Upsilon_n(u, u') =0, \]
we have
\[ \lim_{n \rightarrow \infty } \int_{A \times A} W_1( \mu_{u}, \mu_{u'}  ) d \Upsilon_n(u, u') =0.  \]

\label{Prop:StagnatingSequence}
\end{proposition}
\begin{proof}
Sequences $\{\Upsilon_{n}\}_{n \in \N}$ satisfying the hypothesis in the proposition are called \textit{stagnating sequences} of transport plans; see \cite{NGTSlepcev}. 

Let $\veps>0$. For such $\veps>0$ we can build a finite partition $\{ B_l \}_{l=1, \dots, L}$ of the set $B$ in such a way that each set $B_l$ has diameter less than $\veps/3$; this partition can be constructed by simply intersecting a grid of boxes in $\R^{k'}$ of diameter less than $\veps/3$ with the set $B$. Select now a point $v_l$ in each of the $B_l$. Associated to each $l=1, \dots, L$, we define a function $h_l \in L^1(\mu)$ as follows: for every $u $ in the support of $\mu$, we define $h_l(u):=\mu_u(B_l)$. We now consider the measures $\hat{\mu}_u := \sum_{l=1}^L  h_l(u) \delta_{v_l}. $
Notice that these are probability measures satisfying $W_1(\hat{\mu}_u, \mu_u) \leq \veps/3$. In particular, using the triangle inequality for $W_1$  we deduce
\begin{align*}
 \int_{A \times A} W_1( \mu_{u}, \mu_{u'}  ) d \Upsilon_n(u, u') & \leq  \int_{A \times A} W_1( \mu_{u},\hat{\mu}_{u}  ) d \Upsilon_n(u, u')  + \int_{A \times A} W_1( \hat{\mu}_{u}, \hat{\mu}_{u'}  ) d \Upsilon_n(u, u')
 \\ & + \int_{A \times A} W_1( \hat{\mu}_{u'}, \mu_{u'}  ) d \Upsilon_n(u, u'). 
 \\ & \leq \frac{2}{3} \veps + \int_{A \times A} W_1( \hat{\mu}_{u}, \hat{\mu}_{u'}  ) d \Upsilon_n(u, u'). 
\end{align*}
Let us now find an upper bound for the term  $\int_{A \times A} W_1( \hat{\mu}_{u}, \hat{\mu}_{u'}  ) d \Upsilon_n(u, u')$. By the Kantorovich duality for the $W_1$ distance, we have
\[ W_1( \hat{\mu}_{u}, \hat{\mu}_{u'}  ) 
 =\sup_{Lip(f) \leq 1}  \{\int f(v) d  \hat{\mu}_{u}(v) -  \int f(v) d  \hat{\mu}_{u'}(v)\}.  \]
 Since the set $B$ is bounded, and the argument inside the sup is invariant under addition of a constant to a given $f$, we can further assume that the $\sup$ is taken over functions $f$ whose supremum norm is  bounded by a fixed constant $C$. For such a function $f$ we have
 \[ \int f(v) d  \hat{\mu}_{u}(v) -  \int f(v) d  \hat{\mu}_{u'}(v) = \sum_{l=1}^L (h_l(u)- h_l(u')) f(v_l) \leq C \sum_{l=1}^L |h_l(u) - h_l(u')|.    \]
From the above it follows
\[ \int_{A \times A} W_1( \hat{\mu}_{u}, \hat{\mu}_{u'}  ) d \Upsilon_n(u, u') \leq C \sum_{l=1}^L \int_{A\times A} | h_l(u)- h_l(u')| d\Upsilon_n(u,u'). \]
We now invoke Lemma 3.10 in \cite{NGTSlepcev} to conclude that the right-hand side of the above expression converges to zero as $n \rightarrow \infty$. In particular, there exists $N$ large enough such that for all $n \geq N$ we have $C \sum_{l=1}^L \int_{A\times A} | h_l(u)- h_l(u')| d\Upsilon_n(u,u') \leq \frac{\veps}{3}$. In turn, we conclude that if $n \geq N$, then
\[  \int_{A \times A} W_1( \mu_{u}, \mu_{u'}  ) d \Upsilon_n(u, u') \leq \veps. \]
This establishes the desired result.
\end{proof}

\begin{lemma}
\label{lem:SatisfyingAssumptions}
Let $A, B$ be two bounded Borel subsets of $\R^{d}$ and $\R^{d'}$, respectively. Let $\mu \in \mathcal{P}(A)$, and let $ u \in A \mapsto \mu_u (\cdot ) \in \mathcal{P}(B)  $ be a measurable map.

Let $u_1, \dots, u_n, \dots$ be a sequence of i.i.d. samples from $\mu$, and for each $i\in \N$, let $v_{i1}, \dots, v_{im}, \dots,$ be i.i.d. samples from $\mu_{u_i}(\cdot)$. For each $n$ and $m$ consider the (random) measures
\[ \mu^{n,m}:= \frac{1}{nm}\sum_{i=1}^{n}\sum_{j=1}^{m}  \delta_{(u_i, v_{ij} )}, \quad \mu^n:=\frac{1}{n}\sum_{i=1}^n \delta_{u_i}, \]
and let $\mu^{n,m}(\cdot| u)$ be the conditional distribution, according to $\mu^{n,m}$, of the variable $v$ given the value $u$ of the first coordinate.

Then 
\[  \lim_{n \rightarrow \infty} \lim_{m \rightarrow \infty} \E \left[   \inf_{\Upsilon_n \in \Gamma_{\text{Opt}}(\mu^{n}, \mu )} \int   W_1(\mu^{n,m}(\cdot | u ), \mu_{u'} )   d\Upsilon_n(u,u')  
 \right]=0.      \]
In particular, there is a sequence $\{(n_k,m_k)\}_{k \in \N}$ such that 
\[   \lim_{k \rightarrow \infty} \E \left[   \inf_{\Upsilon_k \in \Gamma_{\text{Opt}}(\mu^{n_k}, \mu )} \int   W_1(\mu^{n_k,m_k}(\cdot | u ), \mu_{u'} )   d\Upsilon_k(u,u')  
 \right]=0,      \]
and a subsequence (not relabeled) such that
\[ \lim_{k \rightarrow \infty}  \inf_{\Upsilon_k \in \Gamma_{\text{Opt}}(\mu^{n_k}, \mu )} \int   W_1(\mu^{n_k,m_k}(\cdot | u ), \mu_{u'} )   d\Upsilon_k(u,u')  =0,  \]
almost surely.
\end{lemma}

\begin{proof}

Let $\Upsilon_n \in \Gamma_{\text{Opt}}(\mu^n , \mu)$. By Corollary 5.22 in \cite{VillaniOldNew} this random measure can be chosen in a measurable way over the tacitly defined sample space giving support to the random variables in the problem.  

From the triangle inequality for $W_1$ we have
\begin{align}
\label{App:AuxApp1}
\begin{split}
\int   W_1(\mu^n(\cdot | u ), \mu_{u'} )   d\Upsilon_n(u,u') & \leq  \int  W_1(\mu^n(\cdot | u ), \mu_{u} )   d\Upsilon_n(u,u') + \int  W_1(\mu_u, \mu_{u'} )   d\Upsilon_n(u,u')
\\& =  \frac{1}{n}\sum_{i=1}^n W_1(\mu^{n,m}(\cdot | u_i ), \mu_{u_i} )  + \int  W_1(\mu_u, \mu_{u'} )  d\Upsilon_n(u,u').
\end{split}
\end{align}
In what follows we analyze each of the terms on the right-hand side of the above expression. We start with the second term.

Let us introduce $ \hat{\Upsilon}_n:= \E [ \Upsilon_n ]$, the (deterministic) measure that acts on test functions $\phi$ according to
\[ \int \phi(u,u') d\hat{\Upsilon}_n(u, u') = \E [ \int \phi(u,u') d{\Upsilon}_n(u, u')   ].   \]
It is straightforward to see that $\hat{\Upsilon}_n \in \Gamma(\mu, \mu)$. Now, due to the boundedness of the space $A$ and the fact that $\mu^n $ converges weakly to $\mu$ almost surely, we know that, almost surely,
\[ \lim_{n \rightarrow \infty} \int |u-u'| d\Upsilon_n(u,u') = \lim_{n \rightarrow \infty} W_1(\mu^n, \mu)  =0.   \]
By the dominated convergence theorem it thus follows
\[  \lim_{n \rightarrow \infty} \int |u-u'| d\hat{\Upsilon}_n(u,u')=  \lim_{n \rightarrow \infty} \E [\int |u-u'| d{\Upsilon}_n(u,u') ] = 0. \]
In particular, $\{ \hat{\Upsilon}_n \}_{n \in \N}$ is a stagnating sequence of transport plans for $\mu$, and thus, from Lemma \ref{Prop:StagnatingSequence} it follows that 
\[ \lim_{n \rightarrow \infty} \E [\int W_1( \mu_{u}, \mu_{u'}  )  d{\Upsilon}_n(u,u')]=  \lim_{n \rightarrow \infty} \int W_1( \mu_{u}, \mu_{u'}  )  d\hat{\Upsilon}_n(u,u') =0. \]

We now study the first term on the right-hand side of \eqref{App:AuxApp1}. To avoid introducing cumbersome notation, we will assume for simplicity that all the $u_i$ are different so that in particular $\mu^{n,m}(\cdot|u_i)= \frac{1}{m} \sum_{j=1}^m \delta_{v_{ij}}$. We then have 
\begin{align}
\begin{split}
\lim_{m \rightarrow \infty} \E [ \frac{1}{n}\sum_{i=1}^n W_1(\mu^{n,m}(\cdot | u_i ), \mu_{u_i} )     ] & =    \E [ \lim_{m \rightarrow \infty} \frac{1}{n}\sum_{i=1}^n W_1(\mu^{n,m}(\cdot | u_i ), \mu_{u_i} )     ] 
\\& =    \E [ \E [ \lim_{m \rightarrow \infty} \frac{1}{n}\sum_{i=1}^n W_1(\mu^{n,m}(\cdot | u_i ), \mu_{u_i} ) | u_1, \dots, u_n ]    ]
\\&= 0,
\end{split}
\end{align}
where we have used the dominated convergence theorem in the first line, and the fact that $\frac{1}{m} \sum_{j=1}^m \delta_{v_{ij}}$ converges almost surely in the Wasserstein sense toward $\mu_{u_i}$ in the last line.
\end{proof}

\begin{remark}
\label{rem:App1}
Let $\{\mu^n\}_{n \in \N}$ be a sequence of probability measures over $A \times B$ and let $\mu$ be a probability measure. We show that the condition
\[  \inf_{\Upsilon_n \in \Gamma_{\text{Opt}}(\mu^{n}_u, \mu_u )} \int   W_1(\mu^{n}(\cdot | u ), \mu(\cdot| u') )   d\Upsilon_n(u,u') \rightarrow 0  \]
implies
\[ W_1(\mu^n , \mu) \rightarrow 0 ,\]
while the converse is not true in general; in the above, $\mu^n_u$ and $\mu_u$ denote the first marginals of $\mu^n$ and $\mu$, respectively. Indeed, suppose that the first condition holds, and for each $u,u'$ let $\Upsilon^{u,u'}$ be a coupling between $\mu^n(\cdot|u)$ and $\mu(\cdot | u')$ realizing the $W_1$ distance. Also, choose $\Upsilon_n$ in $\Gamma_{\text{Opt}}(\mu^n, \mu)$ such that $\int   W_1(\mu^{n}(\cdot | u ), \mu(\cdot| u') )   d\Upsilon_n(u,u') \rightarrow 0$ , and consider the measure $d\pi_n((u,v),(u',v')):=d\Upsilon^{u,u'}(v,v') d \Upsilon_n(u,u')  $. It is straightforward to verify that $\pi_n \in \Gamma(\mu^n, \mu)$ and that $\int |(u,v) - (u',v')| d \pi_n \rightarrow 0 $. This implies $W_1(\mu^n, \mu) \rightarrow 0$. 

As we stated earlier, the converse statement is not true. For example, taking $A=[0,1]$, $B=[0,1]$, $\mu$ the uniform measure on $[0,1]^2$, and $\mu^n = \frac{1}{n}\sum_{j} \delta_{(u_j, v_j)}$ with $(u_1,v_1), \dots, (u_n, v_n)$ i.i.d. samples from $\mu$, we see that $W_1(\mu^n, \mu) \rightarrow 0$, while $\inf_{\Upsilon_n \in \Gamma_{\text{Opt}}(\mu^{n}_u, \mu_u )} \int   W_1(\mu^{n}(\cdot | u ), \mu(\cdot| u') )   d\Upsilon_n(u,u')  =1$ for all $n$.
\end{remark}

\begin{lemma}
\label{lem:Aux2App}
Consider the same setting and notation as in Lemma \ref{lem:SatisfyingAssumptions}. Let $\rho: A \times B\rightarrow [0,D]$ be a measurable function satisfying 
\[ \int \rho(u,v) d\mu_u(v) = 1,   \]
for all $u$ in the support of $\mu$. Then, with probability one, 
\[ \lim_{n \rightarrow \infty} \lim_{m \rightarrow \infty} \frac{1}{n}\sum_{i=1}^n \left| \frac{1}{\frac{1}{m}\sum_{j=1}^m \rho(u_i, v_{ij})  } -1 \right| =0 .\] 
\end{lemma}
\begin{proof}
This is a direct consequence of the law of large numbers.
\end{proof}

\subsection{Auxiliary lemmas for section \ref{sec:LongTimeNash}}
\label{app2}

The following result follows from a Gronwall-type argument.

\begin{lemma}
\label{lemm:Gronwall}
Let $\tilde B, M , K, \lambda>0$, and suppose $h: [0, \infty) \rightarrow [0,\infty)$ is a function satisfying
\[ h(t) \leq 2M + \frac{\tilde B }{K} t - \lambda \int_{0}^t h(s) ds, \]
for all $t$. Then, for all $T>0$,
\[ \frac{1}{T} \int_0^T h(s) ds \leq \frac{\tilde B}{K \lambda} + \frac{A}{T}, \]
where $A:= \frac{1}{\lambda}|2M - \frac{\tilde B}{K \lambda}|$.
\end{lemma}

\begin{proof}
The condition on $h$ can be equivalently written as
\[ h(t) - \frac{\tilde B}{K \lambda} \leq (2M - \frac{\tilde B}{K \lambda}) - \lambda \int_0^t ( h(s) - \frac{\tilde B}{K \lambda}) ds. \]
Let $H(t):=\int_0^t ( h(s) - \frac{\tilde{B}}{K \lambda}) ds $. The above condition can thus be written as
\[ H'(t) \leq  (2M - \frac{\tilde{B}}{K \lambda}) - \lambda H(t). \]
From this it follows that
\[ \frac{d}{dt} ( H(t) e^{\lambda t}) \leq (2M - \frac{\tilde{B}}{K \lambda}) e^{\lambda t}.\]
Integrating the above expression, we get:
\[ H(t) e^{\lambda t} \leq (2M - \frac{\tilde{B}}{K\lambda})\frac{1}{\lambda} (e^{\lambda t } - 1), \]
or what is the same
\[  H(t) \leq (2M - \frac{\tilde{B}}{K \lambda}) \frac{1}{\lambda}  - (2M- \frac{\tilde{B}}{K \lambda })\frac{1}{\lambda} e^{- \lambda t}.  \]
Recalling the definition of $H$, we deduce that
\[ \frac{1}{T}\int_0^T h(s) ds  \leq \frac{\tilde{B}}{K \lambda}  + \frac{1}{T} (2M- \frac{\tilde{B}}{K \lambda})\frac{1}{\lambda} - \frac{1}{T}(2M - \frac{\tilde{B}}{K \lambda }) \frac{1}{\lambda}e^{- \lambda T} \leq \frac{\tilde{B}}{K \lambda} + \frac{A}{T}. \]

\end{proof}

\printbibliography

\end{document}